%% file: main_arxiv.tex
\documentclass[a4paper,10pt]{article}
\usepackage{geometry}

\usepackage[round]{natbib}

\usepackage{graphicx}
\usepackage{hyperref}
\hypersetup{colorlinks=true, linkcolor=magenta, citecolor=blue}
\usepackage{subcaption}
\usepackage{amsmath,amsthm,amssymb}
\usepackage{bm}
\usepackage{color}
\input{head}

\begin{document}

\title{The global optimum of shallow neural network is attained by ridgelet transform}

\author{
	Sho Sonoda, 
	Isao Ishikawa,
	Masahiro Ikeda,
	Kei Hagihara\\
	RIKEN AIP, Tokyo, Japan\\
	\texttt{\{sho.sonoda, isao.ishikawa, masahiro.ikeda, kei.hagihara\}@riken.jp}\\
	\and
	Yoshihiro Sawano\\
	Tokyo Metropolitan University, Tokyo, Japan\\
	\texttt{ysawano@tmu.ac.jp}
	\and
	Takuo Matsubara, Noboru Murata\\
	Waseda University, Tokyo, Japan\\
	\texttt{takuo.matsubara@suou.waseda.jp}\\
	\texttt{noboru.murata@eb.waseda.ac.jp}
}

\date{January 23, 2019}



\maketitle

\input{body_arxiv_v3}

\subsubsection*{Acknowledgements}
This work was supported by JSPS KAKENHI 18K18113.

\bibliography{library_summary}

\input{results_all}

\end{document}

%% file: head.tex
\newcommand{\argmin}{\operatornamewithlimits{argmin}}

\newcommand{\sign}{\mathrm{sgn\,}}

\newcommand{\id}{\mathrm{Id}}

\newcommand{\RR}{{\mathbb{R}}}
\newcommand{\NN}{{\mathbb{N}}}
\newcommand{\CC}{{\mathbb{C}}}
\newcommand{\EE}{{\mathbb{E}}}
\newcommand{\dd}{{\mathrm{d}}}

\newcommand{\iprod}[1]{\langle#1\rangle}

\renewcommand{\aa}{{\bm{a}}}

\newcommand{\xx}{{\bm{x}}}
\newcommand{\yy}{{\bm{y}}}

\newcommand{\ttheta}{{\bm{\theta}}}

\newcommand{\refeq}[1]{\eqref{eq:#1}}
\newcommand{\reffig}[1]{Figure~\ref{fig:#1}}

\newcommand{\refapp}[1]{Appendix~\ref{app:#1}}
\newcommand{\refsec}[1]{\S~\ref{sec:#1}}

\newtheorem{thm}{Theorem}[section]

\newtheorem{prop}[thm]{Proposition}

\newcommand{\subhead}[1]{\noindent\textbf{#1}\;}%

%% file: body_arxiv_v3.tex
\begin{abstract}
We prove that the global minimum of the backpropagation (BP) training problem of neural networks with an arbitrary nonlinear activation is given by the ridgelet transform.
A series of computational experiments show that there exists an interesting similarity between the scatter plot of hidden parameters in a shallow neural network after the BP training and the spectrum of the ridgelet transform.
By introducing a continuous model of neural networks, we reduce the training problem to a convex optimization in an infinite dimensional Hilbert space, and obtain the explicit expression of the global optimizer via the ridgelet transform.
\end{abstract}

\section{Introduction}

Training a neural network is conducted by backpropagation (BP), 
which results in a \emph{high-dimensional and non-convex} optimization problem.
Despite the difficulty of the optimization problem, deep learning has achieved great success in a wide range of applications such as image recognition \citep{Redmon2015}, speech synthesis \citep{Oord2016a}, and game playing \citep{Silver2017}.
The empirical success of deep learning suggests a conjecture that ``all'' local minima of the training problem are close or equal to global minima \citep{Dauphin2014,Choromanska2015a}.
Therefore, radical reviews of the shape of loss surfaces are ongoing \citep{Draxler2018,Garipov2018}.
However, these lines of studies pose strong assumptions such as linear activation \citep{Kawaguchi2016,Hardt2017}, overparameterization \citep{Nguyen2017}, Gaussian data distribution \citep{Brutzkus2017}, and shallow network, i.e. single hidden layer \citep{Li2017, Soltanolkotabi2017, Zhong2017, Du2018, Ge2018, Soudry2018}.

The scope of this study is the shape of the global optimizer itself, rather than the reachability to the global optimum via empirical risk minimization.
By recasting the BP training as a variational problem, i.e. an optimization problem in a function space, 
in the settings of the shallow neural network with an \emph{arbitrary} activation function and the mean squared error,
we present an \emph{explicit expression} of the global minimizer via the \emph{ridgelet transform}.
By virtue of functional analysis, our result is independent of the parameterization of neural networks.

\reffig{ridgelet} presents an intriguing example that motivates our study. 
Both Figures \ref{fig:bp} and \ref{fig:spectrum} 
were obtained from the same dataset shown in \reffig{data},
and they show similar patterns to each other.
However, they were obtained from entirely different procedures: numerical optimization and numerical integration.
In the following, we provide a brief explanation of the experiments. See \refsec{motivation} for more details.

\begin{figure}[t]
\centering
\begin{subfigure}[c]{0.5\textwidth}
\includegraphics[width=\linewidth, trim=1cm 0cm 1cm 1cm, clip]{./exp/nsin02pt_e0010_n1000_tanh_h010_adam_nn1000}
\caption{BP trained parameters}
\label{fig:bp}
\end{subfigure}%
\begin{subfigure}[c]{0.5\textwidth}
\includegraphics[width=\linewidth, trim=1cm 0cm 1cm 1cm, clip]{./exp/nsin02pt_e0010_n1000_tanh_2d}
\caption{ridgelet spectrum}
\label{fig:spectrum}
\end{subfigure}\\
\begin{subfigure}[c]{1.0\textwidth}
\includegraphics[width=\linewidth, trim=0cm 0cm 0cm 0cm, clip]{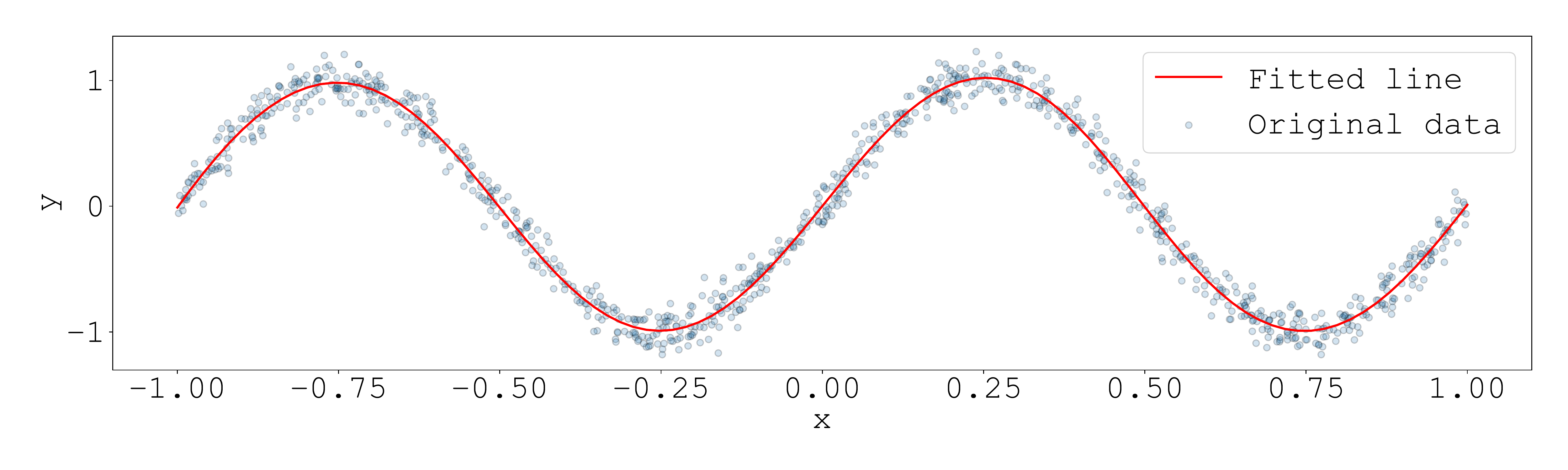}
\caption{dataset and an example of training results}
\label{fig:data}
\end{subfigure}%
\caption{Motivating example: Scatter plot (a) and ridgelet spectrum (b) were obtained from the same dataset (c) and bear an intriguing resemblance to each other, despite the fact that they were obtained from different procedures---numerical optimization and numerical integration.}
\label{fig:ridgelet}
\end{figure}

\reffig{bp} shows the scatter plot of the parameters $(a_j,b_j,c_j)$ in neural networks
$g(x ; \ttheta) = \sum_{j=1}^{p} c_j \sigma( a_j x - b_j  )$
that had been trained with dataset $D = \{ (x_i, y_i) \}_{i=1}^{1,000}$.
The dataset is composed of uniform random variables $x_i$ in $[-1,1]$, and the response variables $\sin 2 \pi x_i + \varepsilon_i$ with Gaussian random noise $\varepsilon_i$.
We trained $n=1,000$ shallow neural networks.
Each network had $p=10$ hidden units with activation $\sigma(z) = \tanh z$.
We employed ADAM for the training.
The scatter plot presents the $np = 10,000$ sets of trained parameters $(a_j, b_j, c_j)$, where $c_j$ is visualized in color.

On the other hand, \reffig{spectrum} shows the spectrum of the (classic) ridgelet transform 
\begin{align}
    R[f](a,b) := \int_{\RR} f(x) \overline{\rho( a x - b )} \dd x,
\end{align}
of $f(x) = \sin 2 \pi x + \varepsilon, \, (x \in [-1,1])$ with a certain ridgelet function $\rho$.
(See \refsec{background.ridgelet} for more details on the ridgelet transform.)
We calculated the spectrum by using numerical integration, and used the dataset $D$.
Therefore, two figures are obtained from the same dataset.

Even though the two figures are obtained from different procedures, both results are $10$-point star shaped.
In other words, the BP trained parameters $(a_j, b_j, c_j)$ concentrate in the high intensity areas in the ridgelet spectrum.
From this interesting similarity, we can conjecture that the global minimizer has a certain relation to the ridgelet transform.

In this study, we investigate the relation between the BP training problem and the ridgelet transform
by reformulating the BP training in a function space,
and show that the ridgelet transform can offer the global minimizer of the BP training problem.

 \section{Preliminaries} \label{sec:background.ridgelet}
 We provide several notation and describe the problem formulation.
 The most important notion is the `BP in the function space,' which plays a key role to formulate our research question.

 \subsection{Mathematical Notation}
 $\overline{z}$ denotes the complex conjugate of a complex number $z$.
 $\widehat{f}$ denotes the Fourier transform $\int f(\xx) e^{-i \xx \cdot \bm{\xi}} \dd \xx$ of a function $f$.
 $\widetilde{f}$ denotes the reflection of a function $f$, i.e. $\widetilde{f}(\xx) = f(-\xx)$.

 $L^2(\mu)$ denotes the Hilbert space equipped with inner product $\iprod{f,g} := \int f(\xx) \overline{g(\xx)} \dd \mu(\xx)$.
 $T^*$ denotes the adjoint operator of a linear operator $T$ on a Hilbert space.

 $\EE_X[ f(X) ]$ denotes the expectation of a function $f(x)$ with respect to the random variable $X$.

 \subsection{Problem Settings}

 \subhead{Neural Network.}
 We consider an $m$-in-$1$-out shallow neural network with 
 an arbitrary activation function $\sigma : \RR \to \CC$:
 \begin{align}
 g(\xx;\ttheta) = \sum_{j=1}^p c_j \sigma( \aa_j \cdot \xx - b_j ), \quad \xx \in \RR^m \label{eq:ordinary.nn}
 \end{align}
 where $p \in \NN$ is the number of hidden units,
 $(\aa_j, b_j) \in \RR^m \times \RR$ are hidden parameters and $c_j \in \CC$ are output parameters.
 By $\ttheta$, we collectively write a set of parameters $\{(\aa_j, b_j, c_j) \}_{j=1}^p$.
 Here, we remark that the $1$-dimensional output assumption is only for simplicity,
 and we can easily generalize our results to the multi-dimensional output case.
 Examples of the activation function are Gaussian, hyperbolic tangent, sigmoidal function and rectified linear unit (ReLU).

 \subhead{Cost Function.}
 We formulate the BP training as the minimization problem of the mean squared error
 \begin{align}
     L(\ttheta) = \EE_X | f(X) - g(X ; \ttheta) |^2 + \Omega(\ttheta), \label{eq:ordinary.bp}
 \end{align}
 with a certain regularization $\Omega$,
 where $f : \RR^m \to \CC$ denotes the ground truth function.
 Here, we remark that this formulation covers any empirical risk function $L_s(\ttheta) = \frac{1}{s} \sum_{i=1}^s | y_i - g(\xx_i ; \ttheta) |^2 + \Omega(\ttheta)$,
 by choosing the data distribution as an empirical distribution.

 \subsection{Integral Representation of Neural Network}
 In order to recast the BP training in the function space,
 we introduce 
 the \emph{integral representation} of a neural network:
 \begin{align}
 S[\gamma](\xx) := \int_{\RR^m \times \RR} \gamma(\aa, b) \sigma( \aa \cdot \xx - b ) \dd \lambda(\aa, b), \label{eq:intrep}
 \end{align}
 where $\gamma : \RR^m \times \RR \to \CC$ are the coefficient function,
 $\sigma : \RR \to \CC$ is the activation function employed in \refeq{ordinary.nn},
 and  
 $\lambda$ is the base measure on $\RR^m \times \RR$.

 \subhead{Brief Description.}
 Formally speaking, $S[\gamma]$ is an infinite sum of hidden units $\sigma(\aa \cdot \xx - b)$.
 In the integral representation, all the hidden parameters $(\aa, b)$ are integrated out,
 and only the output parameter $\gamma(\aa,b)$ is left.
 In other words, $\gamma(\aa,b)$ indicates which $(\aa,b)$ to use by weighting on them.

 \subhead{Function Class.}
 In this study, we assume that $\gamma \in L^2(\lambda)$, and $\lambda$ be a Borel measure.
 As described in Proposition \ref{image of S}, the base measure $\lambda$ controls the expressive power of neural networks, i.e. the capacity of $\{ S[\gamma] \mid \gamma \in L^2(\lambda) \}$.

 \subhead{Important Examples.} Two extreme cases are important: (a) $\lambda$ is the Lebesgue measure, and (b) $\lambda$ is a sum of Dirac measures.
 When $\lambda$ is the Lebesgue measure $\dd \aa \dd b$, then $S[\gamma]$ can express any $L^2$-function \citep{Sonoda2015}. On the other hand, when $\lambda$ is a sum of Dirac measures,
 then $S[\gamma]$ can express any finite neural network \refeq{ordinary.nn}.
 With a slight abuse of notation, write
 \begin{align}
     \gamma_\ttheta \dd \lambda = \sum_{j=1}^p c_j \delta_{(\aa_j, b_j)},
 \end{align}
 for $\ttheta = \{ (\aa_j, b_j, c_j) \}_{j=1}^p$,
 where $\delta_{(\aa, b)}$ denotes the Dirac delta centered at $(\aa, b)$.
 Then, $S[\gamma_\ttheta](\xx) = \sum_{j=1}^p c_j \sigma( \aa_j \cdot \xx - b_j )$.
 In other words, the integral representation is a \emph{reparameterization} of neural networks, and $\gamma_\ttheta$ is the simplest way to connect the integral representation and the ordinary representation.

 \subhead{Advantages.}
 The integral representation has at least two advantages over the `ordinary representation' \refeq{ordinary.nn}.
 The first advantage is that $\gamma_\ttheta$ can expressive any distribution of parameters. By virtue of this flexibility, we can identify the scatter plot \reffig{bp} as a point spectrum, and \reffig{spectrum} as a continuous spectrum.

 The second advantage is that the hidden parameters are integrated out, and that the output parameter is the only trainable parameter.
 Recall that the BP training of ordinary neural networks $g(\xx;\ttheta)$ is a non-convex optimization problem.
 The non-convexity is caused by the hidden parameters $(\aa_j,b_j)$, because they are placed in the nonlinear function $\sigma$.
 At the same time, the non-convexity is never caused by the output parameters $c_j$, because they are placed out of $\sigma$.
 On the other hand, in the integral representation, no trainable parameters are placed in $\sigma$.
 By virtue of this linearity, the BP training of $S[\gamma]$, which is described later in this section, becomes a \emph{convex} optimization problem.

 \subhead{Brief History.} Originally, the integral representation and ridgelet transform have been developed to investigate the expressive power of neural networks \citep{Barron1993, Murata1996, Candes.PhD}, and to estimate the approximation errors \citep{Kurkova2012}.
 Recently, it has been applied to synthesize neural networks without BP training, by approximating the integral transform with a Riemannian sum \citep{Sonoda2014, Bach2014, Bach2015}; to facilitate the inner mechanism of the so-called ``black-box'' networks \citep{Sonoda2018a}, and to estimate the generalization errors of deep neural networks from the decay of eigenvalues \citep{Suzuki2018}.

 \subsection{Ridgelet Transform} \label{sec:ridgelet}
 We placed the explanation of the ridgelet transform soon after the integral representation, because it is natural to understand the ridgelet transform as a right inverse operator for the integral representation operator.

 Let us consider an integral equation $S[\gamma] = f$, where $S$ is an integral representation operator, $f$ is a given function, and $\gamma$ is the unknown function.
 In the context of neural networks, this equation means a prototype of learning.
 Namely, to learn $f$ is to find a solution $\gamma$ from the observation $f$.
 \citet{Murata1996} and \citet{Candes.PhD} discovered that the ridgelet transform provides a particular solution to the equation.

 To be precise, when the base measure $\lambda$ of $S$ is the Lebesgue measure $\dd \aa \dd b$,
 the function $f$ belongs to $L^2(\RR^m)$,
 and there exists a ridgelet function $\rho : \RR \to \CC$ that satisfies the \emph{admissibility condition}
 \begin{align}
 \int_\RR \frac{\widehat{\sigma}(\zeta)\overline{\widehat{\rho}(\zeta)}}{|\zeta|^{m}} \dd \zeta = 1,
 \end{align}
 for the activation function $\sigma$ in $S$,
 then a particular solution to $S[\gamma] = f$ is given by the ridgelet transform
 \begin{align}
     R[f](\aa,b) = \int_{\RR^m} f(\xx) \overline{\rho(\aa \cdot \xx - b)} \dd \xx. \label{eq:classic}
 \end{align}
 This is what we call the \emph{classic ridgelet transform}.

 Here, we remark that the solution is not unique. On the contrary, there are an infinite number of different particular solutions, say $\gamma$ and $\gamma'$, that satisfy $\gamma \neq \gamma'$ but $S[\gamma] = S[\gamma']$. This is immediate from the fact that there are infinitely many different admissible ridgelet functions $\rho$ and $\rho'$.
 Therefore, a single $R$ (specified by $\rho$) is not the exact inverse to $S$, which must satisfy both $SR = \id$ and $RS = \id$; but only a right inverse, which only satisfies $SR = \id$.

 In the context of neural networks, the existence of a solution operator $R[f]$ for any function $f$
 means the \emph{universal approximation property}, 
 because a neural network $S[\gamma]$ can express any function $f$ by just letting $\gamma = R[f]$.

 As demonstrated in \refsec{motivation}, the ridgelet transform can be computed by numerical integration.
 See \citet{Starck2010} and \citet{Sonoda2015} for more details on ridgelet analysis.

 \subsection{BP in the Function Space}
 We rewrite the BP training as the minimization problem of 
 \begin{align}
     L[\gamma] = \EE_X| f(X) - S[\gamma](X) |^2 + \Omega[\gamma],
 \end{align}
 with respect to $\gamma \in L^2(\lambda)$.
 This reformulation formally extends the ordinary formulation \refeq{ordinary.bp}, because $L[\gamma_\ttheta] = L(\ttheta)$.
 In other words, we can understand the ordinary BP problem in the function space, as depicted in \reffig{fgradient}.
 We call the minimization problem of $L[\gamma]$ as the \emph{BP in the function space}.
 As we mentioned above, by virtue of the linearity of $S$, the BP in the function is reduced as a \emph{quadratic programming} problem.

 \begin{figure}
     \centering
     \includegraphics[trim=0cm 0cm 1cm 2cm, width=0.5\textwidth,clip]{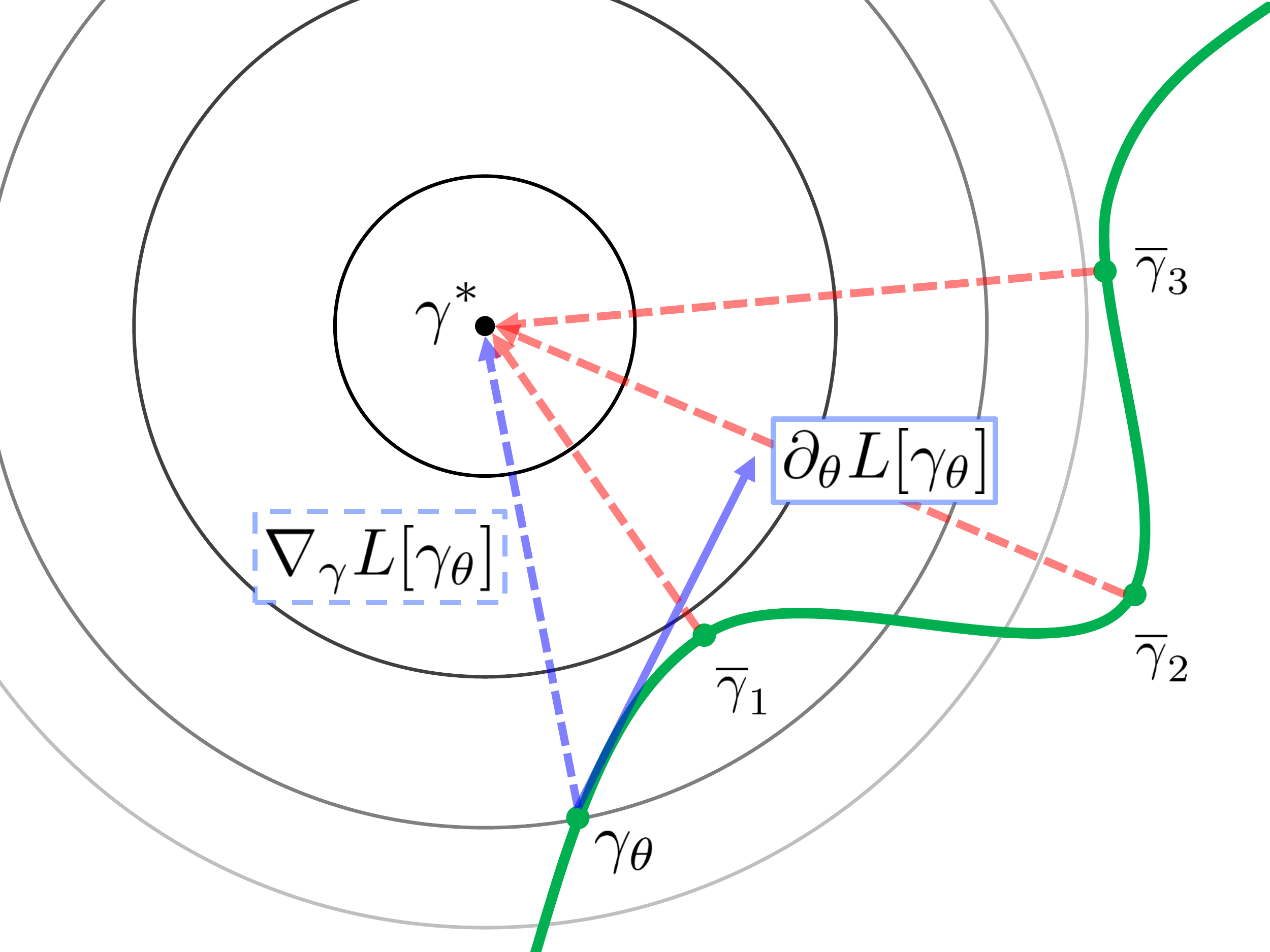}
     \caption{Relation between the minimization problems of $L[\gamma]$ and $L(\ttheta)$.
     The green curve depicts a finite dimensional subspace parameterized by $\ttheta$ and embedded in the ground function space. Since $\gamma_\ttheta$ is restricted to the subspace, the gradient vector (dashed) is also projected to the tangent space (solid), and thus the gradient descent generally goes off in a different direction from the global minimizer $\gamma^*$. If the subspace is curved in the ground function space, there would be multiple local optima such as $\overline{\gamma}_1, \overline{\gamma}_2$ and $\overline{\gamma}_3$.}
     \label{fig:fgradient}
 \end{figure}

 Mathematically speaking, contrary to the finite dimensional optimization problem,
 existence and uniqueness of the solution depend on the properties of $\gamma, f$ and $S$.
 For the sake of simplicity, we consider a simple case $\Omega(\gamma) = \beta \| \gamma \|_{L^2(\lambda)}^2$ with $\beta > 0$.
 In this case, the sufficient condition for the unique existence of the solution is that $S$ is Lipschitz continuous.
 See \refapp{tikhonov} for more details.

 \subhead{Where are the Local Minima?}
 The BP training in the function space, i.e. $\min_\gamma L[\gamma]$, has a unique global minimum because it is a quadratic programming,
 while the BP training in the parameter space, i.e. $\min_\ttheta L(\ttheta)$, generally has a large number of local minima.
 This is not a paradox, but simply a matter of parameterization.

 In order to figure out the paradox, let us consider performing gradient descent for cost functions $L[\gamma]$ and $L(\ttheta)$.
 Namely, for $L[\gamma]$, we use functional gradient (Fr\'{e}chet derivative) $\nabla_\gamma L$; and for $L(\ttheta)$, we use partial derivative $\partial_\ttheta L$.

 Between these two derivatives, a chain-rule holds:
 \begin{align}
 \partial_\ttheta L[ \gamma_\ttheta ]
 &= \int_{\RR^m \times \RR} \nabla_\gamma L[\gamma_\ttheta(\aa,b)] \partial_\ttheta \gamma_\ttheta(\aa, b) \dd \lambda(\aa,b) \nonumber \\
 &= \iprod{\nabla_\gamma L[\gamma_\ttheta], \partial_\ttheta \gamma_\ttheta}_{L^2(\lambda)}.
 \end{align}
 In other words, this is a change-of-coordinate from $\gamma$ to $\ttheta$.

 According to the chain-rule, if the functional gradient vanishes: $\nabla_\gamma L[\gamma_{\ttheta^*}] = 0$ at some $\gamma_{\ttheta^*}$,
 then the partial derivative also vanishes: $\partial_\ttheta L( \ttheta^* ) = 0$.
 However, the converse is not always true.
 As depicted in \reffig{fgradient},
 if the partial derivative vanishes: $\partial_\ttheta L( \overline{\ttheta} ) = 0$ at some $\overline{\ttheta}$,
 then $\gamma_{\overline{\ttheta}}$ is simply a local optimizer such as $\overline{\gamma}_1, \overline{\gamma}_2$ and $\overline{\gamma}_3$.

 \subsection{Main Problem}
 At last, our research problem is formulated as to show
 \begin{align}
     \argmin_{\gamma \in L^2(\lambda)} L[\gamma ; f] = R[f],
 \end{align}
 with a suitable reformulation of ridgelet transform $R$, if needed.

 \section{Details on Motivating Examples} \label{sec:motivation}

 \input{results_tanh_adam.tex}

 In \reffig{ridgelet}, we have compared the BP trained parameters and the ridgelet spectrum.
 Here, we explain the details of these experiments and review the results with nine additional datasets and three additional conditions.
 We note that readers are also encouraged to refer supplementary materials for further results.

 \subsection{Datasets}
 We prepared $10$ artificial datasets.
 For the sake of visualization, all the datasets are $1$-in-$1$-out.
 We emphasize that our main results described in \ref{sec:theory} are valid for any dimension.
 In the following, $N(\mu,\sigma^2)$ denotes the normal distribution with mean $\mu$ and variance $\sigma^2$,
 $U(s,t)$ denotes the uniform distribution over the interval $[s,t]$.

 \subhead{Common Settings.} In all the datasets, $x_i \sim U(-1,1)$, and except for `Topologist's Sinusoidal Curve', sample size $s = 1,000$.

 \subhead{Sinusoidal Curve.}
 $y_i = \sin 2 \pi x_i$.
 We prepared this dataset as a basic example.

 \subhead{Sinusoidal Curve with Gaussian Noise.}
 $y_i = \sin 2 \pi x_i + \varepsilon_i, \varepsilon_i \sim N(0,\sigma^2)$ with $\sigma^2 = 0.1^2$ and $1^2$. 
 We prepared these datasets to examine the effect of noise.
 By the linearity of ridgelet transform: $R[f + \varepsilon] = R[f] + R[\varepsilon]$, we can expect that the effect will be cancelled out in average.

 \subhead{Gaussian Noise.} $y_i \sim N(0,1^2)$.
 We prepared this dataset to extract the effect of noise.
 In theory, the ridgelet spectrum is a random process. Therefore, the visualization result is only a single realization of the random process.

 \subhead{High Frequency Sinusoidal Curve.} $y_i = \sin 10 \pi x_i$.
 We prepared this dataset to examine the effect of the change in frequency. Since the hidden parameter $\aa$ reflects the frequency, we can expect that the spectrum $R[f](\aa,b)$ will change in $\aa$.

 \subhead{Topologist's Sinusoidal Curve.} $y_i = \sin 1/x_i$, $s = 10,000$.
 We prepared this dataset to examine the effect of the change in frequency. Compared to sinusoidal curve, it contains an  infinitely wide range of frequencies.

 \subhead{Gaussian Kernel.} $y_i = \exp(|x_i - \mu|^2/2)$ with $\mu=-0.5, 0, 0.5$.
 We prepared these datasets to examine the effect of the change in location. Since the hidden parameter $b$ reflects the location, we can expect that the spectrum $R[f](\aa,b)$ will change in $b$.

 \subhead{Square Wave.} $y_i = \sign( \sin 2 \pi x_i )$.
 We prepared this dataset to examine the effect of discontinuity. By the locality of the ridgelet transform, we can expect that the effect is also localized.

 \subsection{Scatter Plots of BP Trained Parameters}

 Given a dataset $D = \{ (x_i, y_i) \}_{i=1}^s$, we repeatedly trained $n=1,000$ neural networks $g(x ; \ttheta) = \sum_{j=1}^p c_j \sigma( a_j x - b_j )$.
 The training is conducted by minimizing the empirical mean squared error: $L(\ttheta) = \frac{1}{s} \sum_{i=1}^s | y_i - g( x_i ; \ttheta ) |^2$.
 After the training, we obtained $np$ sets of parameters $(a_j, b_j, c_j)$, and plotted them in the $(a,b,c)$-space. ($c$ is visualized in color.)

 We preliminary adjusted the hidden units number $p$ according to the dataset.
 Otherwise, the plots become noisy, typically because the initial parameters were not moved during the training.
 If $p$ is too small, the network underfits, and all the parameters become nothing more than noise in the plot. When a network underfits, some $c_j$ get extremely large.
 On the other hand, if $p$ is too large, a large majority of hidden parameters $(a_j, b_j)$ remain to be updated, which again become noise in the plot. When a parameter $(a_j, b_j)$ remains to be updated, the $c_j$ gets extremely small.
 So, we can judge if the parameters are noise or not, by checking if $c_j$ is either extremely large or extremely small.
 For the sake of visualization, we got rid of those noisy parameters.

 We examined the following settings. See supplementary materials for all the results.

 \subhead{Activation Function.} $\tanh$ and ReLU.

 \subhead{Optimization Method.} LBFGS and ADAM.

 \subsection{Numerical Integration of Ridgelet Spectrum}
 We employed the classical definition of the ridgelet transform given in \refeq{classic}.
 As admissible functions $\rho$, we employed
 \begin{align}
     \rho(z) &= (4 z^2 - 2) F(z) - 2z, \tag{\mbox{for $\sigma = \tanh$}} \\
     \rho(z) &= (12 z - 8 z^3) F(z) + 4 z^2 -4,  \tag{\mbox{for $\sigma = $ ReLU}}
 \end{align}
 where $F(z) := e^{-z^2} \int_0^z e^{w^2} \dd w$ is the Dawson function.
 See \citet{Sonoda2015} for more details on the construction of other admissible functions.

 Given the dataset $D = \{ (x_i, y_i) \}_{i=1}^{s}$, we have conducted a simple Monte Carlo integration at every grid points $(a, b)$:
 \begin{align}
 R[f](a,b) \approx \frac{1}{s} \sum_{i=1}^{s} y_i \overline{\rho( a x_i - b )} \Delta x,
 \end{align}
 where $\Delta x$ is a normalizing constant (because $x_i$ is uniformly distributed).
 For simplicity, we omitted calculating $\Delta x$, and simply scaled $R[f]$ so that it has value in $[-1,1]$.
 We remark that more sophisticated methods for the numerical computation of the ridgelet transform has been developed.
 See \citet{Do2003} and \citet{Sonoda2014} for example.

 \subsection{Experimental Results}
 \reffig{results_tanh_adam} presents the experimental results when the activation function $\sigma$ is $\tanh$ and the optimization method is ADAM.
 As mentioned in \ref{sec:ridgelet}, there are an infinite number of different ridgelet transforms $R$,
 and the presented spectra are calculated by just one particular case of $R$.
 Nevertheless, we can find a visual resemblance in all the cases (a-j).

 In (a-c), the three spectra are also similar to each other.
 As we have expected, the effect of noise has been canceled.
 The three scatter plot get blurred, as the noise level $\sigma^2$ gets increased.
 In (d),  the spectrum presents a single shot of the random field $R[\varepsilon]$. In the scatter plot, parameters accumulated along the line. This is because when $(a,b)$ is on this line, the corresponding base function $\sigma(ax-b)$ tends to be a constant function in the domain $x \in [-1,1]$, as the fitting example depicted in the top.
 In (e-d), the scatter plots are noisy because the training easily fails, as shown in the training example. We can find some sharp peaks and troughs.
 In (g-i), we can observe that the location in the real domain is encoded as the angle in the spectrum.
 In (j), the parameters accumulated in a few sharp lines. These lines encode the locations of discontinuities in the real domain.

 \section{Theory} \label{sec:theory}
 We prove the main theorem. All the proofs are given in \refapp{proofs}.
 We fix a locally integrable activation function $\sigma: \RR \to \CC$ and a probability measure $\mu$ on $\RR^m$.
 We remark that the boundedness is not critical but just for simplicity. We can treat unbounded activation functions with small modifications employed in \citet{Sonoda2015}. We define a positive definite kernel $K$ on $\mathbb{R}^m\times \mathbb{R}$ by
 \begin{align}
 K\!\left((\aa,b),(\aa', b')\right):=\int_{\RR^m}\sigma(\aa \cdot \xx-b)\overline{\sigma(\aa' \cdot \xx-b')}\dd\mu(\xx).
 \end{align}
 Let $\mathcal{P}:=\left\{\text{finite Borel measure on }\mathbb{R}^m\times \mathbb{R}\right\}$ and assume $\lambda \in \mathcal{P}$.  For example, the Dirac measure $\delta_{(\aa, b)}$ with support on $(\aa, b)$ is contained in $\mathcal{P}$.

 \subsection{Integral Representation of Neural Network}
 Here, we investigate the properties of integral representation operator $S : L^2(\lambda) \to L^2(\mu)$.

 \begin{prop}
 \label{S bdd}
 The operator $S$ is a bounded linear operator, more precisely, for any $\gamma\in L^2(\lambda)$, we have $||S[\gamma] ||_{L^2(\mu)}\le ||K||_{L^2(\lambda\otimes\lambda)}^{1/2}\cdot||\gamma||_{L^2(\lambda)}$.
 \end{prop}
 Let $k$ be a positive definite kernel on $\RR^m$ defined by \begin{align}k(\xx,\yy):=\int_{\RR^m}\sigma(\aa\cdot \xx-b) \overline{\sigma(\aa\cdot \yy -b)}\,\dd \lambda(\aa,b). \end{align}
 Let $H$ be the reproducing kernel Hilbert space (RKHS) associated with $k$. We note that $H\subset L^2(\mu)$.  Then we have 
 \begin{prop}
 \label{image of S}
 The image of $S$ is $H$.
 \end{prop}
 Proposition \ref{image of S} means that the representation ability of $S$ is described by the RKHS $H$. If $k$ is a universal kernel (cc-universal in the sense in \citet{Sriperumbudur2010}), the RKHS $H$ can approximate any compact support function in $L^2(\mu)$, which is just the universal approximation property of neural networks in $L^2(\mu)$ for arbitrary probability distribution $\mu$.  Actually, under mild condition, we can make $k$ a universal kernel as follows:
 \begin{thm}
 \label{universal kernel}
 Let $A>0$ be a (large) positive number.  Let $\nu\in L^1(\RR^m)$.  Assume that $\sigma$ is a non-constant periodic function with period $2A$, i.e. $\sigma(x+2A)=\sigma(x)$. We also impose one of the two conditions: (1) ${\rm supp}(\nu)=\RR^m$ or (2) $0\in{\rm supp}(\nu)$ and $\#\{\widehat{\sigma}(n)\neq 0\}=\infty$, where $\widehat{\sigma}(n)=(2A)^{-1}\int_{[-A,A]}\sigma(x)e^{\pi i nx/A}dx$.  Then $\lambda= \nu(\aa)\mathbf{1}_{[-A,A]}(b)\dd \aa\dd b\in \mathcal{P}$ induces a universal kernel $k$.
 \end{thm}
 Note that in a real world problems, we can use a periodic function $\sum_{n=-\infty}^\infty \sigma(x-2nA)\mathbf{1}_{[(n-1)A, (n+1)A]}$ with sufficiently large $A>0$ as an alternative activation function. Thus the condition of periodicity for $\sigma$ is not harmful. In particular, we can deal with ReLU.

 \subsection{Ridgelet Transform}
 Here, we introduce a modified version of the ridgelet transform, which attains the global minimum of the BP training problem.

 Let $\lambda\in\mathcal{P}$ and let $\rho\in L^2(\mu\otimes\lambda)$.  Then we define the {\em Ridgelet transform with respect to $\rho$} as a linear operator $R_\rho: L^2(\mu)\rightarrow L^2(\lambda) $  defined by
 \begin{align}R_\rho[f](\aa,b):=\int_{\mathbb{R}^m}f(\xx)\rho(\xx,(\aa,b))\,\dd\mu(\xx).\end{align}
 We note this definition includes the classic definition when $\rho(\xx, (\aa,b)) = \rho(\aa \cdot \xx - b)$ and $\dd \mu(\xx) = \dd \xx$.
 \begin{prop}
 \label{bdd R}
 The ridgelet transform $R_\rho$ is a bounded linear operator, more precisely, $||R_\rho [f]||_{L^2(\lambda)}\le ||f||_{L^2(\mu)}||\rho||_{L^2(\mu\otimes\lambda)}$.
 \end{prop}
 The adjoint $S^*$ of $S$ is described by the ridgelet transform:
 \begin{prop}
 \label{adjoint of R}
 We have $R_{\rho_\sigma}=S^*$, where  $\rho_\sigma(\xx, (\aa,b)):=\sigma(\aa\cdot\xx- b)$.
 \end{prop}

 Let $T: L^2(\lambda)\rightarrow L^2(\lambda)$ be the integral transform with respect to $K$, namely, 
 \begin{align}T[\gamma](\aa,b):=\int_{\RR^m\times \RR}\gamma(\aa',b')K((\aa,b), (\aa',b'))\,\dd\lambda(\aa',b'). \end{align}
 Then we have the following proposition:
 \begin{prop}
 \label{T=SS}
 We have $T=S^*S$.
 \end{prop}

 \subsection{Main Result}
 For $f\in L^2(\mu)$ and $\beta > 0$, let $L[\gamma;f,\beta]:=||S[\gamma]-f||_{L^2(\mu)}^2+\beta||\gamma||_{L^2(\lambda)}^2$ be the risk function for the integral representation of neural networks with respect to $\lambda$, then we have the following theorem:
 \begin{thm}
 \label{global minima}
 For $f\in L^2(\mu)$ and $\beta>0$, there exists a function $\rho^*$ on $\RR\times(\RR^m\times\RR)$ such that $R_{\rho^*}[f]$ attains the unique minimum of the minimization problem $\min_{\gamma\in L^2(\lambda)}L(f,\beta;\gamma)$. Moreover $\rho^*$ satisfies
 \begin{align}(\beta + T) \rho^*(\xx,\cdot)=\sigma_\xx,\end{align}
 where $\sigma_\xx(\aa,b)=\sigma(\aa\cdot\xx-b)$.
 \end{thm}
 In the classic case when $\dd \lambda = \dd \aa \dd b$ and $\dd \mu(\xx) = \dd \xx$, the solution $\rho^*$ is given by $\rho_\xx(\aa,b) = (\beta + 1)^{-1}\rho(\aa \cdot \xx -b )$ with any admissible $\rho$, which results in a shrink version of the classic ridgelet transform $(\beta + 1)^{-1} R_{\rho}$. See \refapp{limitcase} for a sketch of proof.

 \section{Conclusion}
 We have shown that the global minimizer of the BP training problem is given by the ridgelet transform.
 In order to treat the scatter plot of hidden parameters, such as \reffig{bp}, we introduced the integral representation of neural networks, and reformulated the BP training problem in the Hilbert space of coefficient functions, i.e. $\min_{\gamma \in L^2(\lambda)} L[\gamma]$.
 As a result, the BP training problem was reduced to the quadratic programming, without harming the generality of activation functions.
 At last, we have successfully discovered a modified version of the ridgelet transform that attains the global minimum. In the classic setting, the modified transform simplifies to a shrink ridgelet transform. By virtue of functional analysis, our formulation is independent of the parameterization of neural networks, which would contribute to the geometric understanding of neural networks.
 Extensions to general risk functions and deep networks, and applications to the analysis of local minima will be our important future works.

 \appendix

 \section{Optimization Problem in Hilbert spaces} \label{app:tikhonov}
 Let $H_0, H_1$ be Hilbert spaces endowded with the inner products $\iprod{\cdot, \cdot}_0$ and $\iprod{\cdot, \cdot}_1$, respectively, and $A : H_0 \to H_1$ be a densely defined closed linear operator.

 For a given $f \in H_1$, we find $g \in H_0$ satisfying
 \begin{align*}
 A g = f.
 \end{align*}
 For this problem, we have the following.
 \begin{prop}
 \label{minima for tikhonov regularization}
 Let $f\in H_1$. Then for every $\beta >0$, we have
 \begin{align*}
 \argmin_{g \in H_0}\left( \| A g - f \|_1^2 + \beta \| g \|_0^2\right)
 = (\beta + A^* A)^{-1} A^* f,
 \end{align*}
 where $A^*:H_1\rightarrow H_0$ denotes the adjoint operator of $A$.
 \end{prop}
 \begin{proof}
 A direct computation gives
 \begin{align*}
 \lefteqn{\| A g - f \|_1^2 + \beta \| g \|_0^2}\\
 &= \iprod{Ag, Ag}_1 - 2 \Re \iprod{  Ag, f}_1 + \iprod{f, f}_1 + \beta \iprod{g, g}_0 \\
 &= \iprod{\sqrt{\beta + A^* A}g, \sqrt{\beta + A^* A}g}_0 - 2 \Re \iprod{ \sqrt{\beta + A^* A}g, \sqrt{\beta + A^* A}^{-1} A^* f}_0 + \iprod{f, f}_1 \\
 &= \| \sqrt{\beta + A^* A}g- \sqrt{\beta + A^* A}^{-1} A^* f\|_0^2 + ({\rm nonnegative}).
 \end{align*}
 Therefore, the objective functional attains the minimum at $g^* = (\beta + A^* A)^{-1} A^* f$.
 \end{proof}

 \section{Proofs} \label{app:proofs}

 \subsection{Proposition \ref{S bdd}}
 \label{pf S bdd}
 \begin{proof}
 Let $\gamma\in L^2(\lambda)$. Then we have
 \begin{align*}
 &||S [\gamma] ||_{L^2(\mu)}^2\\
 &=\iiint\gamma(\aa,b)\overline{\gamma(\aa',b')}\sigma(\aa\cdot x-b)\overline{\sigma(\aa'\cdot x-b')} \dd \lambda(\aa,b) \dd \lambda(\aa',b')d\mu(x)\\
 &=\int\gamma(\aa,b) \overline{\gamma(\aa',b')}K((\aa,b),(\aa',b'))  \dd\lambda\otimes \dd\lambda((\aa,b),(\aa',b'))\\
 &\le||\gamma||_{L^2(\lambda)}^2||K||_{L^2(\lambda\otimes\lambda)}.
 \end{align*}
 Here, in the last line, we use the Cauchy-Schwarz inequality.
 \end{proof}

 \subsection{Proposition \ref{image of S}}
 \label{pf image of S}
 \begin{proof}
 Let $\sigma_\xx(\aa,b):=\sigma(\aa\cdot \xx -b)$. Let $H'\subset L^2(\lambda)$ be a closure of the linear subspace generated by $\{\sigma_\xx\}_{\xx\in\RR^m}$. By definition, we have $S [\sigma_\xx]=k(\xx, \cdot)$ and $\langle \sigma_\xx, \sigma_\yy\rangle_{L^2(\lambda)}=k(\xx,\yy)$. Thus, $S$ induces an isomorphism between Hilbert spaces $H'$ and $H$, in particular, the image of $S$ is $H$.
 \end{proof}

 \subsection{Theorem \ref{universal kernel}}
 \label{pf universal kernel}
 \begin{proof}
 Let $I:=[-A,A]$.
 Since $\sigma$ is a periodic function, we have a Fourier series expansion of $\sigma$: $\sigma(x)=\sum_{n\in\mathbb{Z}}\widehat{\sigma}(n)e^{\pi i n/A}$. Then we have
 \begin{align*}
 k(\xx,\yy)
 &= \int_{\RR^m \times I} \sigma( \aa \cdot \xx - b) \sigma( \aa \cdot \yy - b ) \nu(\aa) \dd \aa \dd b  \\
 &= \int_{\RR^m} (\sigma * \widetilde{\sigma})( \aa \cdot (\xx - \yy) ) \nu(\aa) \dd \aa, \\ 
 &= \int_{\RR^m}\sum_{n\in\mathbb{Z}} |\widehat{\sigma}(n)|^2e^{2\pi i n \aa \cdot (\xx - \yy) } \nu(\aa) \dd \aa  \\
 &=|a_0|^2||\nu||_{L^1} +\int_{\RR^m} \underbrace{\sum_{n\in\mathbb{Z}\setminus \{0\}} \frac{|\widehat{\sigma}(n)|^2}{|n|^m} \nu \left( \frac{\aa}{n} \right)}_{=:q(\aa)} e^{ \pi i \aa \cdot (\xx - \yy)/A } \dd \aa. 
 \end{align*}
 Here, $\sigma * \widetilde{\sigma}(x):=\int_{[-A,A]}\sigma(t)\sigma(t-x)dt$. 
 Since we assume ${\rm supp}(\nu)=\RR^m$, or $\#\{n~|~\widehat{\sigma}(n)\neq 0\}=\infty$, the support of the function $q$ is $\RR^m$. Therefore, we see that $k$ is universal (see Section 3.2. of \citep{Sriperumbudur2010}).
 \end{proof}

 \subsection{Proposition \ref{bdd R}}
 \label{pf bdd R}
 \begin{proof}
 Let $f\in L^2(\mu)$. Then we have
 \begin{align*}
 &||R_\rho [f]||_{L^2(\lambda)}^2\\
 &=\iiint f(x)f(y)\rho(x,(\aa,b))\rho(y,(\aa,b)) \dd\mu(x) \dd\mu(y) \dd\lambda(\aa,b)\\
 &=\int f(x)f(y)\int \rho(x,(\aa,b))\rho(y,(\aa,b))\dd\lambda(\aa,b)\dd\mu\otimes\dd\mu(x,y)\\
 &\le||f||_{L^2(\mu)}^2||\rho||^2_{L^2(\mu\otimes\lambda)}.
 \end{align*}
 In the last line, we use the Cauchy-Schwartz inequality twice.
 \end{proof}

 \subsection{Proposition \ref{adjoint of R}}
 \label{pf adjoint of R}
 \begin{proof}
 It suffices to prove that $\langle R_{\rho_\sigma}[f], \gamma\rangle_{L^2(\lambda)}=\langle f, S[\gamma]\rangle_{L^2(\mu)}$. By straightforward computation, we have
 \begin{align*}
 &\langle R_{\rho_\sigma}[f], \gamma\rangle_{L^2(\lambda)}\\
 &=\iint f(\xx)\sigma(\aa\cdot\xx-b)\overline{\gamma(\aa,b)}\,d\mu(\xx)d\lambda(\aa,b)\\
 &=\int f(\xx)\int \overline{\gamma(\aa,b)}\sigma(\aa\cdot\xx-b)d\lambda(\aa,b)d\mu(\xx)\\
 &=\langle f, S[\gamma]\rangle_{L^2(\mu)}.
 \end{align*}
 Thus we have $R_{\rho_\sigma}=S^*$.
 \end{proof}

 \subsection{Proposition \ref{T=SS}}
 \label{pf T=SS}
 \begin{proof}
 By Proposition \ref{adjoint of R}, we have
 \begin{align*}
 &S^* S\gamma(\aa,b)\\
 &=\iint \gamma(\aa',b')\sigma(\aa' \cdot \xx-b')\sigma(\aa\cdot\xx-b)\dd\lambda(\aa',b')\dd\mu(\xx)\\
 &=T[\gamma](\aa,b). \qedhere
 \end{align*}
 \end{proof}

 \subsection{Theorem \ref{global minima}}
 \label{pf global minima}
 \begin{proof}
 Since $\beta \id+T$ is an isomorphism, we define $\rho^*(\xx,(\aa,b)):=\left((\beta+T)^{-1}\sigma_\xx\right)(\aa,b)$.  By Proposition \ref{minima for tikhonov regularization}, the minimizer $\gamma^*$ that attains $\min_{\gamma\in L^2(\lambda)}L(f,\beta; \gamma)$ is explicitly given as $(\beta+S^*S)^{-1}S^*f$. By Proposition \ref{adjoint of R} and Proposition \ref{T=SS}, it suffices to prove that $(\beta+T)R_{\rho^*}=R_{\rho_\sigma}$, but it follows from simple computation. 
 \end{proof}

 \subsection{Remark on Theorem \ref{global minima}}
 \label{app:limitcase}
 Let $\rho$ be an arbitrary admissible, namely, $\int_{\RR} \widehat{\sigma}(\zeta) \overline{\widehat{\rho}{\zeta}}|\zeta|^{-m} \dd \zeta = (2 \pi)^{-(m-1)}$. Then,
 \begin{align*}
 &T[ \rho_\xx ](\aa,b) \\
 &= \int \rho( \aa' \cdot \xx - b' ) \int \sigma(\aa \cdot \xx' - b) \overline{\sigma(\aa' \cdot \xx' - b')} \dd \aa' \dd b'\dd \xx' \\
 &= \int \sigma(\aa \cdot \xx' - b)
 \underbrace{\int \rho( \aa' \cdot \xx - b' )\overline{\sigma(\aa' \cdot \xx' - b')} \dd \aa' \dd b'}_{=\delta(\xx-\xx')} \dd \xx' \\
 &= \sigma_\xx(\aa,b).
 \end{align*}
 Here, the last equation holds because
 \begin{align*}
     &\int \rho( \aa \cdot \xx - b )\overline{\sigma(\aa \cdot \xx' - b)} \dd \aa \dd b \\
     &= \int (\overline{\widetilde{\rho}} * \sigma) ( \aa \cdot (\xx - \xx') )\dd \aa \\
     &= (2 \pi)^{-1} \int \int_\RR \overline{\widehat{\rho}(\zeta)}\widehat{\sigma}(\zeta)|\zeta^{-m}| \dd \zeta e^{i \aa \cdot (\xx - \xx')}\dd \aa \\
     &= (2 \pi)^{-m} \int  e^{i \aa \cdot (\xx - \xx')}\dd \aa = \delta( \xx' - \xx ).
 \end{align*}
 Thus, $\rho^* = (\beta + 1)^{-1}\rho_\xx$ solves the equation.

%% file: results_tanh_adam.tex
\begin{figure*}
    \begin{subfigure}[c]{0.1\textwidth}
    \includegraphics[width=\linewidth]{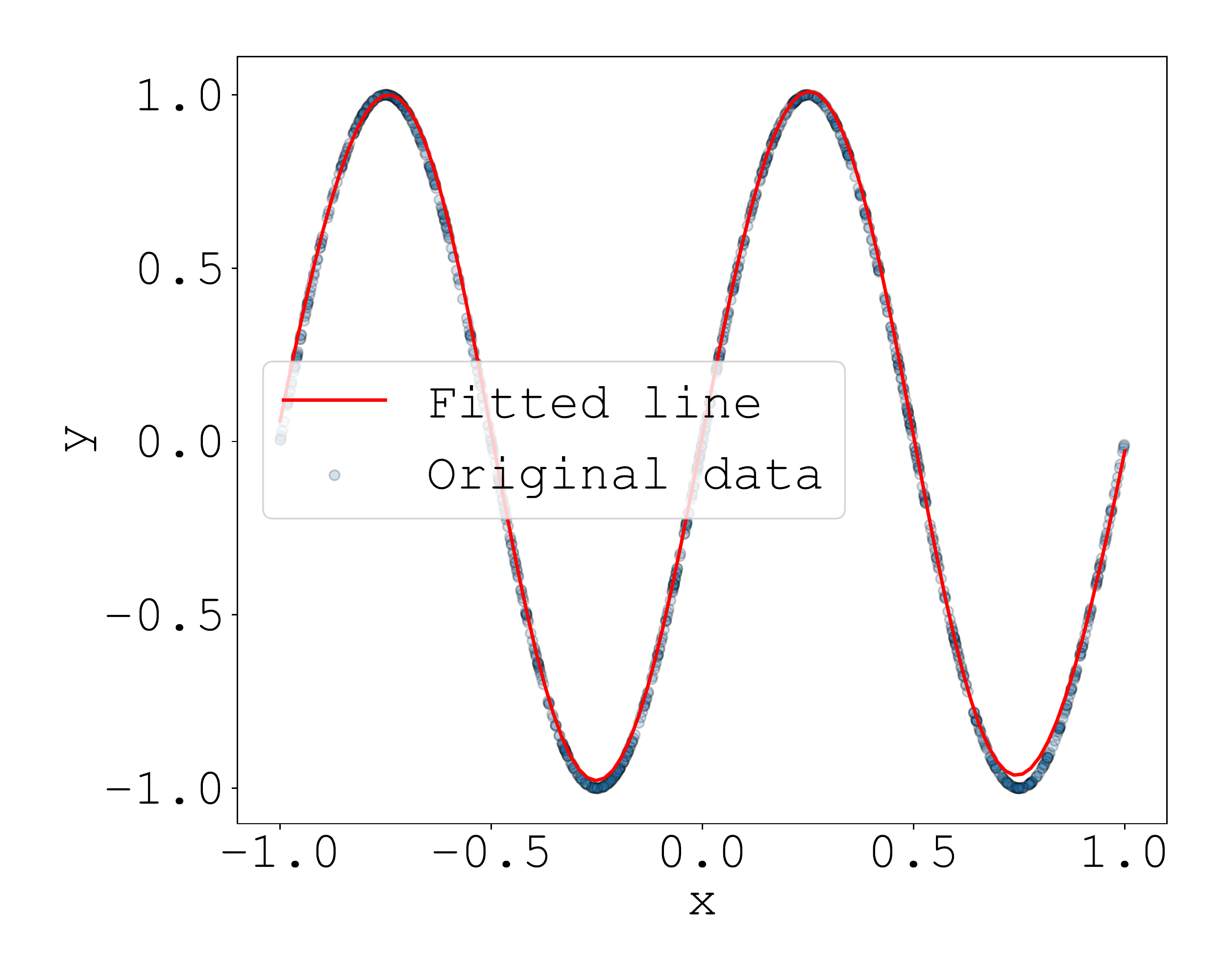}
    \end{subfigure}%
    \begin{subfigure}[c]{0.1\textwidth}
    \includegraphics[width=\linewidth]{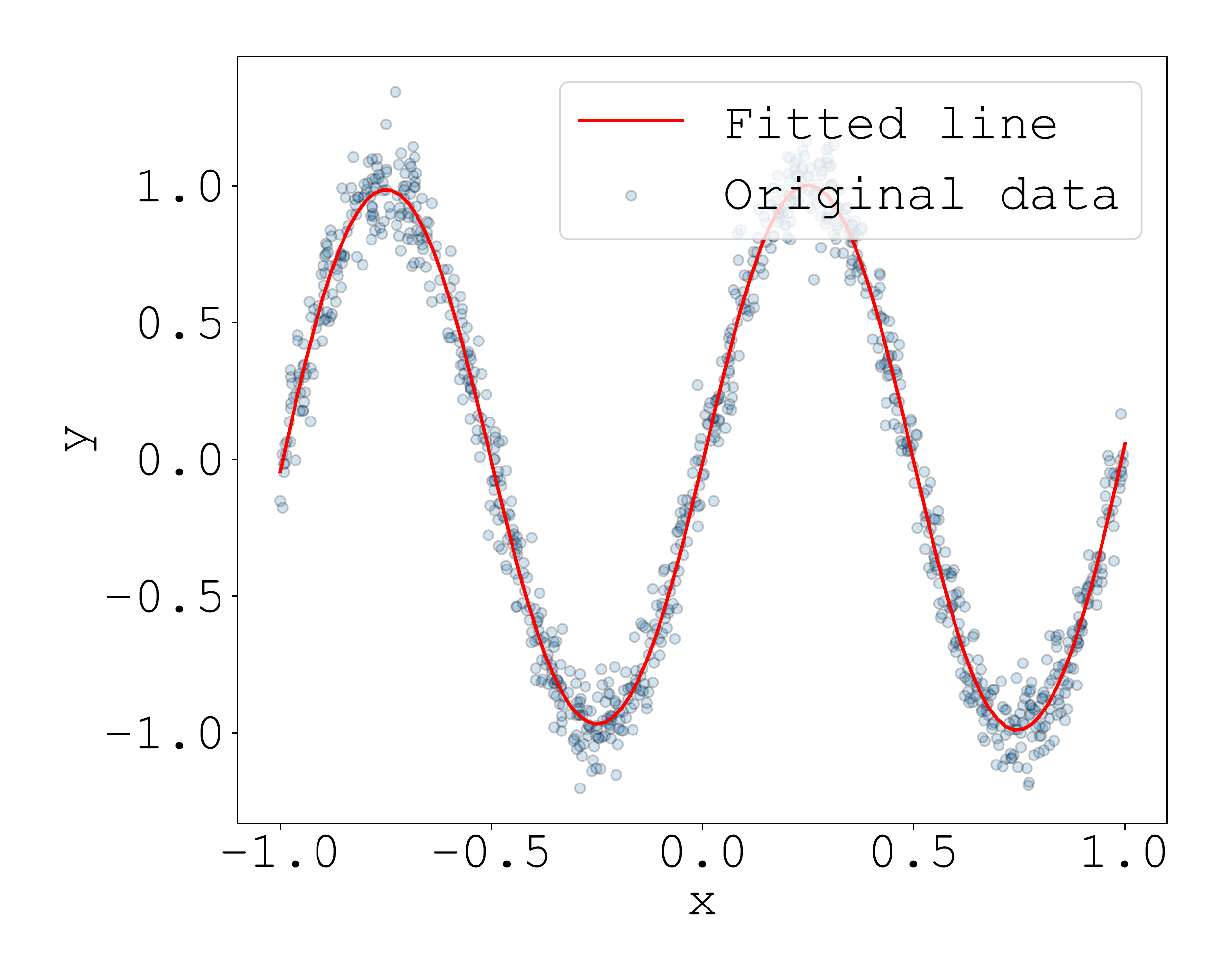}
    \end{subfigure}%
    \begin{subfigure}[c]{0.1\textwidth}
    \includegraphics[width=\linewidth]{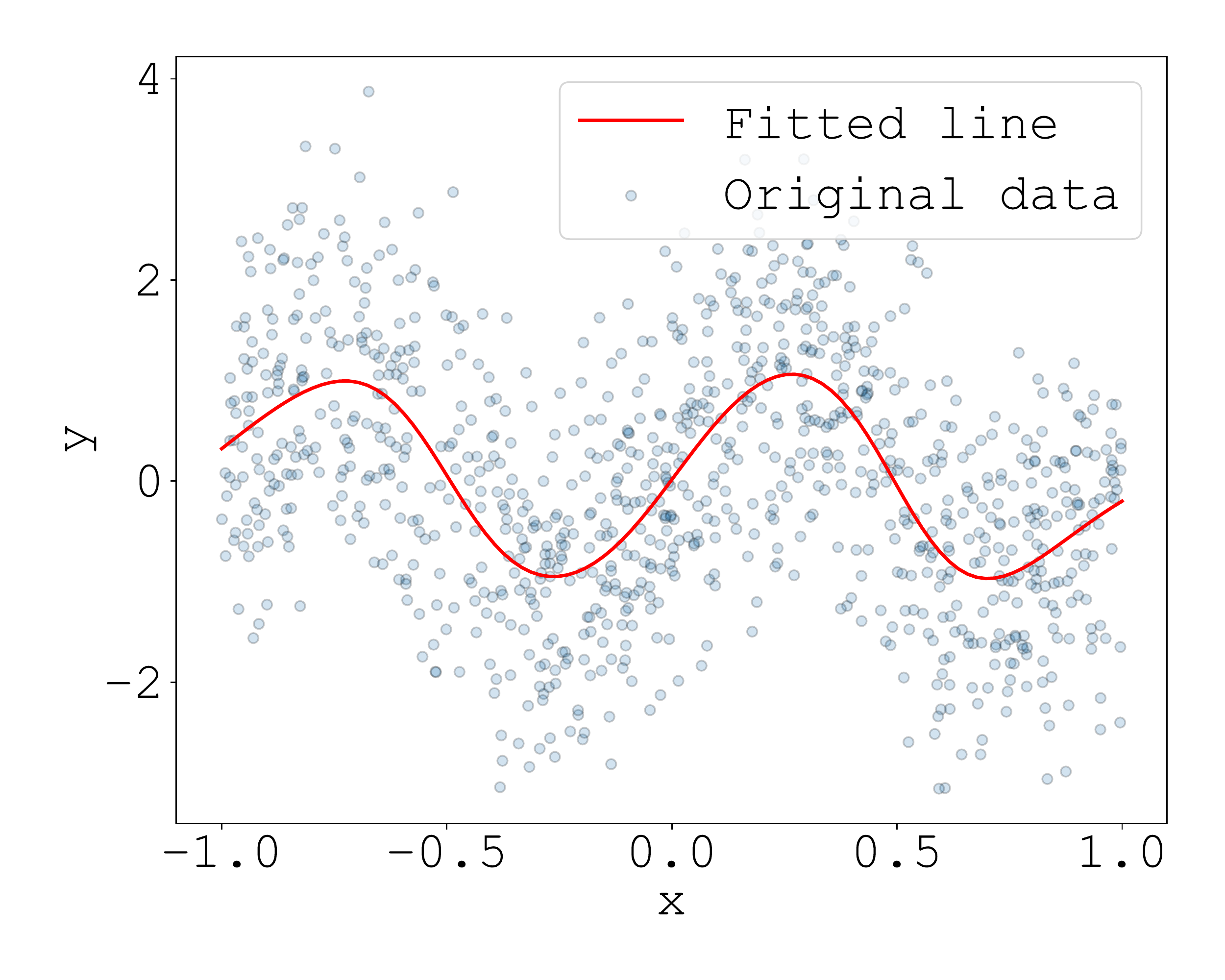}
    \end{subfigure}%
    \begin{subfigure}[c]{0.1\textwidth}
    \includegraphics[width=\linewidth]{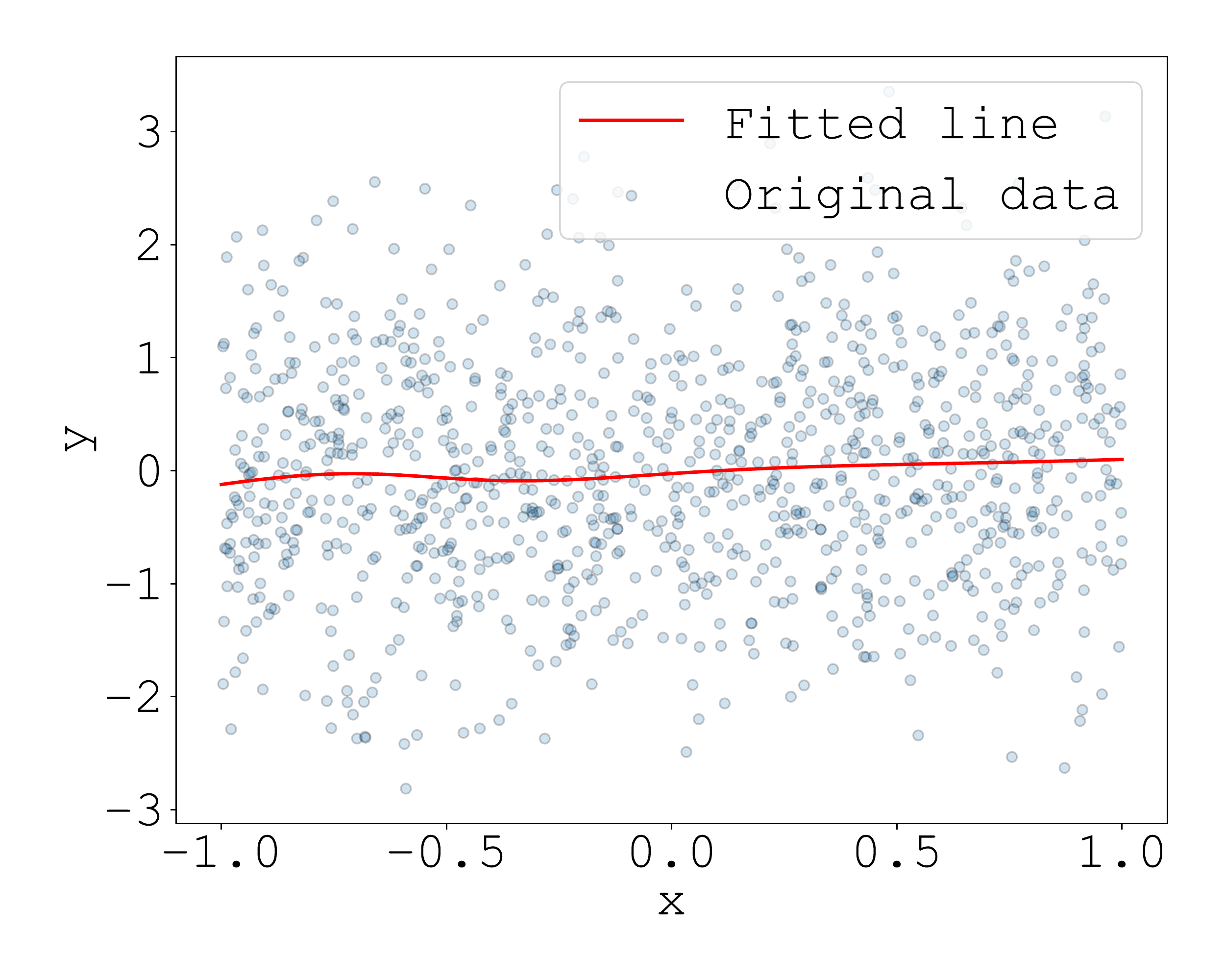}
    \end{subfigure}%
    \begin{subfigure}[c]{0.1\textwidth}
    \includegraphics[width=\linewidth]{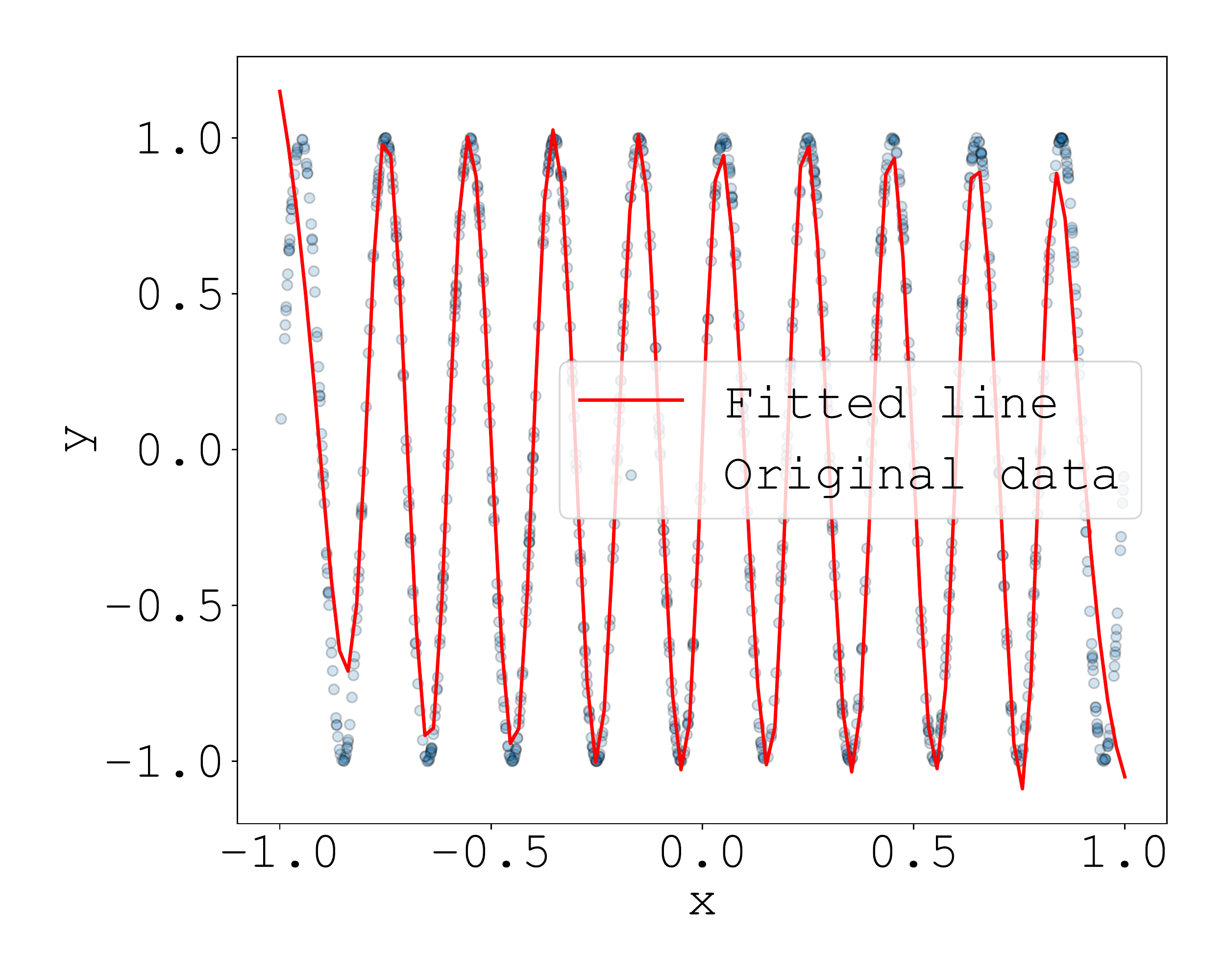}
    \end{subfigure}%
    \begin{subfigure}[c]{0.1\textwidth}
    \includegraphics[width=\linewidth]{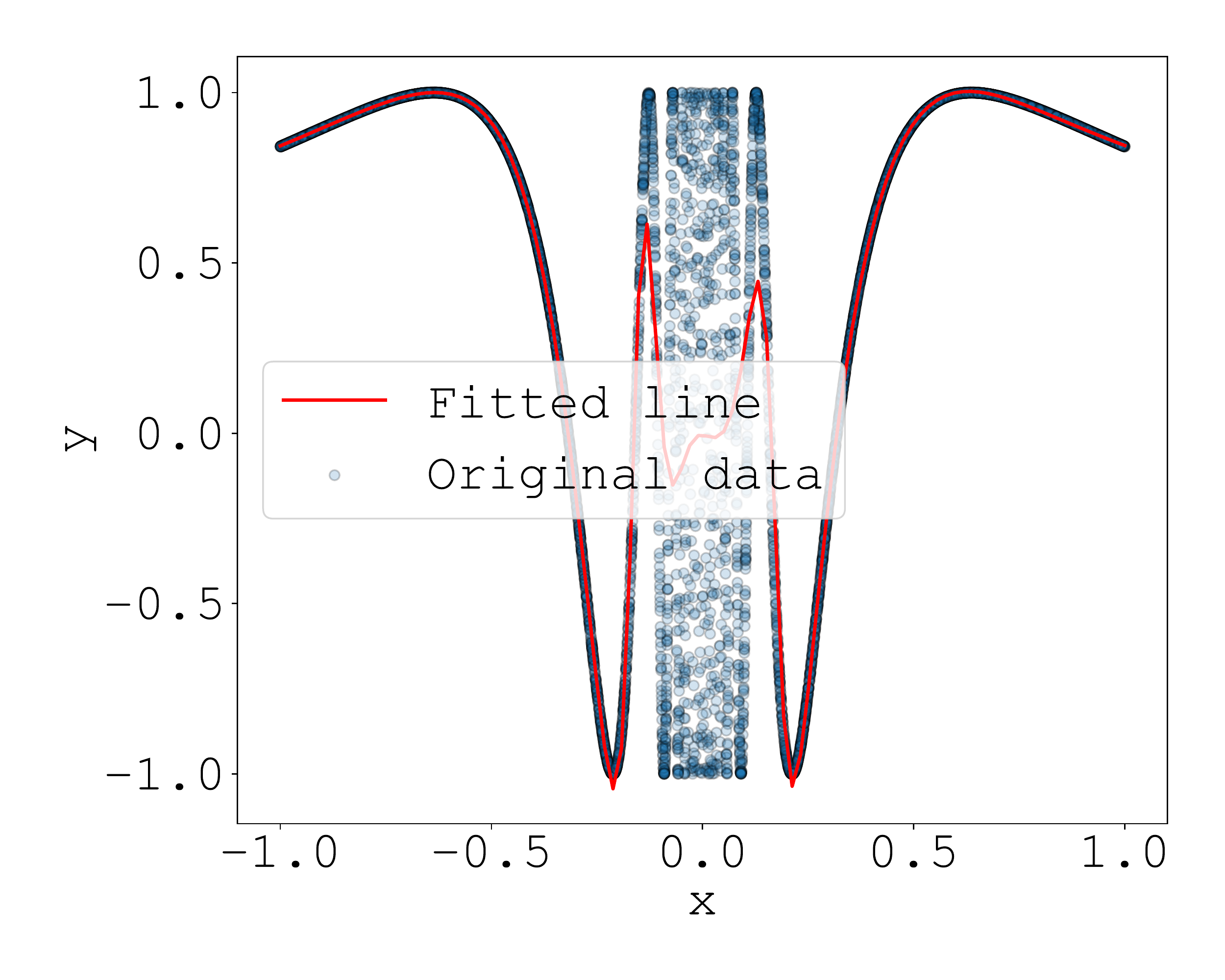}
    \end{subfigure}%
    \begin{subfigure}[c]{0.1\textwidth}
    \includegraphics[width=\linewidth]{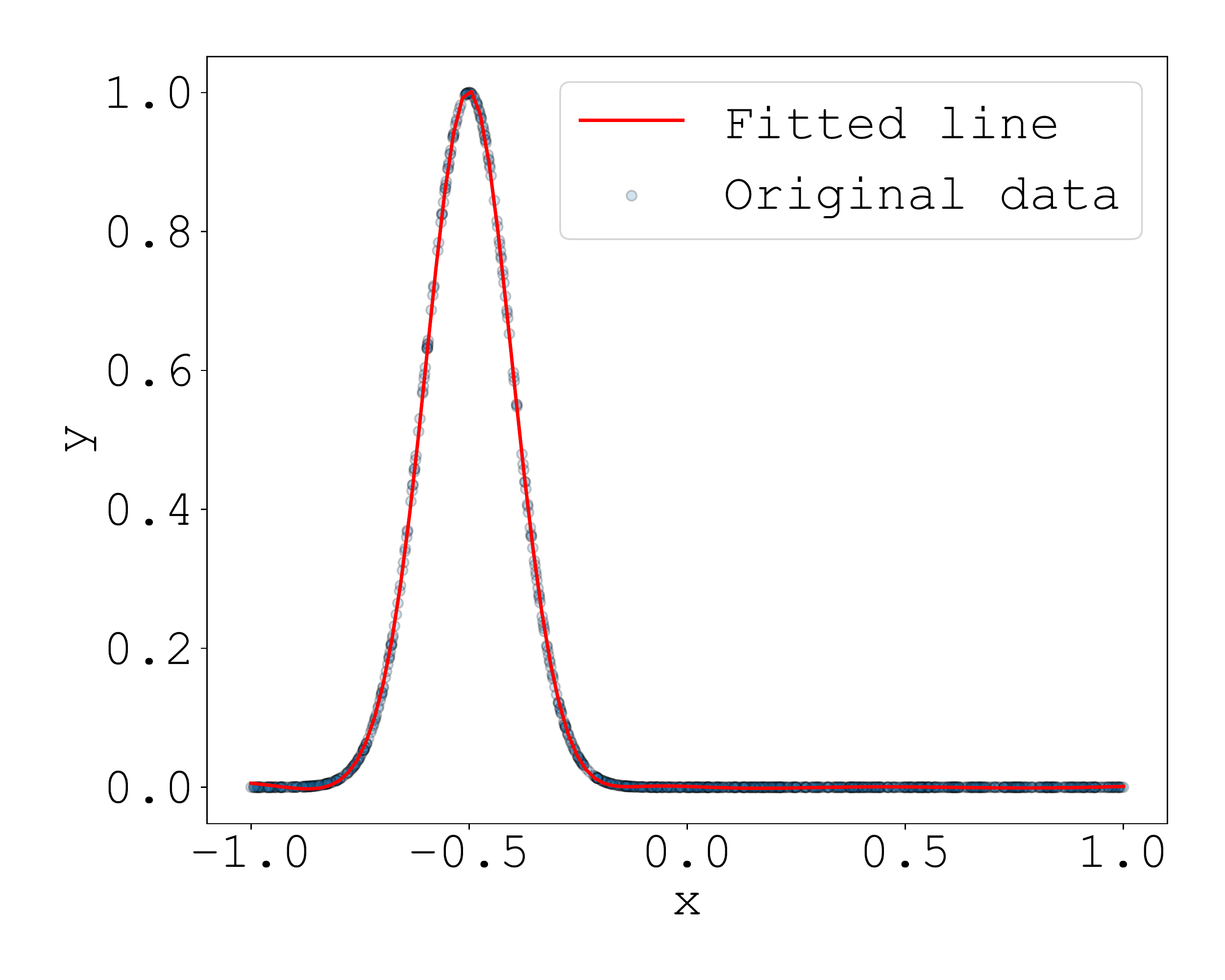}
    \end{subfigure}%
    \begin{subfigure}[c]{0.1\textwidth}
    \includegraphics[width=\linewidth]{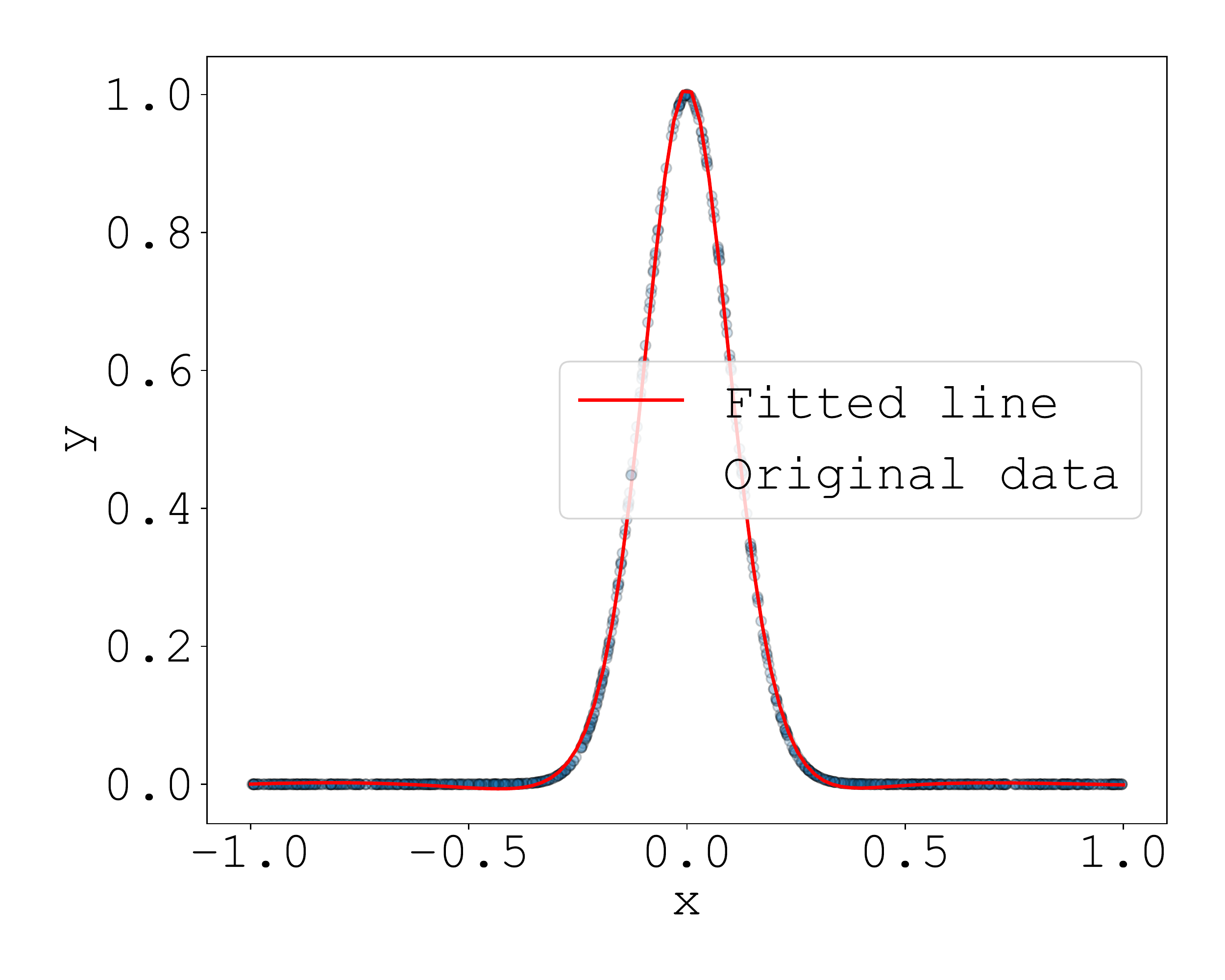}
    \end{subfigure}%
    \begin{subfigure}[c]{0.1\textwidth}
    \includegraphics[width=\linewidth]{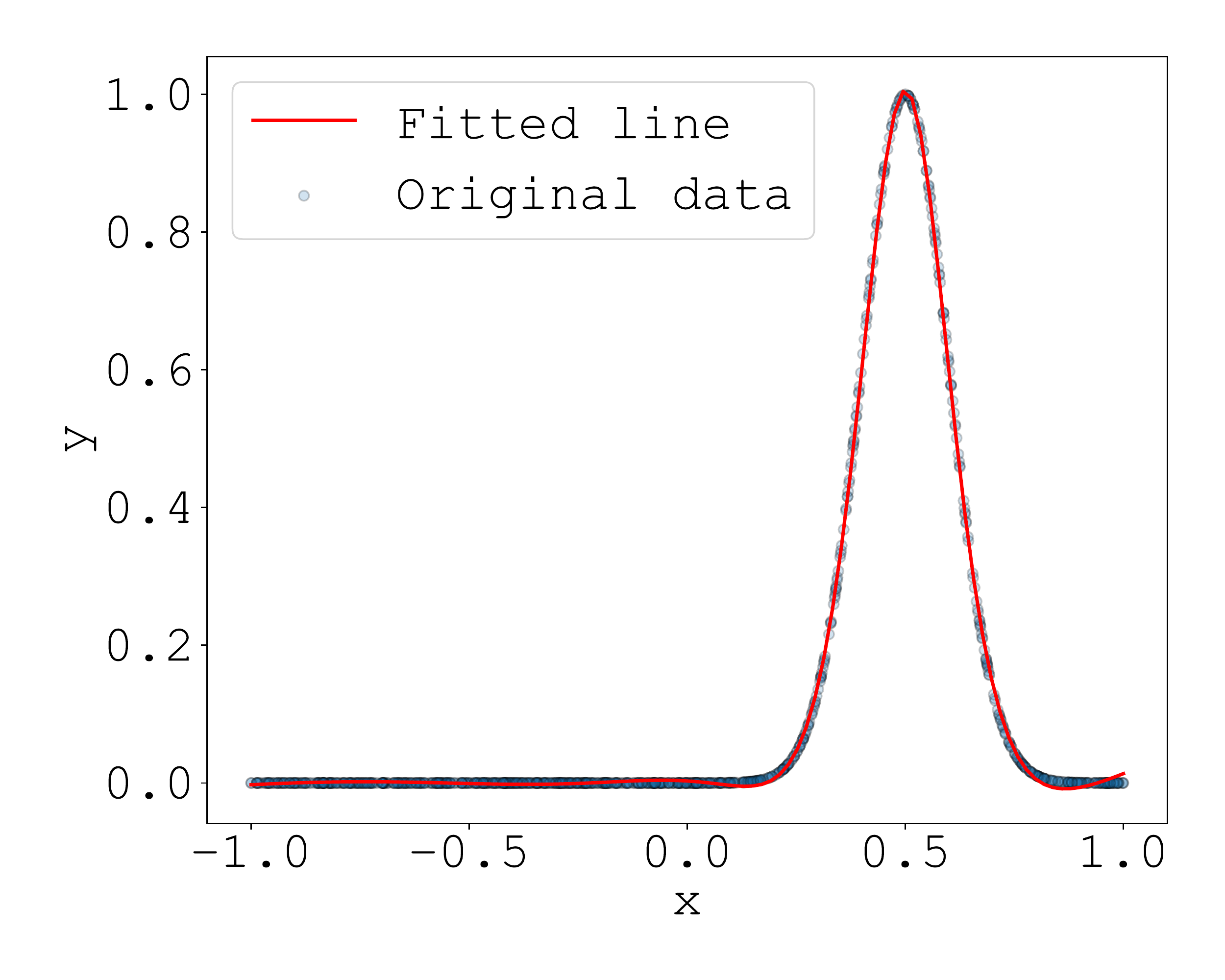}
    \end{subfigure}%
    \begin{subfigure}[c]{0.1\textwidth}
    \includegraphics[width=\linewidth]{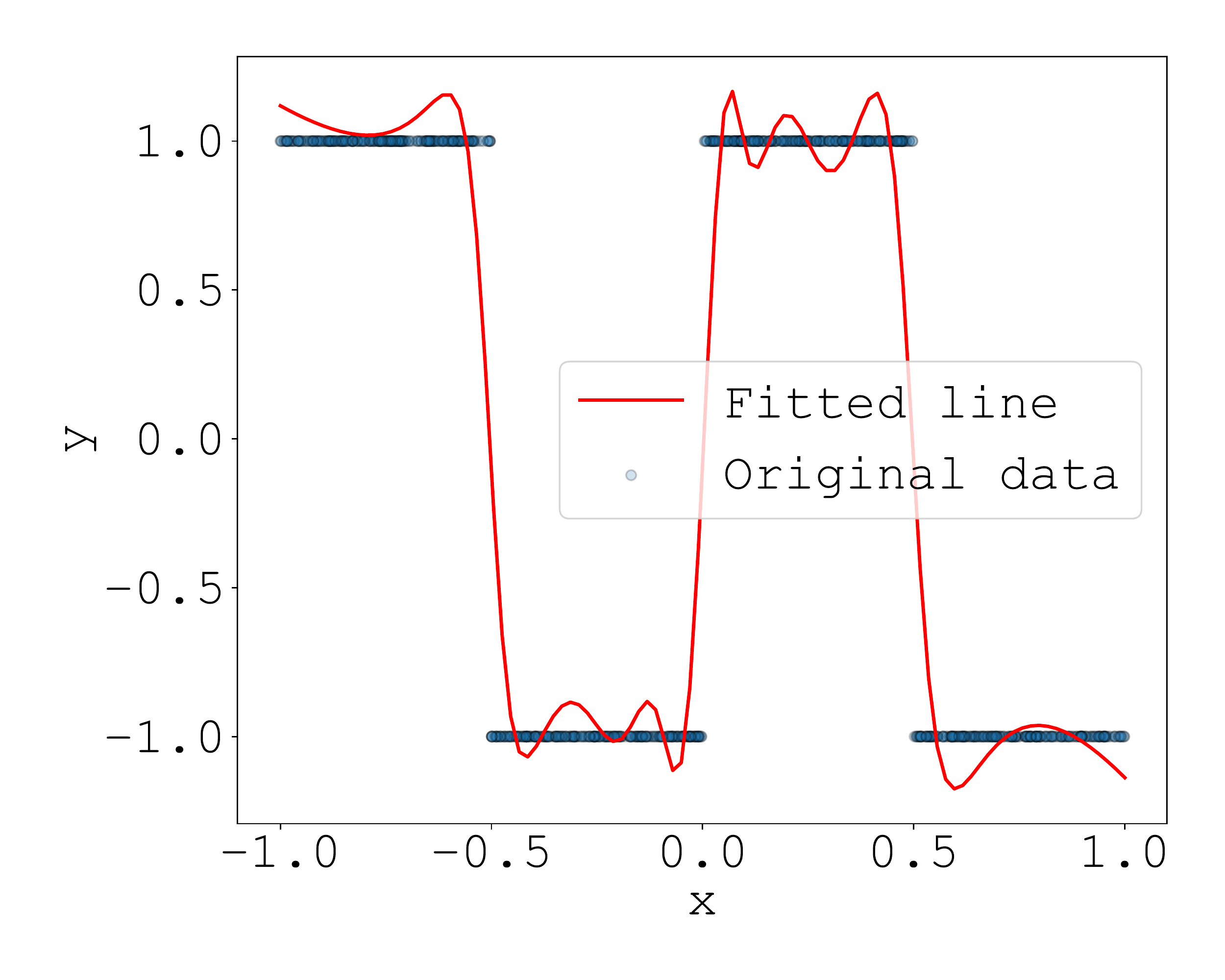}
    \end{subfigure}\\
    \begin{subfigure}[c]{0.1\textwidth}
    \includegraphics[width=\linewidth]{./exp/sin02pt_n1000_tanh_h010_adam_nn1000}
    \end{subfigure}%
    \begin{subfigure}[c]{0.1\textwidth}
    \includegraphics[width=\linewidth]{./exp/nsin02pt_e0010_n1000_tanh_h010_adam_nn1000}
    \end{subfigure}%
    \begin{subfigure}[c]{0.1\textwidth}
    \includegraphics[width=\linewidth]{./exp/nsin02pt_e0100_n1000_tanh_h010_adam_nn1000}
    \end{subfigure}%
    \begin{subfigure}[c]{0.1\textwidth}
    \includegraphics[width=\linewidth]{./exp/gauss_m00_sd01_n1000_tanh_h010_adam_nn1000}
    \end{subfigure}%
    \begin{subfigure}[c]{0.1\textwidth}
    \includegraphics[width=\linewidth]{./exp/sin10pt_n1000_tanh_h100_adam_nn1000}
    \end{subfigure}%
    \begin{subfigure}[c]{0.1\textwidth}
    \includegraphics[width=\linewidth]{./exp/topsin_n10000_tanh_h100_adam_nn1000}
    \end{subfigure}%
    \begin{subfigure}[c]{0.1\textwidth}
    \includegraphics[width=\linewidth]{./exp/rbf_pos-050_n1000_tanh_h010_adam_nn1000}
    \end{subfigure}%
    \begin{subfigure}[c]{0.1\textwidth}
    \includegraphics[width=\linewidth]{./exp/rbf_pos0000_n1000_tanh_h010_adam_nn1000}
    \end{subfigure}%
    \begin{subfigure}[c]{0.1\textwidth}
    \includegraphics[width=\linewidth]{./exp/rbf_pos0050_n1000_tanh_h010_adam_nn1000}
    \end{subfigure}%
    \begin{subfigure}[c]{0.1\textwidth}
    \includegraphics[width=\linewidth]{./exp/sq02pt_n1000_tanh_h100_adam_nn1000}
    \end{subfigure}\\
    \begin{subfigure}[c]{0.1\textwidth}
    \includegraphics[width=\linewidth]{./exp/sin02pt_n1000_tanh_2d}
    \caption{}
    \end{subfigure}%
    \begin{subfigure}[c]{0.1\textwidth}
    \includegraphics[width=\linewidth]{./exp/nsin02pt_e0010_n1000_tanh_2d}
    \caption{}
    \end{subfigure}%
    \begin{subfigure}[c]{0.1\textwidth}
    \includegraphics[width=\linewidth]{./exp/nsin02pt_e0100_n1000_tanh_2d}
    \caption{}
    \end{subfigure}%
    \begin{subfigure}[c]{0.1\textwidth}
    \includegraphics[width=\linewidth]{./exp/gauss_m00_sd01_n1000_tanh_2d}
    \caption{}
    \end{subfigure}%
    \begin{subfigure}[c]{0.1\textwidth}
    \includegraphics[width=\linewidth]{./exp/sin10pt_n1000_tanh_2d}
    \caption{}
    \end{subfigure}%
    \begin{subfigure}[c]{0.1\textwidth}
    \includegraphics[width=\linewidth]{./exp/topsin_n10000_tanh_2d}
    \caption{}
    \end{subfigure}%
    \begin{subfigure}[c]{0.1\textwidth}
    \includegraphics[width=\linewidth]{./exp/rbf_pos-050_n1000_tanh_2d}
    \caption{}
    \end{subfigure}%
    \begin{subfigure}[c]{0.1\textwidth}
    \includegraphics[width=\linewidth]{./exp/rbf_pos0000_n1000_tanh_2d}
    \caption{}
    \end{subfigure}%
    \begin{subfigure}[c]{0.1\textwidth}
    \includegraphics[width=\linewidth]{./exp/rbf_pos0050_n1000_tanh_2d}
    \caption{}
    \end{subfigure}%
    \begin{subfigure}[c]{0.1\textwidth}
    \includegraphics[width=\linewidth]{./exp/sq02pt_n1000_tanh_2d}
    \caption{}
    \end{subfigure}\\
\caption{Experimental results (activation:$\tanh$, optimization:ADAM). (a-c) Sinusoidal Curve (with Noise) $\sigma = 0, 0.1$ and $1.0$, (d) Gaussian Noise, (e) High Frequency Sinusoidal Curve, (f) Topologist's Sinusoidal Curve, (g-i) Gaussian Kernel $\mu = -0.5, 0.0$ and $.5$, (j) Square Wave. See supplementary materials for all the examples with larger images.}
\label{fig:results_tanh_adam}
\end{figure*}

%% file: results_all.tex
\section{Further Examples}
We present all the experimental results described in \refsec{motivation}.
The results are presented as $10$ sets of subfigures.
A single set corresponds to one of the $10$ dataset,
and it contains a scatter plot of dataset,
four scatter plots of parameters,
and two spectra.

On the whole, the scatter plots with ADAM are sharper than those with BFGS,
and those with ReLU are also sharper than those with $\tanh$.
These results suggest the implicit regularization property of both ADAM and ReLU.
The scatter plot with ADAM tends to be sharp because less parameters can remain to be trained for the varieties of mini batches.
The scatter plot with ReLU tends to be sharp because ReLU is homogeneous and thus two different parameters $(a,b)$ and $(ka,kb)$, where $k > 0,$ indicate the same basis function $\sigma(a x - b)$.

\begin{figure}[h]
    \begin{center}
    \begin{subfigure}[c]{0.66\textwidth}
    \includegraphics[width=\linewidth, trim=0cm 0cm 0cm 0cm, clip]{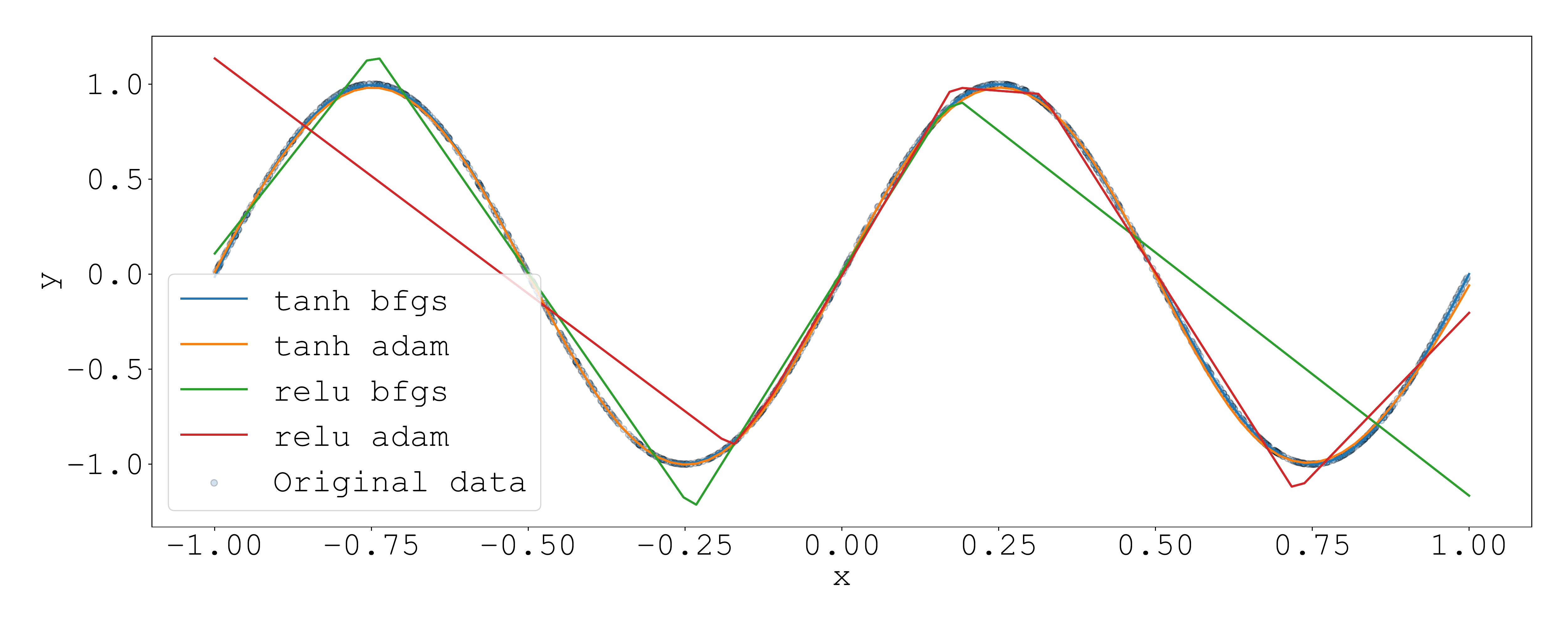}
    \caption{dataset}
    \end{subfigure}
    \end{center}
    \begin{subfigure}[c]{0.33\textwidth}
    \includegraphics[width=\linewidth, trim=1cm 0cm 1cm 1cm, clip]{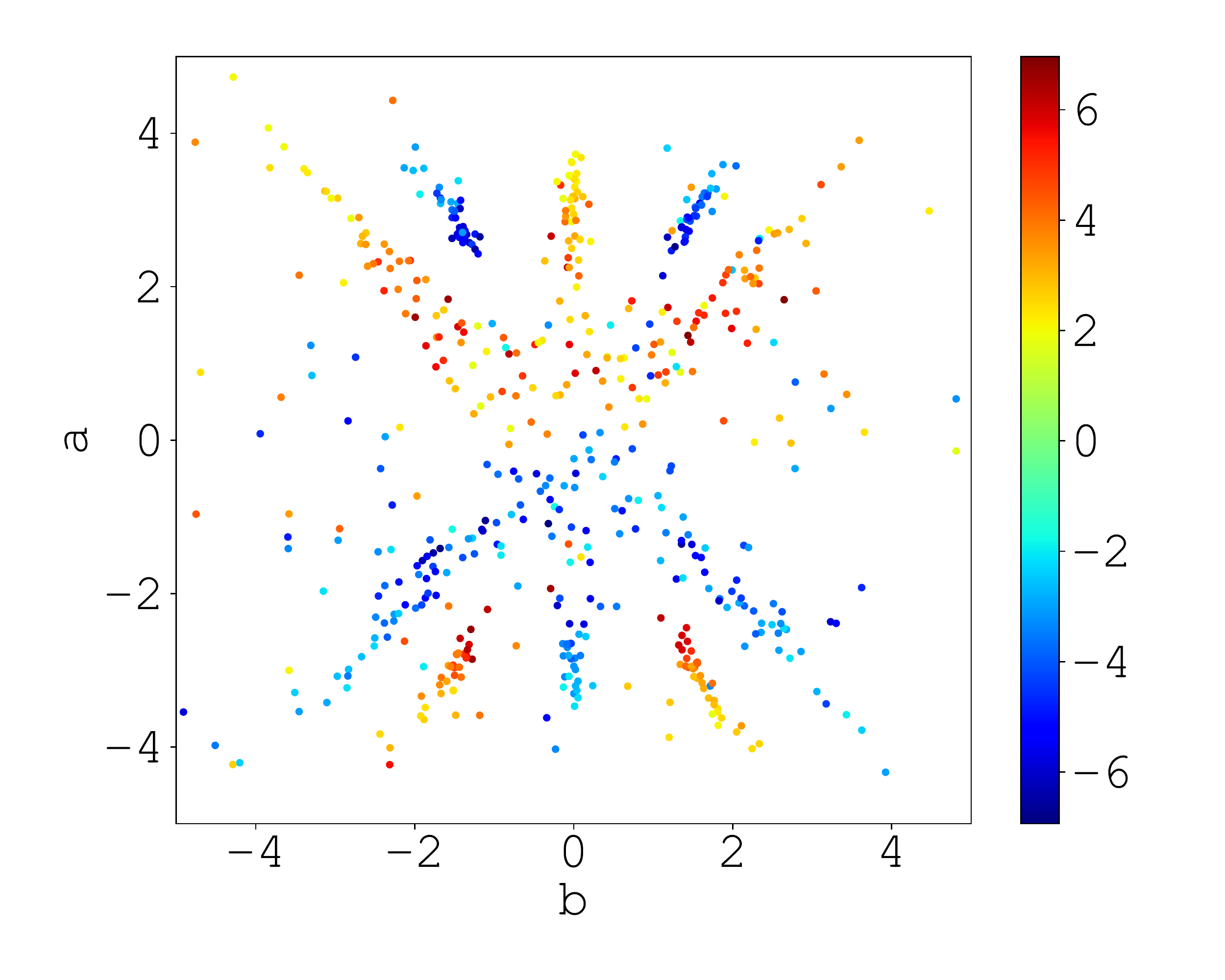}
    \caption{tanh, bfgs}
    \end{subfigure}%
    \begin{subfigure}[c]{0.33\textwidth}
    \includegraphics[width=\linewidth, trim=1cm 0cm 1cm 1cm, clip]{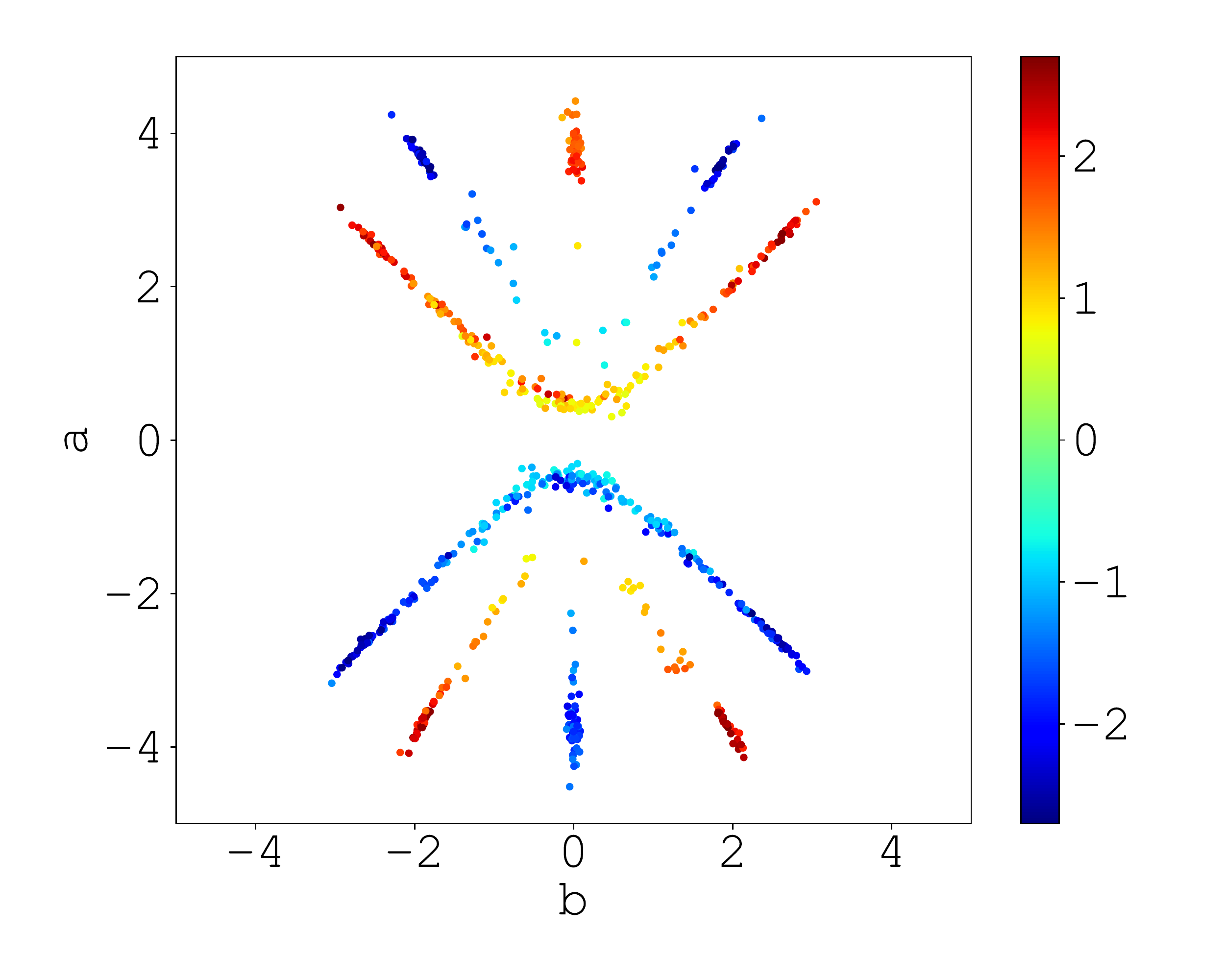}
    \caption{tanh, adam}
    \end{subfigure}%
    \begin{subfigure}[c]{0.33\textwidth}
    \includegraphics[width=\linewidth, trim=1cm 0cm 1cm 1cm, clip]{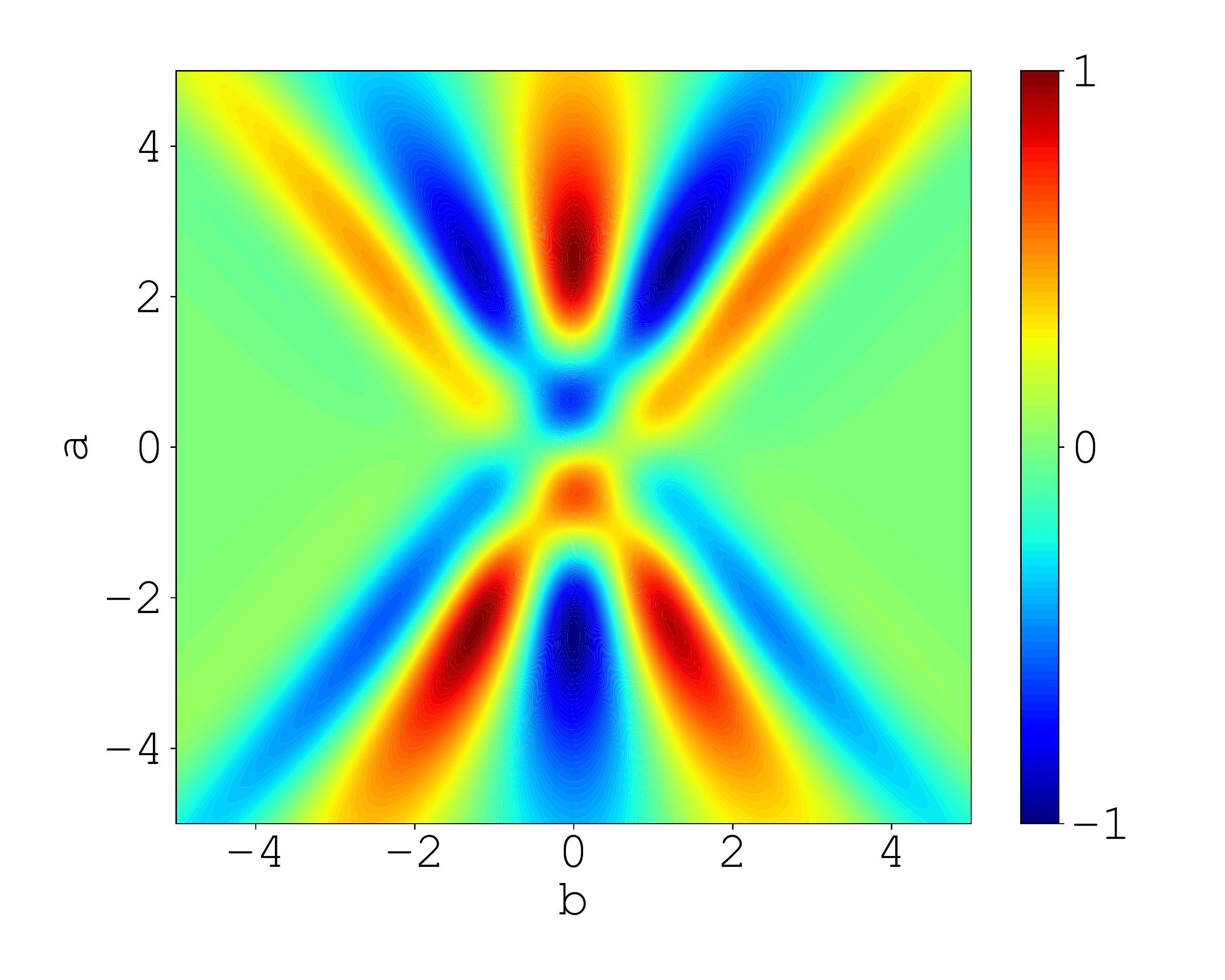}
    \caption{tanh}
    \end{subfigure}\\
    \begin{subfigure}[c]{0.33\textwidth}
    \includegraphics[width=\linewidth, trim=1cm 0cm 1cm 1cm, clip]{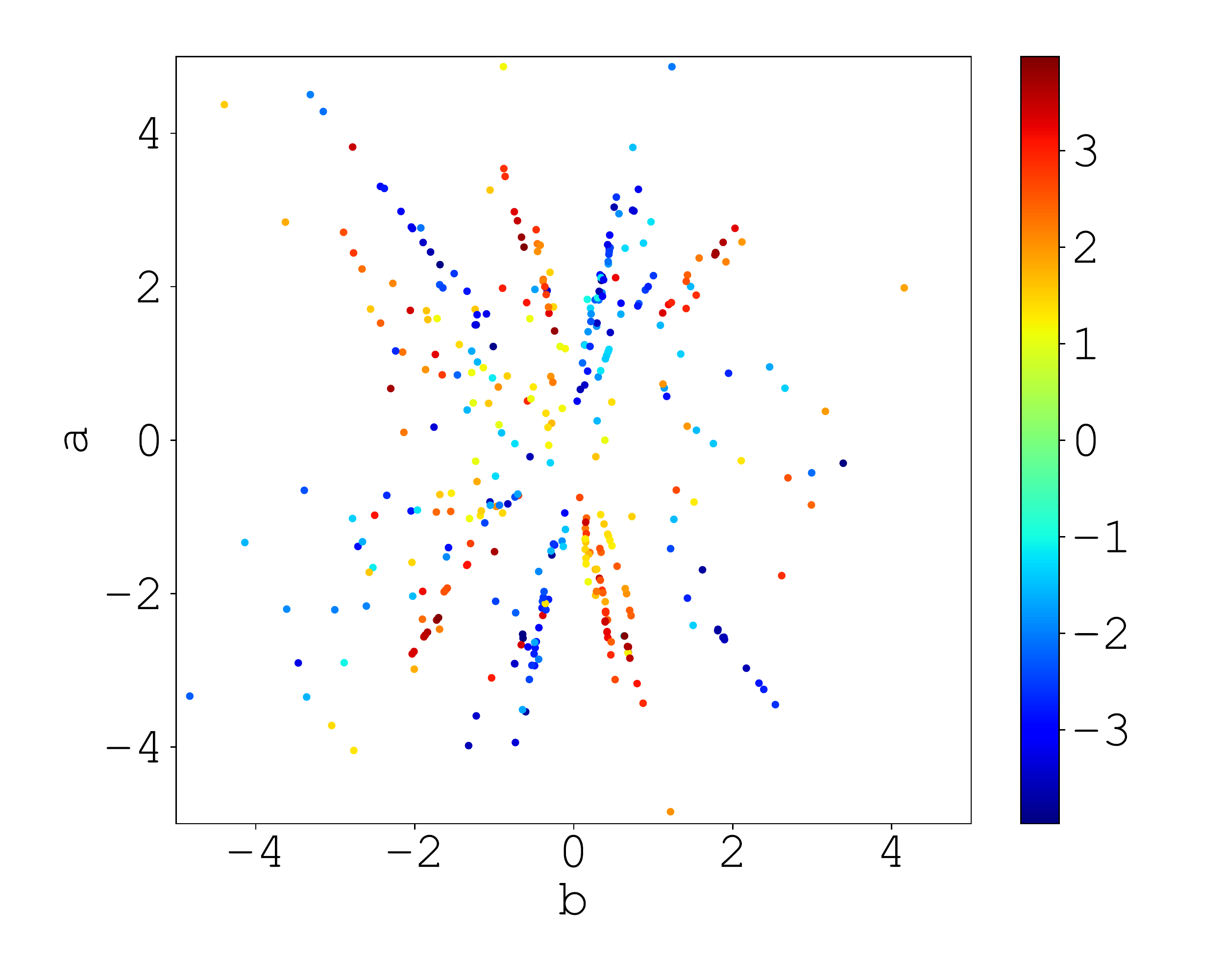}
    \caption{relu, bfgs}
    \end{subfigure}%
    \begin{subfigure}[c]{0.33\textwidth}
    \includegraphics[width=\linewidth, trim=1cm 0cm 1cm 1cm, clip]{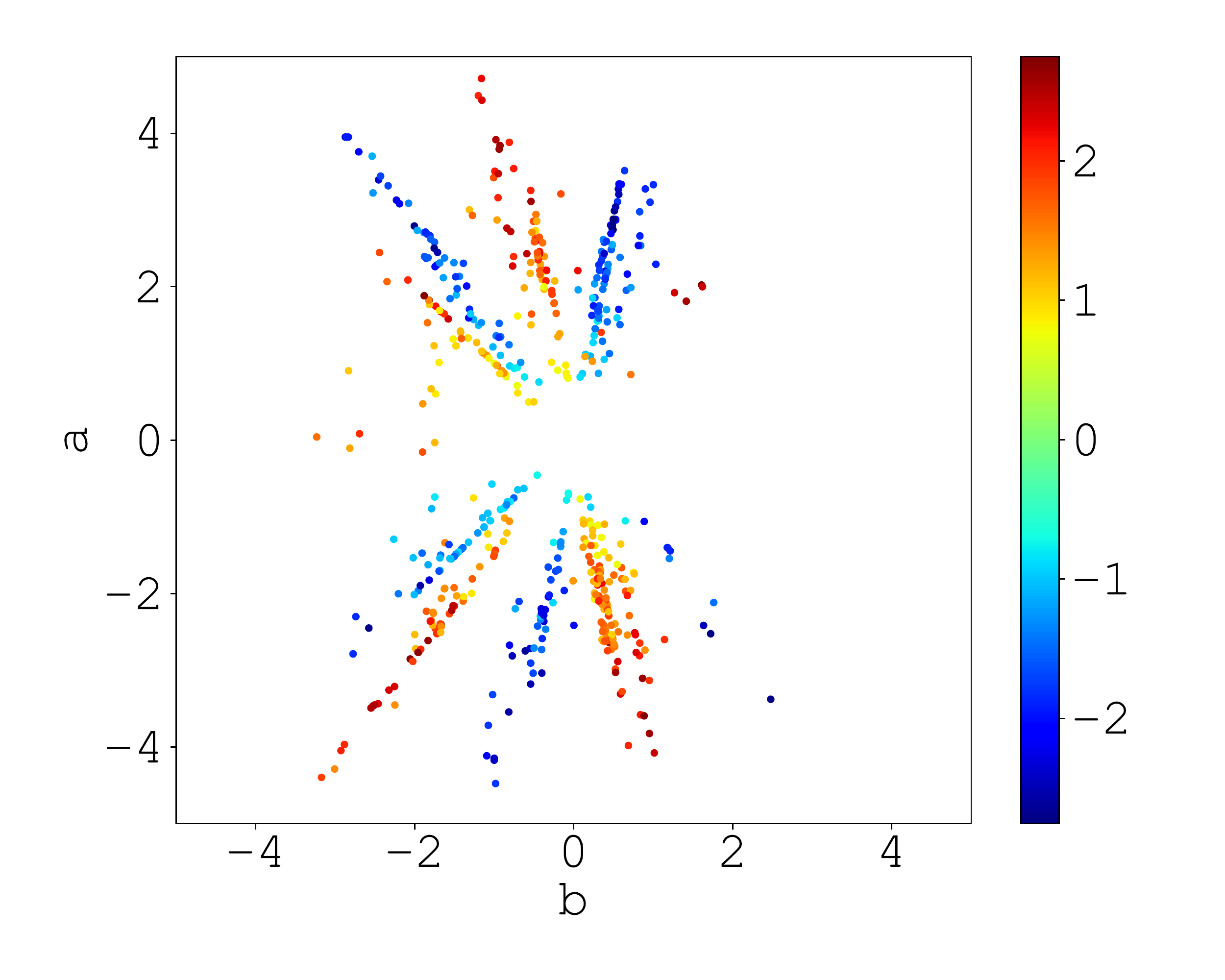}
    \caption{relu, adam}
    \end{subfigure}%
    \begin{subfigure}[c]{0.33\textwidth}
    \includegraphics[width=\linewidth, trim=1cm 0cm 1cm 1cm, clip]{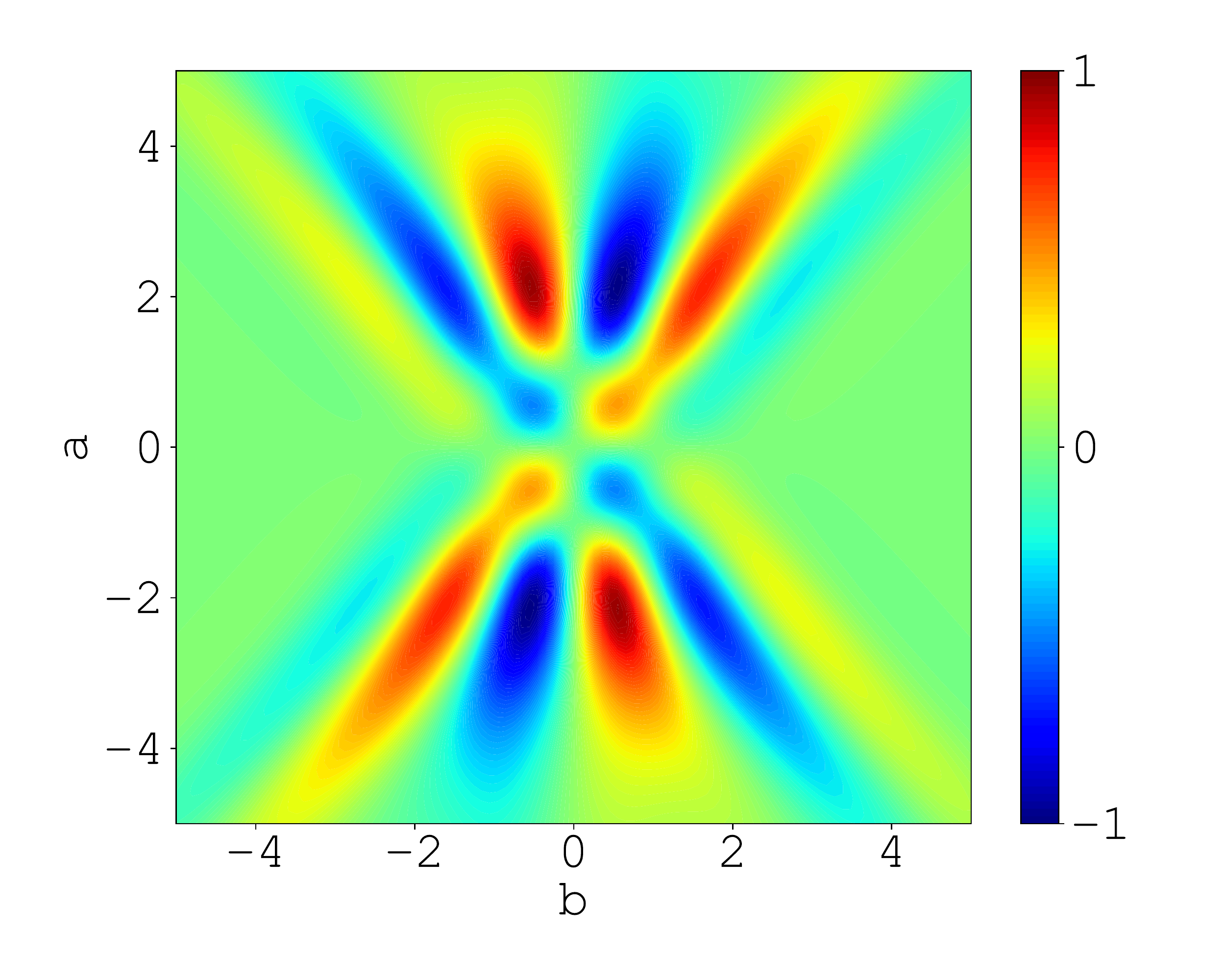}
    \caption{relu}
    \end{subfigure}\\
\caption{Sinusoidal Curve}
\label{fig:sin02pt.n0000}
\end{figure}

\begin{figure}[h]
    \begin{center}
    \begin{subfigure}[c]{0.66\textwidth}
    \includegraphics[width=\linewidth, trim=0cm 0cm 0cm 0cm, clip]{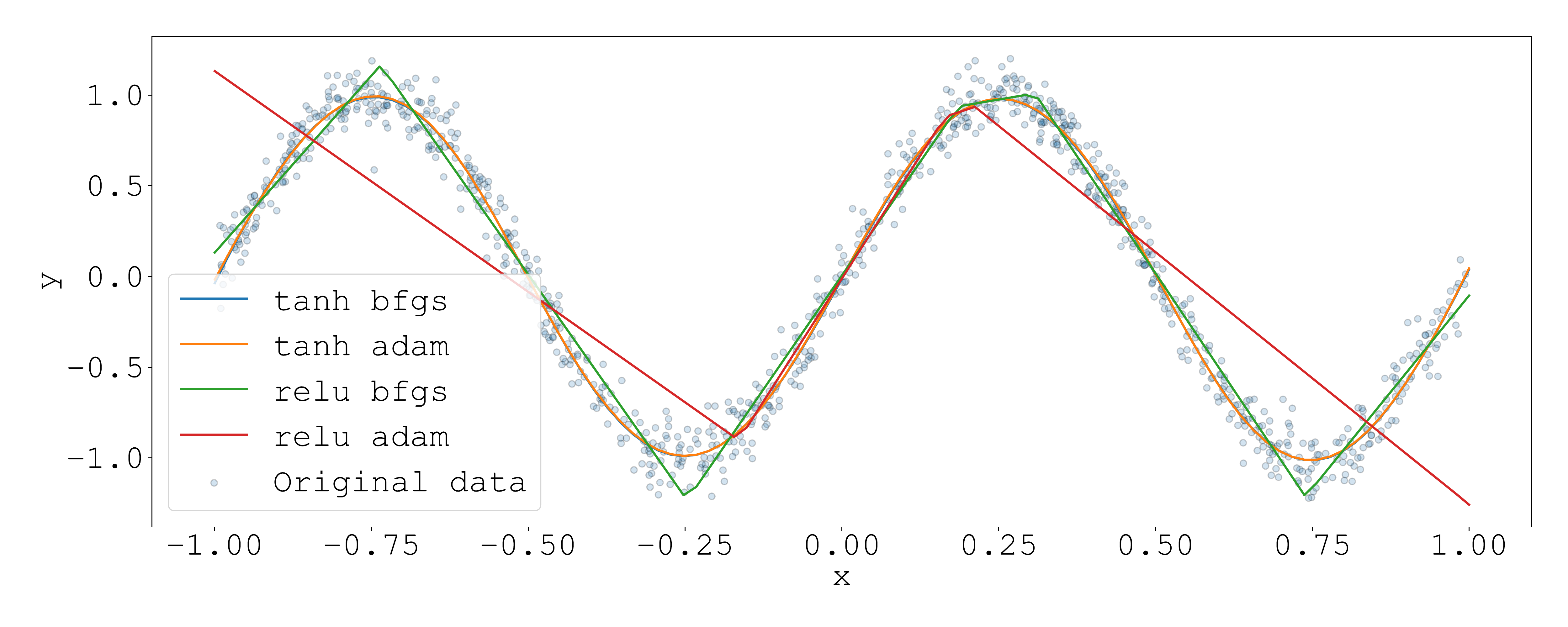}
    \caption{dataset}
    \end{subfigure}
    \end{center}
    \begin{subfigure}[c]{0.33\textwidth}
    \includegraphics[width=\linewidth, trim=1cm 0cm 1cm 1cm, clip]{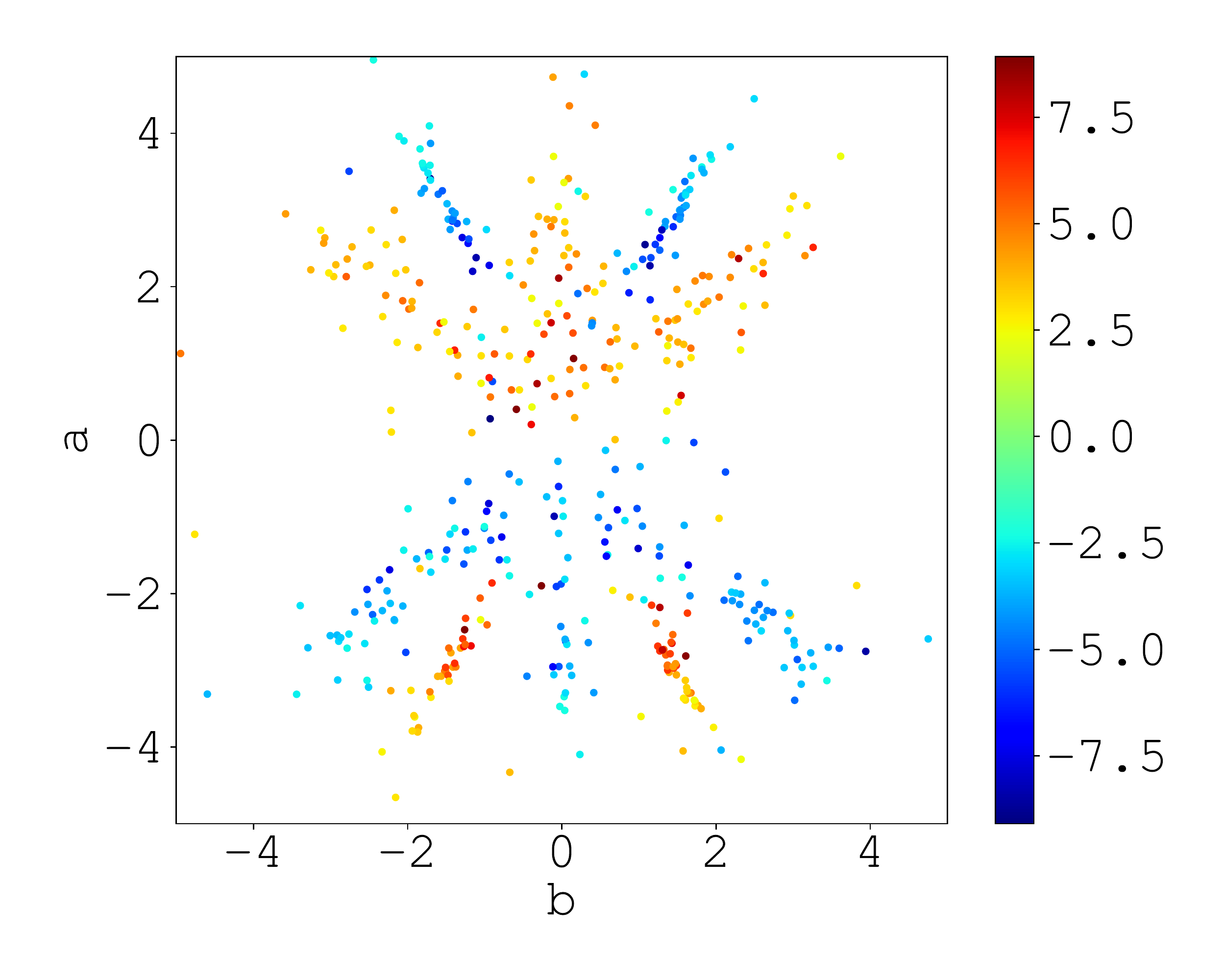}
    \caption{tanh, bfgs}
    \end{subfigure}%
    \begin{subfigure}[c]{0.33\textwidth}
    \includegraphics[width=\linewidth, trim=1cm 0cm 1cm 1cm, clip]{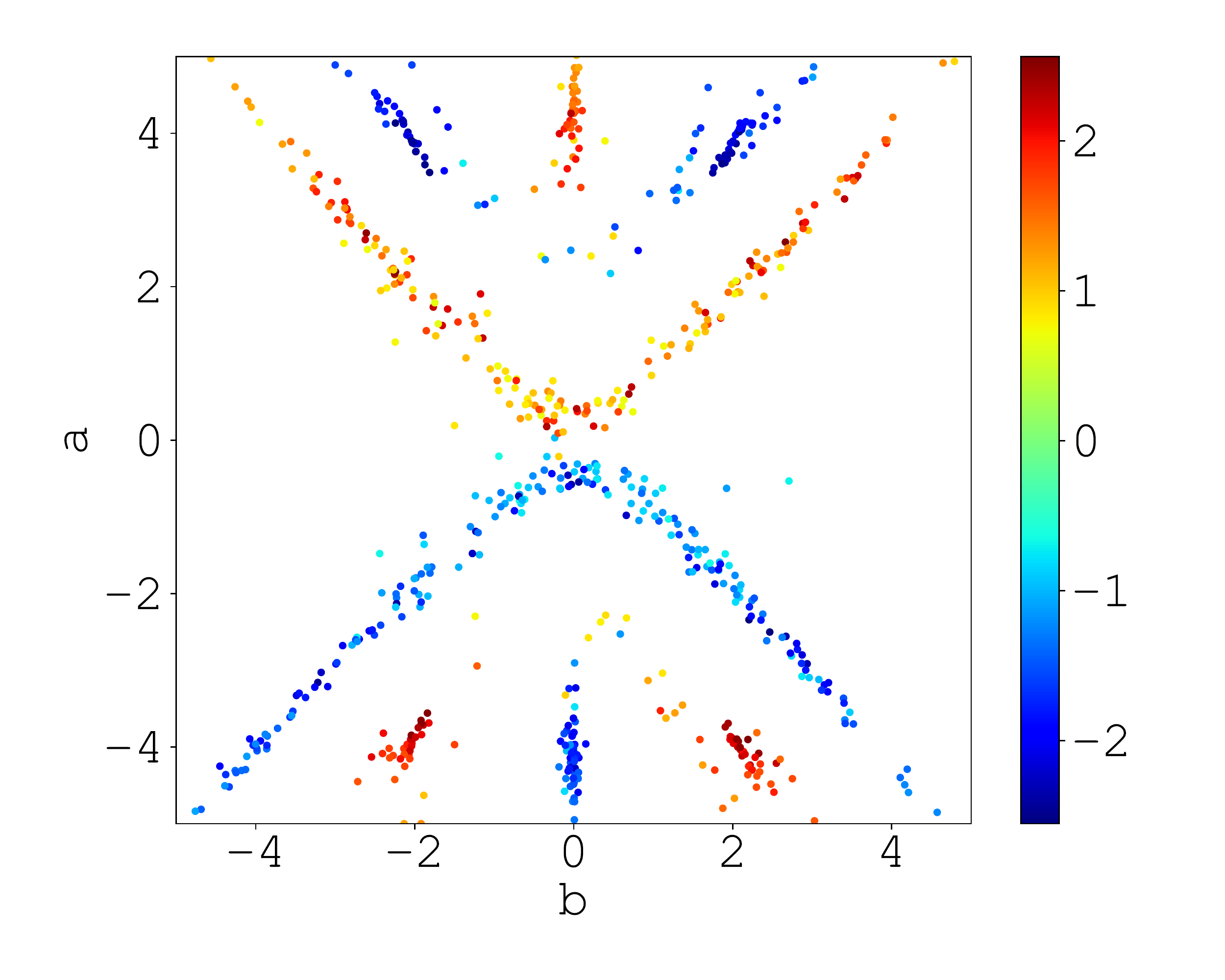}
    \caption{tanh, adam}
    \end{subfigure}%
    \begin{subfigure}[c]{0.33\textwidth}
    \includegraphics[width=\linewidth, trim=1cm 0cm 1cm 1cm, clip]{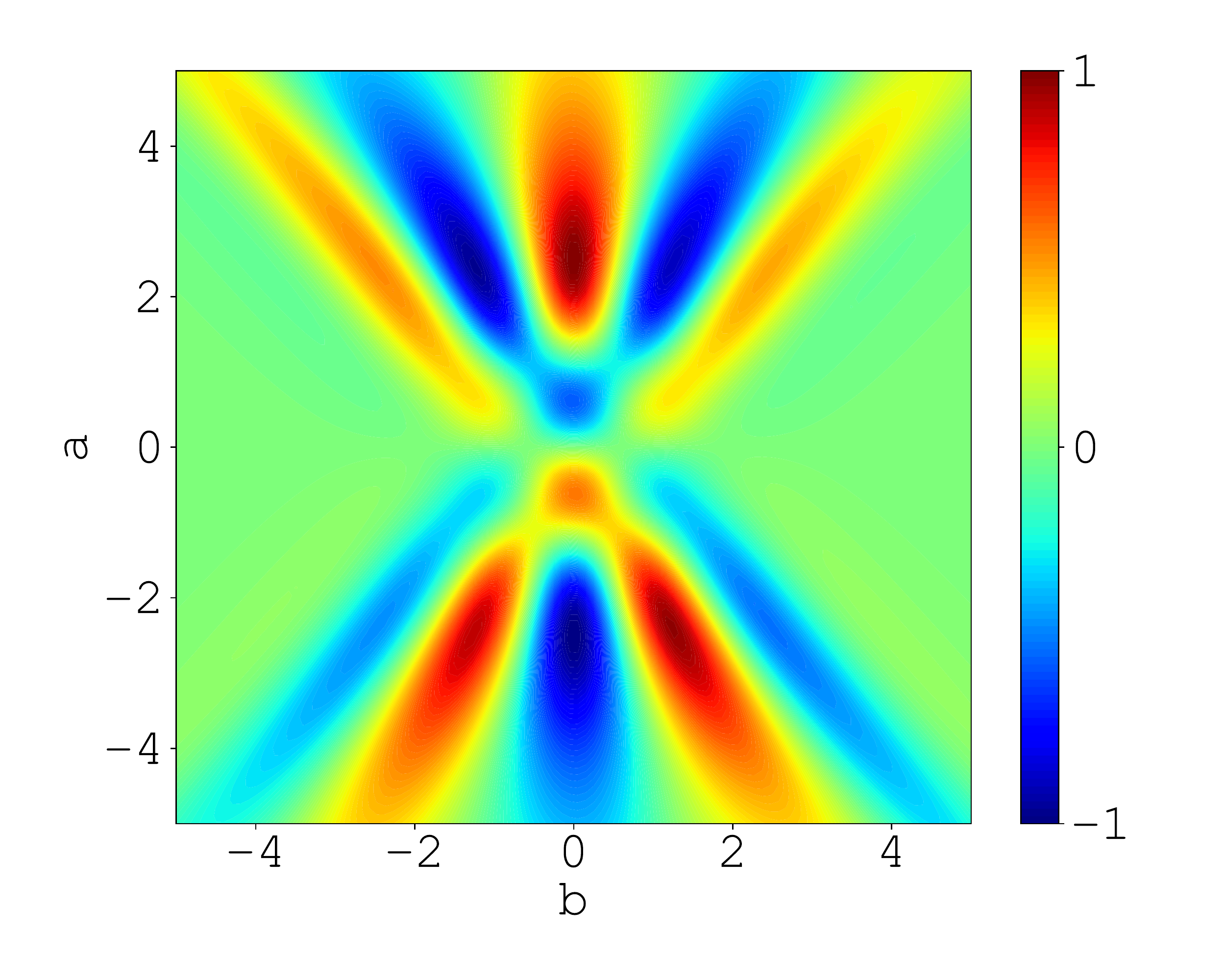}
    \caption{tanh}
    \end{subfigure}\\
    \begin{subfigure}[c]{0.33\textwidth}
    \includegraphics[width=\linewidth, trim=1cm 0cm 1cm 1cm, clip]{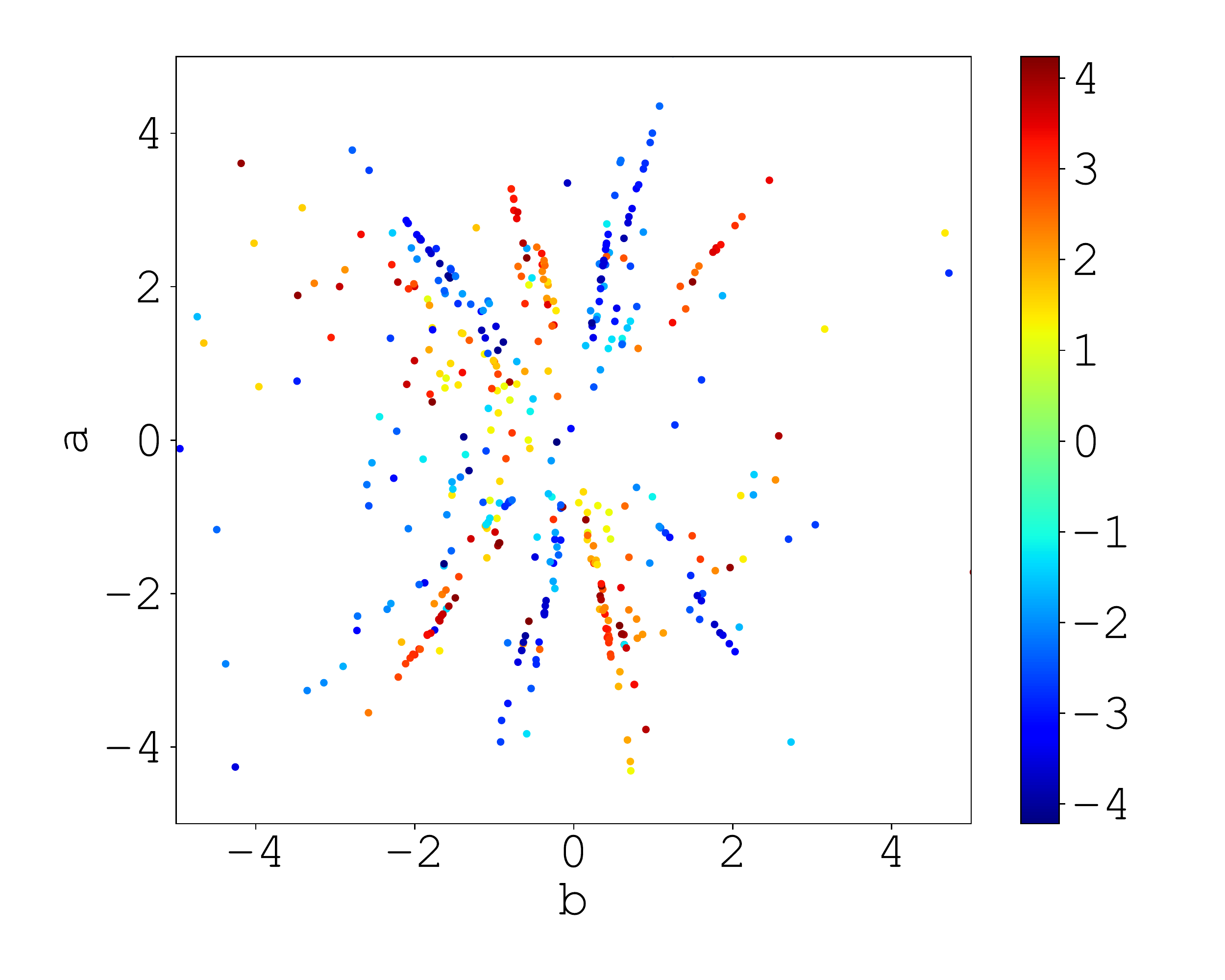}
    \caption{relu, bfgs}
    \end{subfigure}%
    \begin{subfigure}[c]{0.33\textwidth}
    \includegraphics[width=\linewidth, trim=1cm 0cm 1cm 1cm, clip]{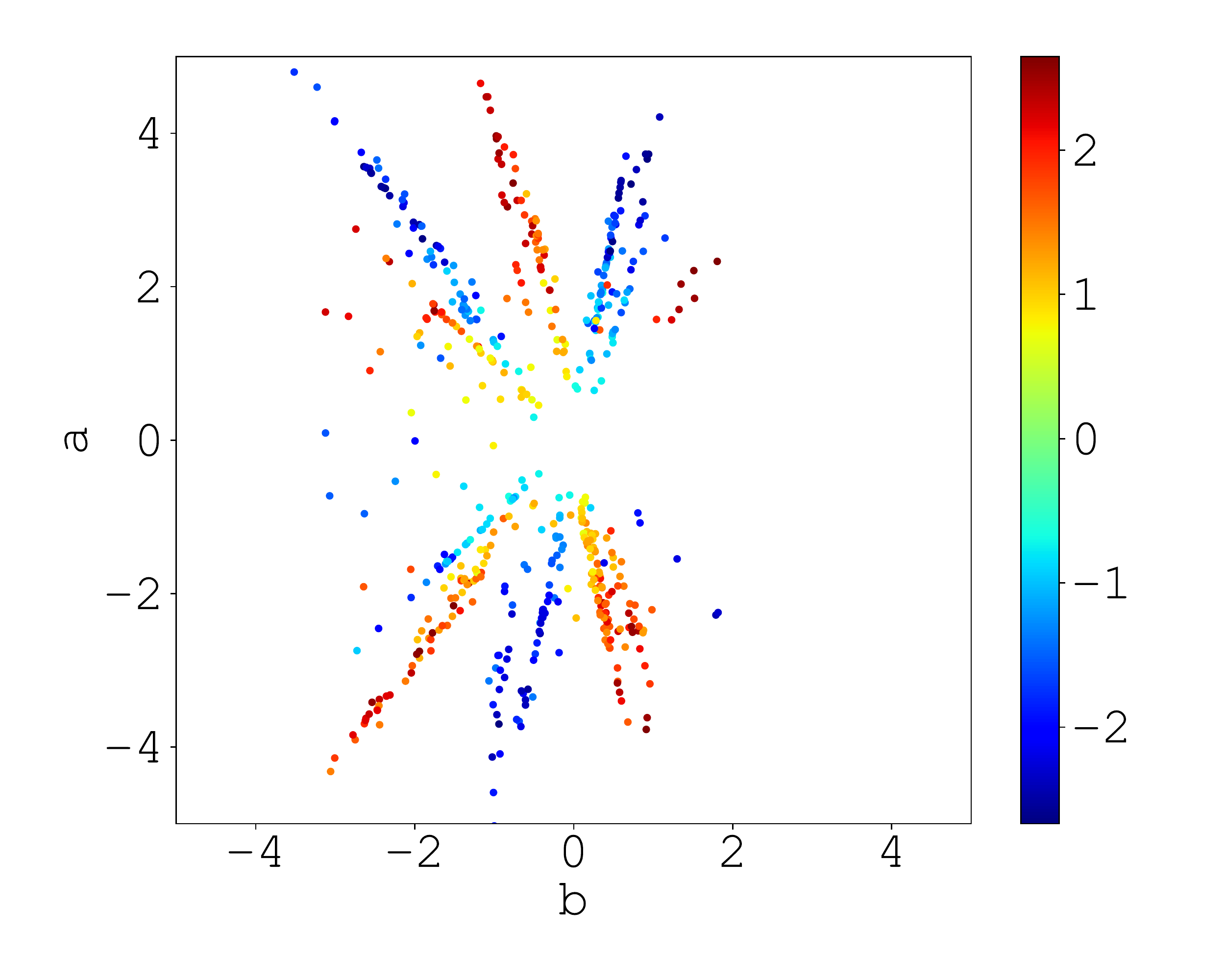}
    \caption{relu, adam}
    \end{subfigure}%
    \begin{subfigure}[c]{0.33\textwidth}
    \includegraphics[width=\linewidth, trim=1cm 0cm 1cm 1cm, clip]{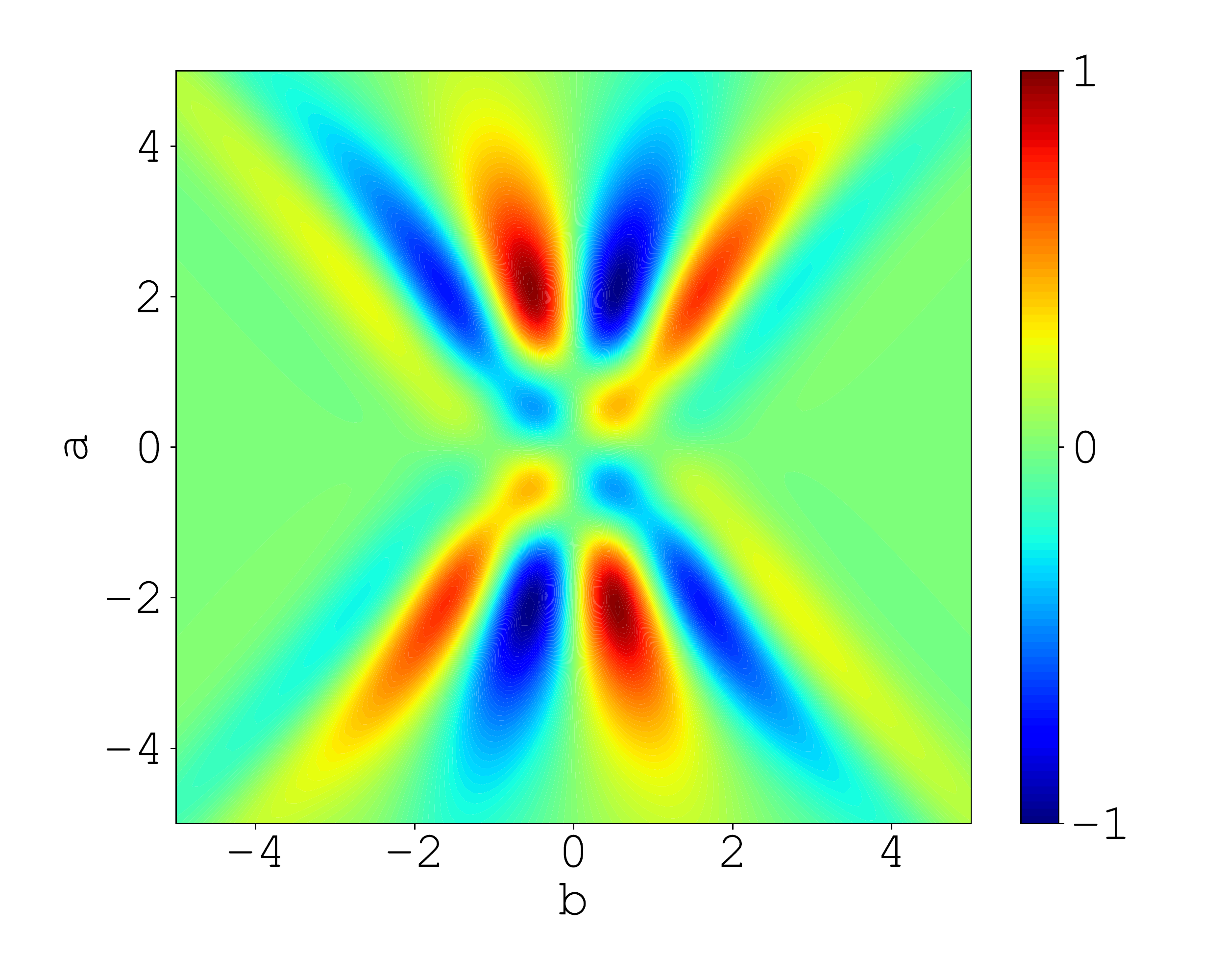}
    \caption{relu}
    \end{subfigure}\\
\caption{Sinusoidal Curve with Gaussian Noise $N(0, 0.1^2)$}
\label{fig:nsin02pt.e0010.n0000}
\end{figure}

\begin{figure}[h]
    \begin{center}
    \begin{subfigure}[c]{0.66\textwidth}
    \includegraphics[width=\linewidth, trim=0cm 0cm 0cm 0cm, clip]{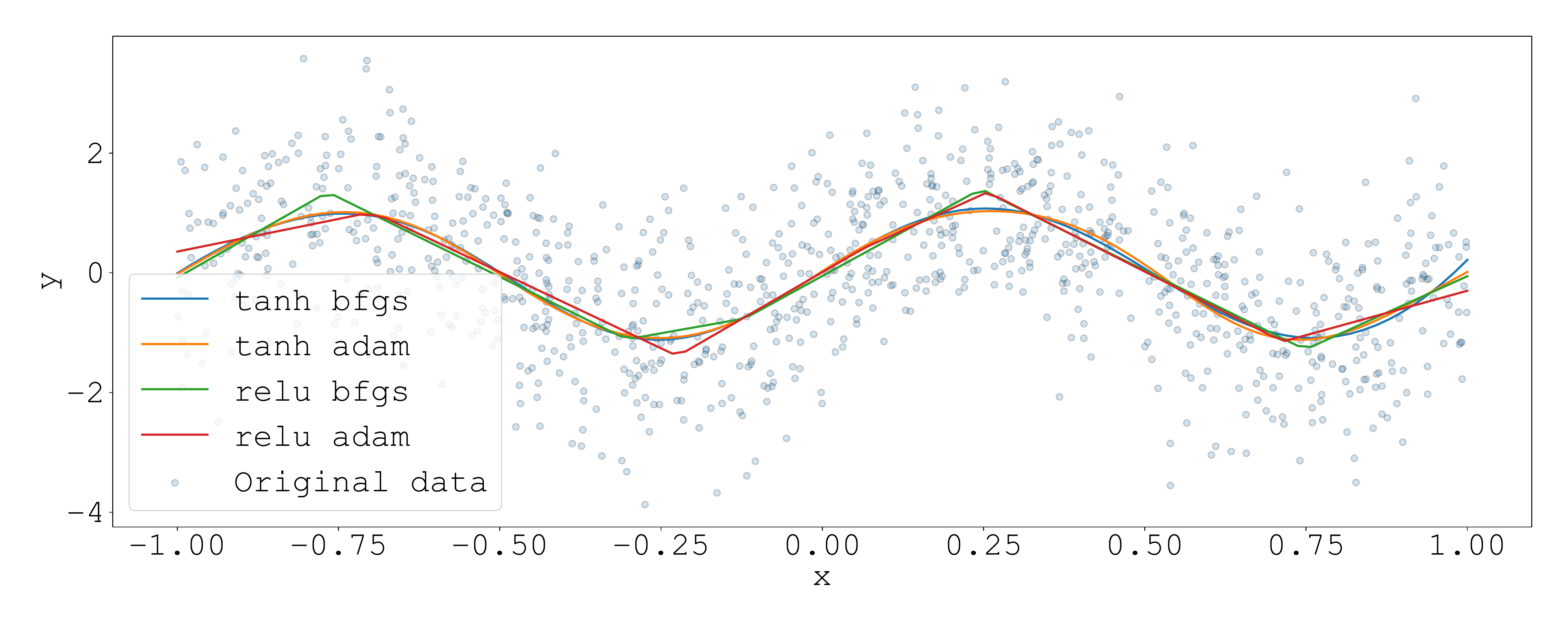}
    \caption{dataset}
    \end{subfigure}
    \end{center}
    \begin{subfigure}[c]{0.33\textwidth}
    \includegraphics[width=\linewidth, trim=1cm 0cm 1cm 1cm, clip]{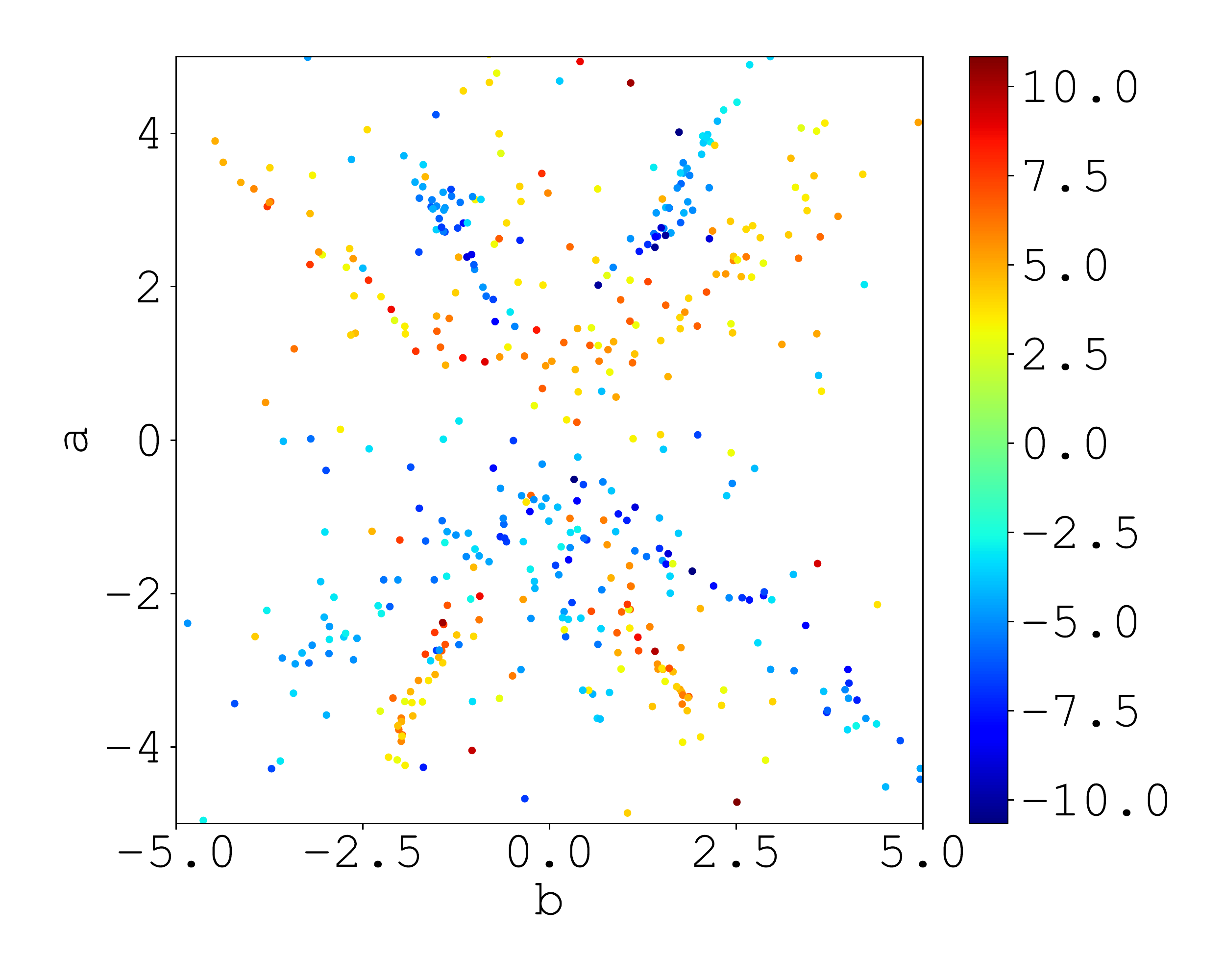}
    \caption{tanh, bfgs}
    \end{subfigure}%
    \begin{subfigure}[c]{0.33\textwidth}
    \includegraphics[width=\linewidth, trim=1cm 0cm 1cm 1cm, clip]{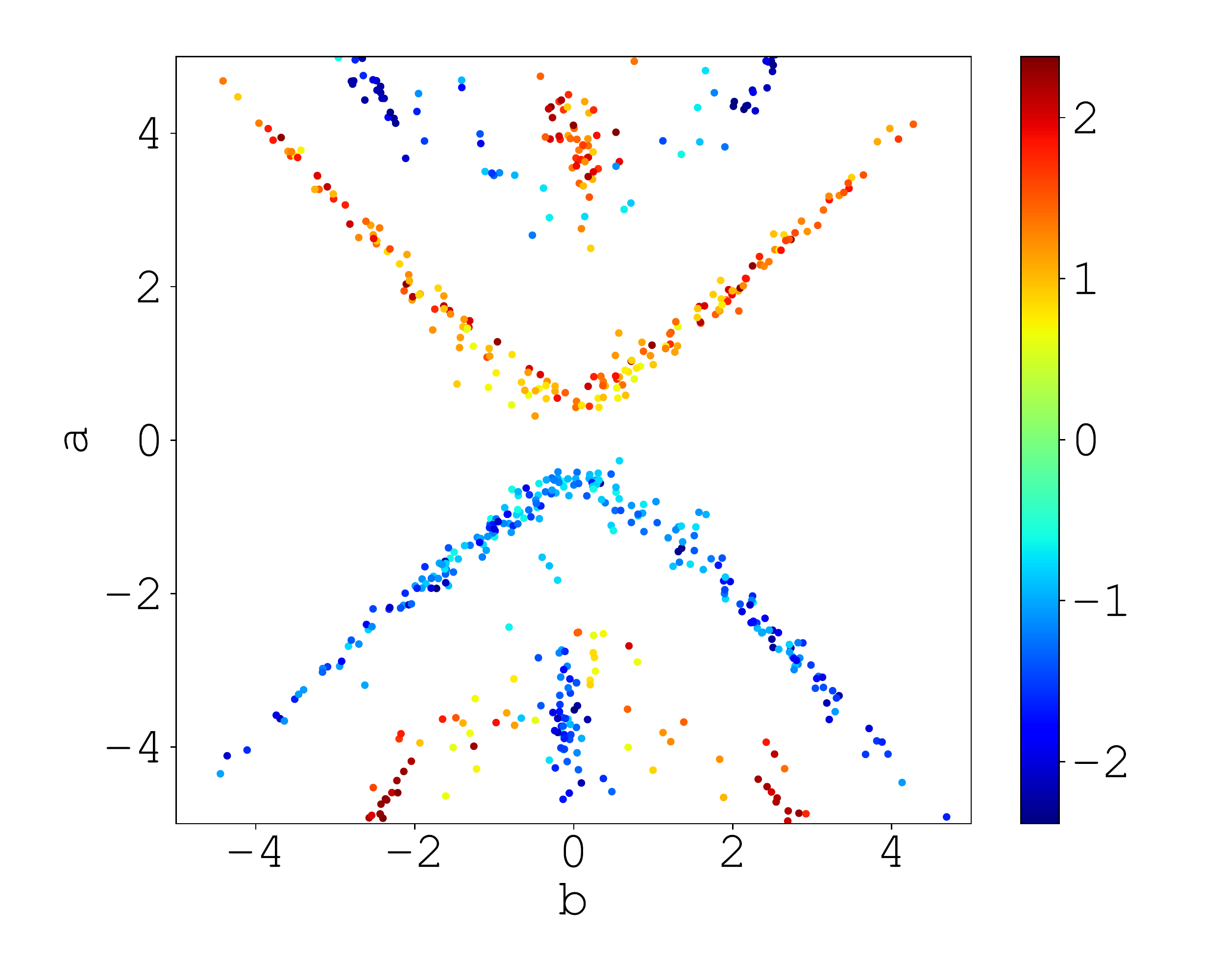}
    \caption{tanh, adam}
    \end{subfigure}%
    \begin{subfigure}[c]{0.33\textwidth}
    \includegraphics[width=\linewidth, trim=1cm 0cm 1cm 1cm, clip]{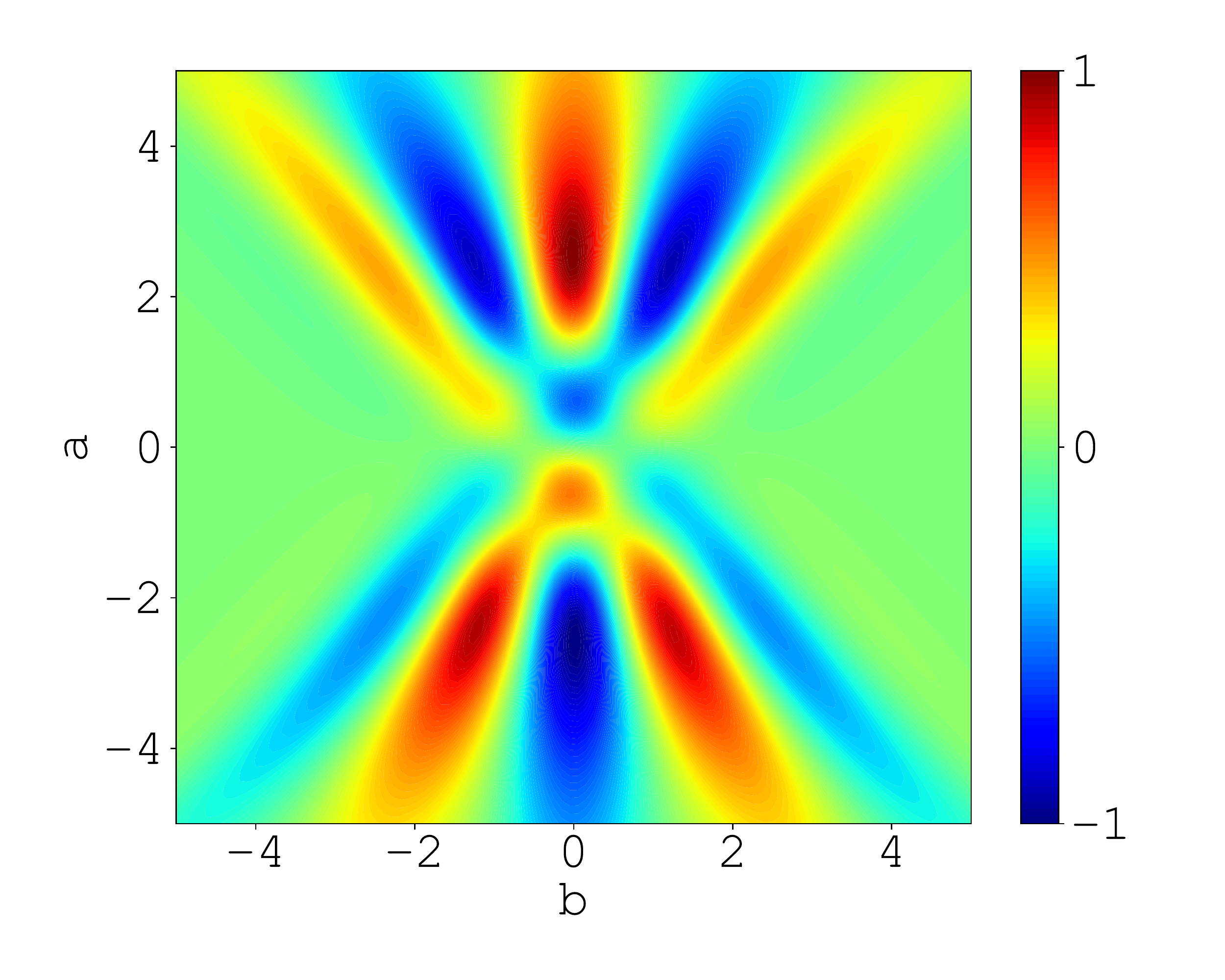}
    \caption{tanh}
    \end{subfigure}\\
    \begin{subfigure}[c]{0.33\textwidth}
    \includegraphics[width=\linewidth, trim=1cm 0cm 1cm 1cm, clip]{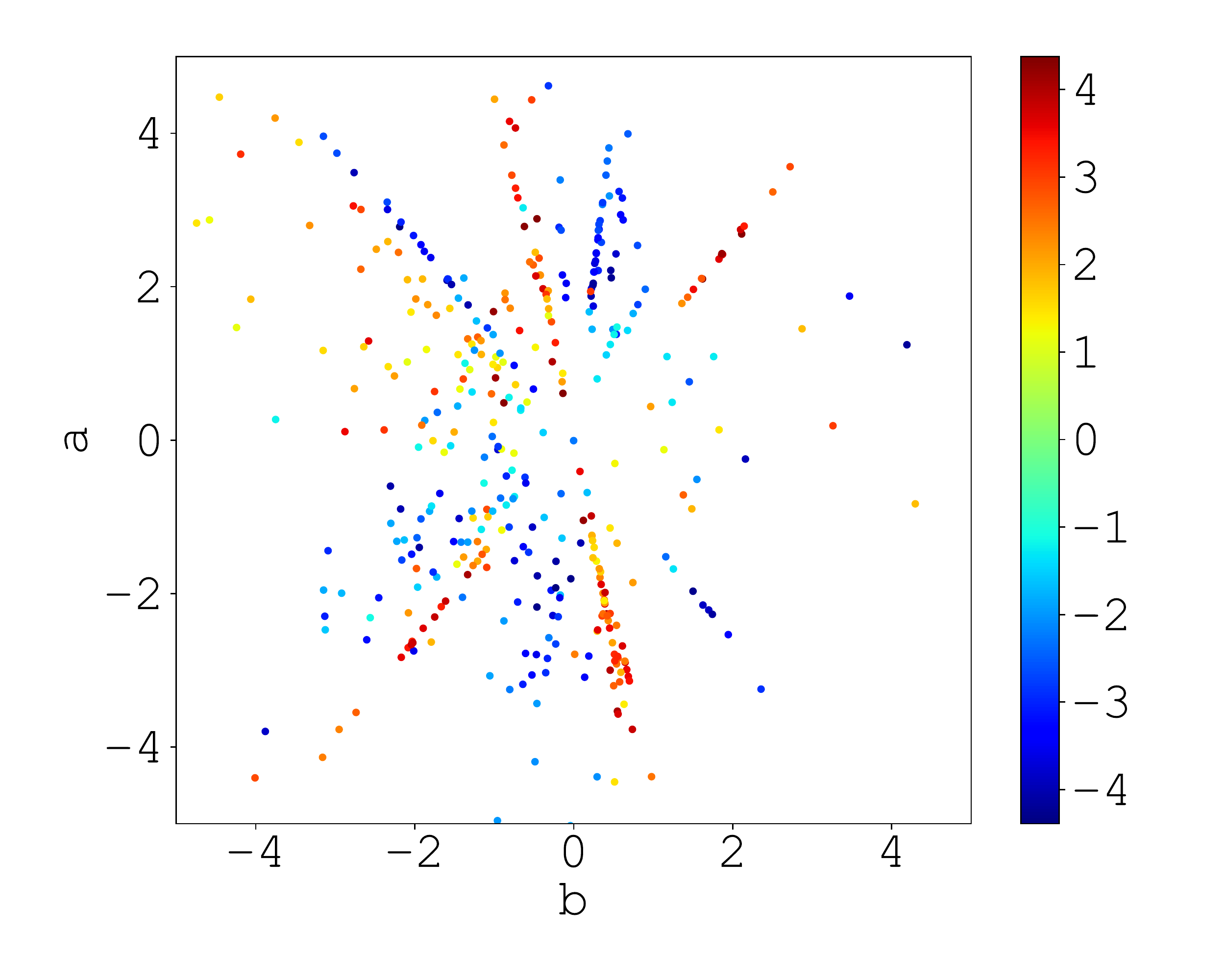}
    \caption{relu, bfgs}
    \end{subfigure}%
    \begin{subfigure}[c]{0.33\textwidth}
    \includegraphics[width=\linewidth, trim=1cm 0cm 1cm 1cm, clip]{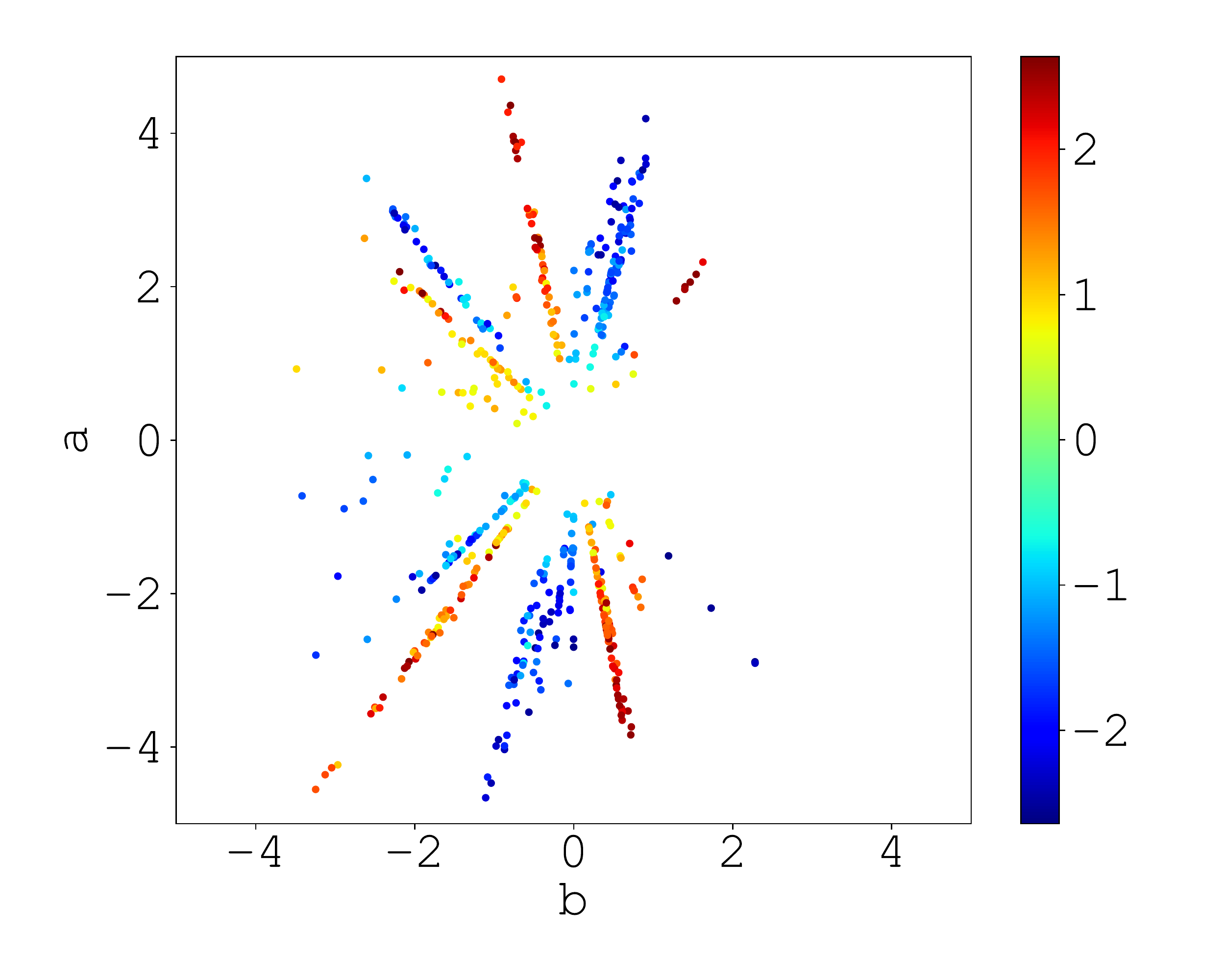}
    \caption{relu, adam}
    \end{subfigure}%
    \begin{subfigure}[c]{0.33\textwidth}
    \includegraphics[width=\linewidth, trim=1cm 0cm 1cm 1cm, clip]{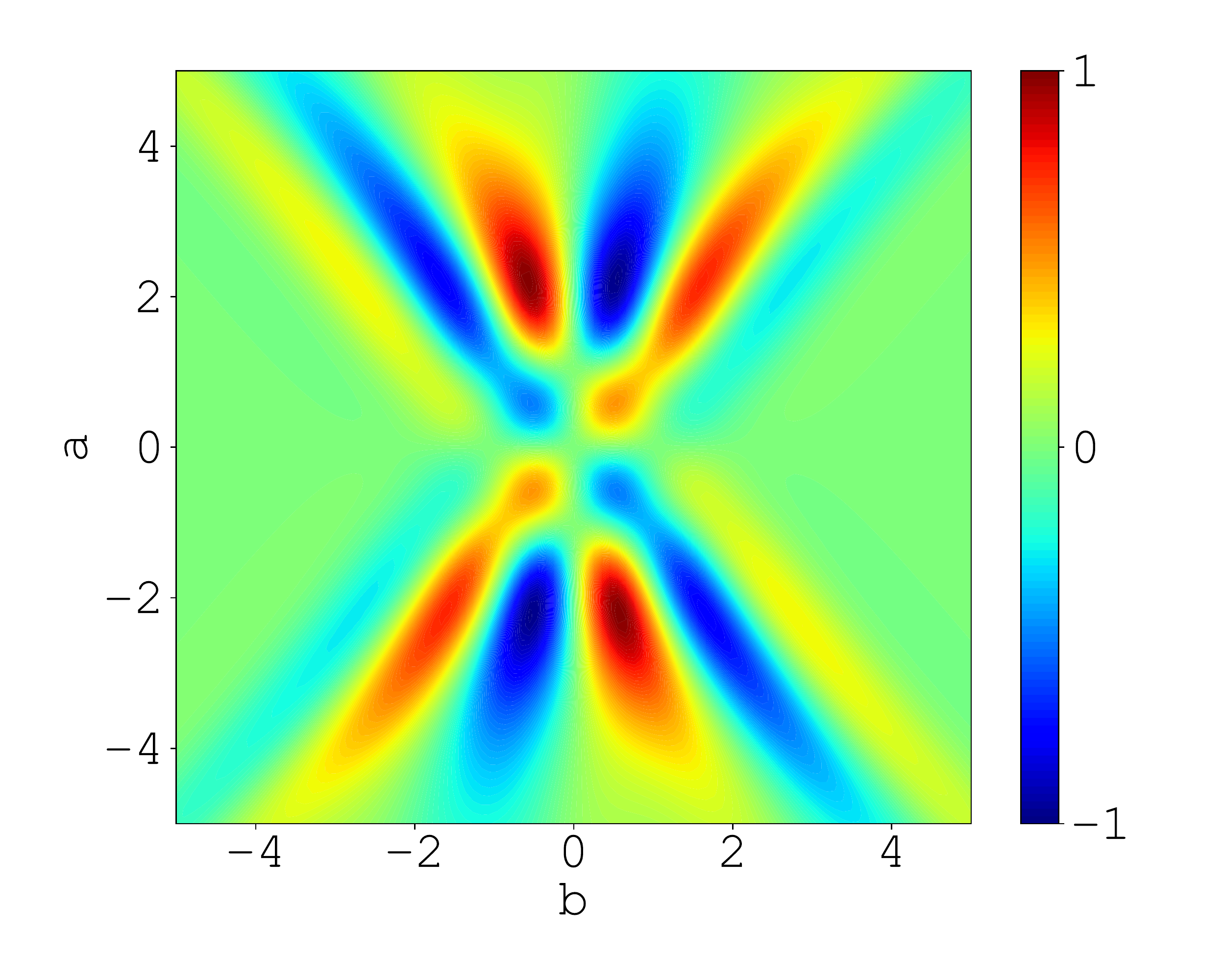}
    \caption{relu}
    \end{subfigure}\\
\caption{Sinusoidal Curve with Gaussian Noise $N(0, 1.0^2)$}
\label{fig:nsin02pt.e0100.n0000}
\end{figure}

\begin{figure}[h]
    \begin{center}
    \begin{subfigure}[c]{0.66\textwidth}
    \includegraphics[width=\linewidth, trim=0cm 0cm 0cm 0cm, clip]{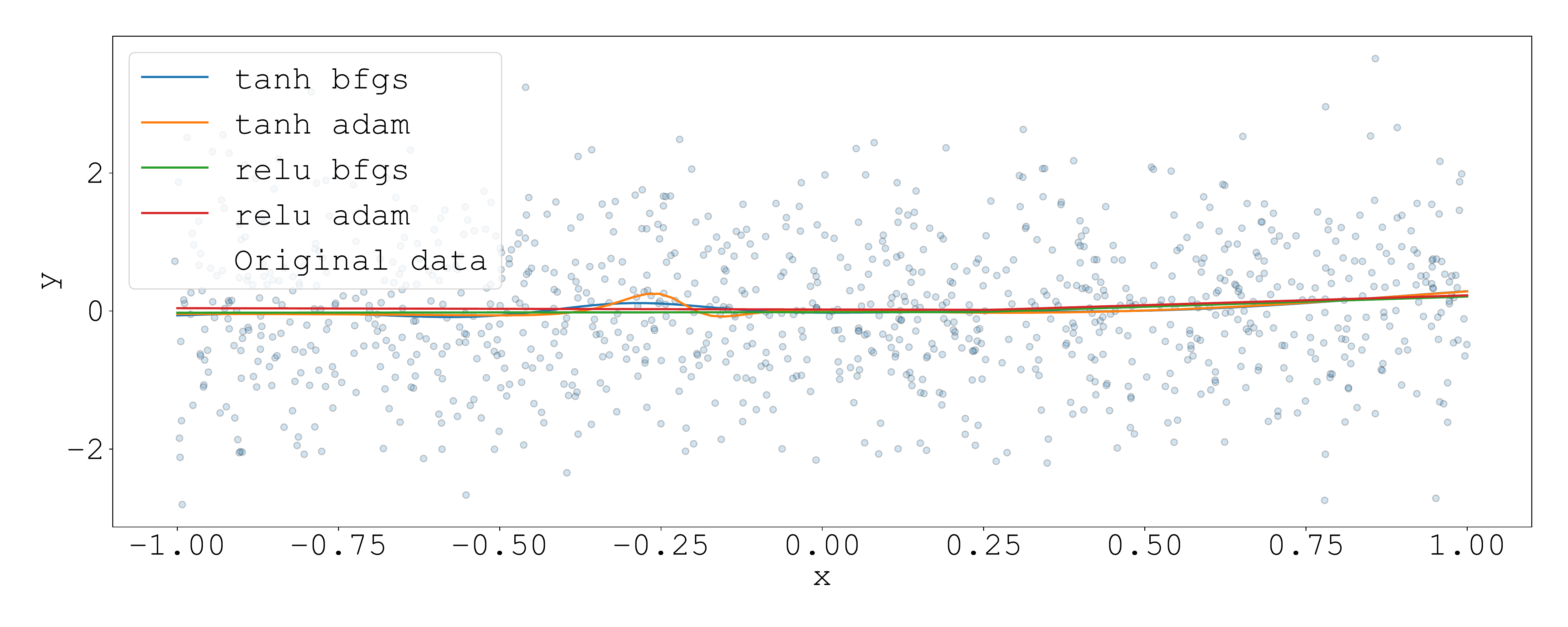}
    \caption{dataset}
    \end{subfigure}
    \end{center}
    \begin{subfigure}[c]{0.33\textwidth}
    \includegraphics[width=\linewidth, trim=1cm 0cm 1cm 1cm, clip]{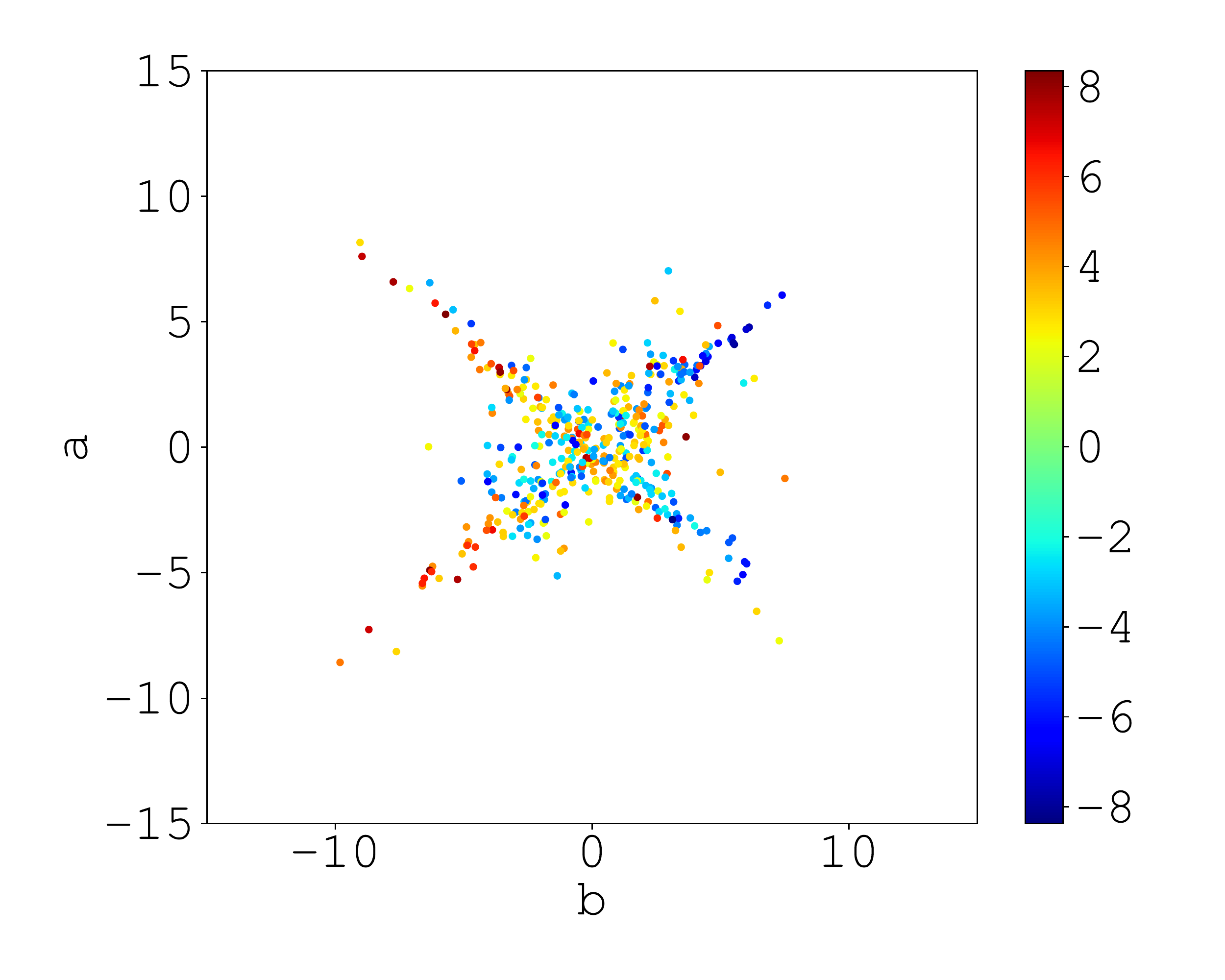}
    \caption{tanh, bfgs}
    \end{subfigure}%
    \begin{subfigure}[c]{0.33\textwidth}
    \includegraphics[width=\linewidth, trim=1cm 0cm 1cm 1cm, clip]{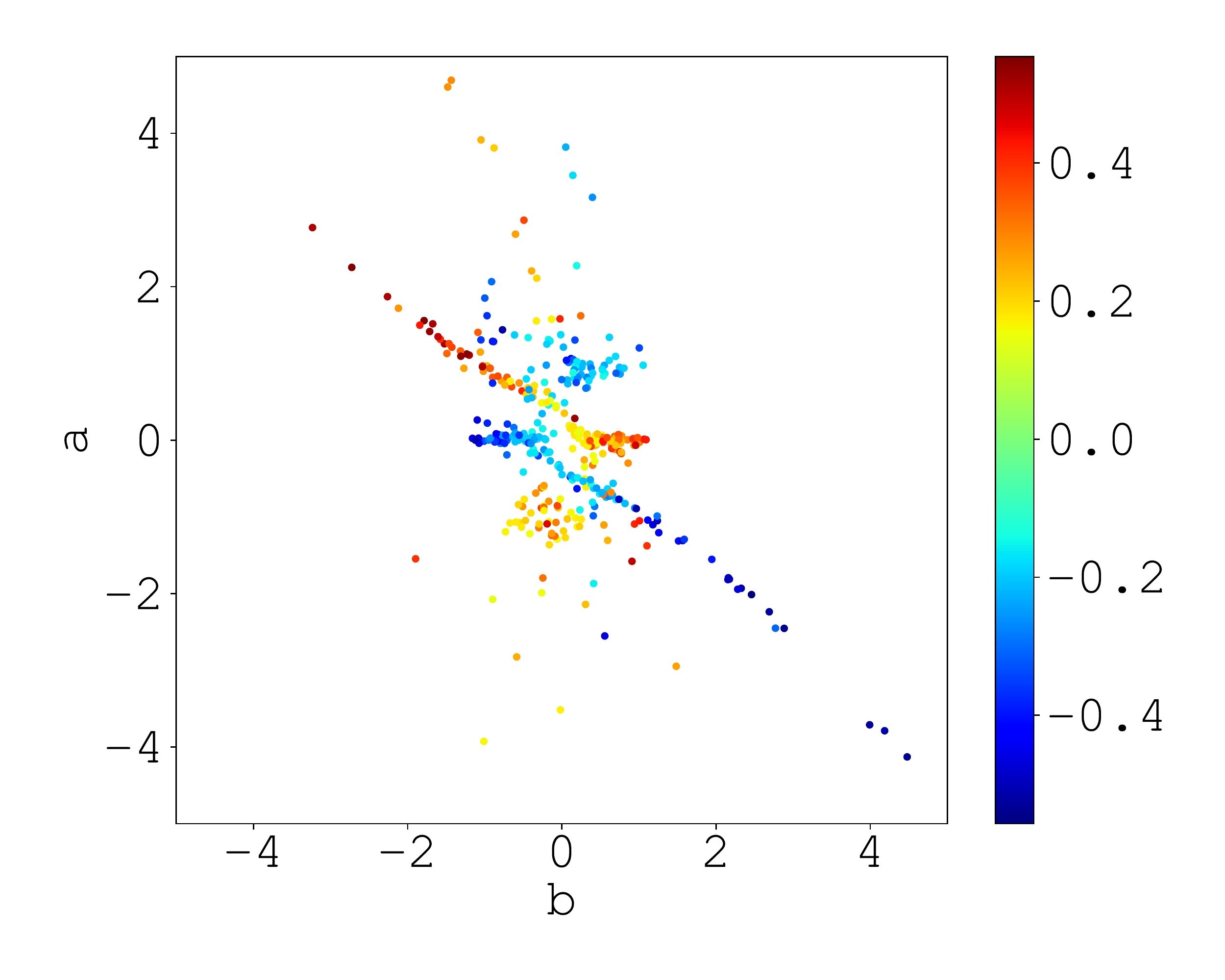}
    \caption{tanh, adam}
    \end{subfigure}%
    \begin{subfigure}[c]{0.33\textwidth}
    \includegraphics[width=\linewidth, trim=1cm 0cm 1cm 1cm, clip]{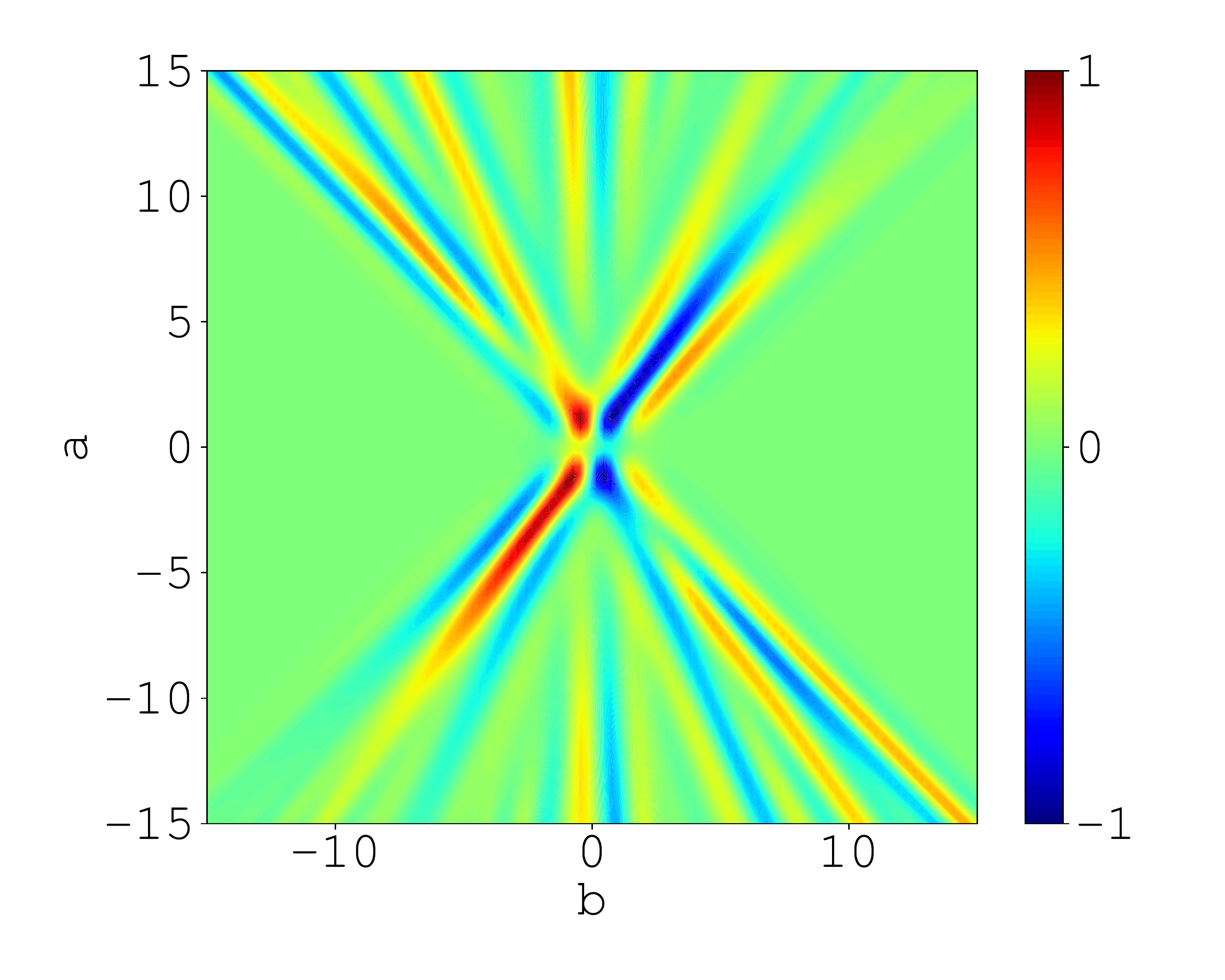}
    \caption{tanh}
    \end{subfigure}\\
    \begin{subfigure}[c]{0.33\textwidth}
    \includegraphics[width=\linewidth, trim=1cm 0cm 1cm 1cm, clip]{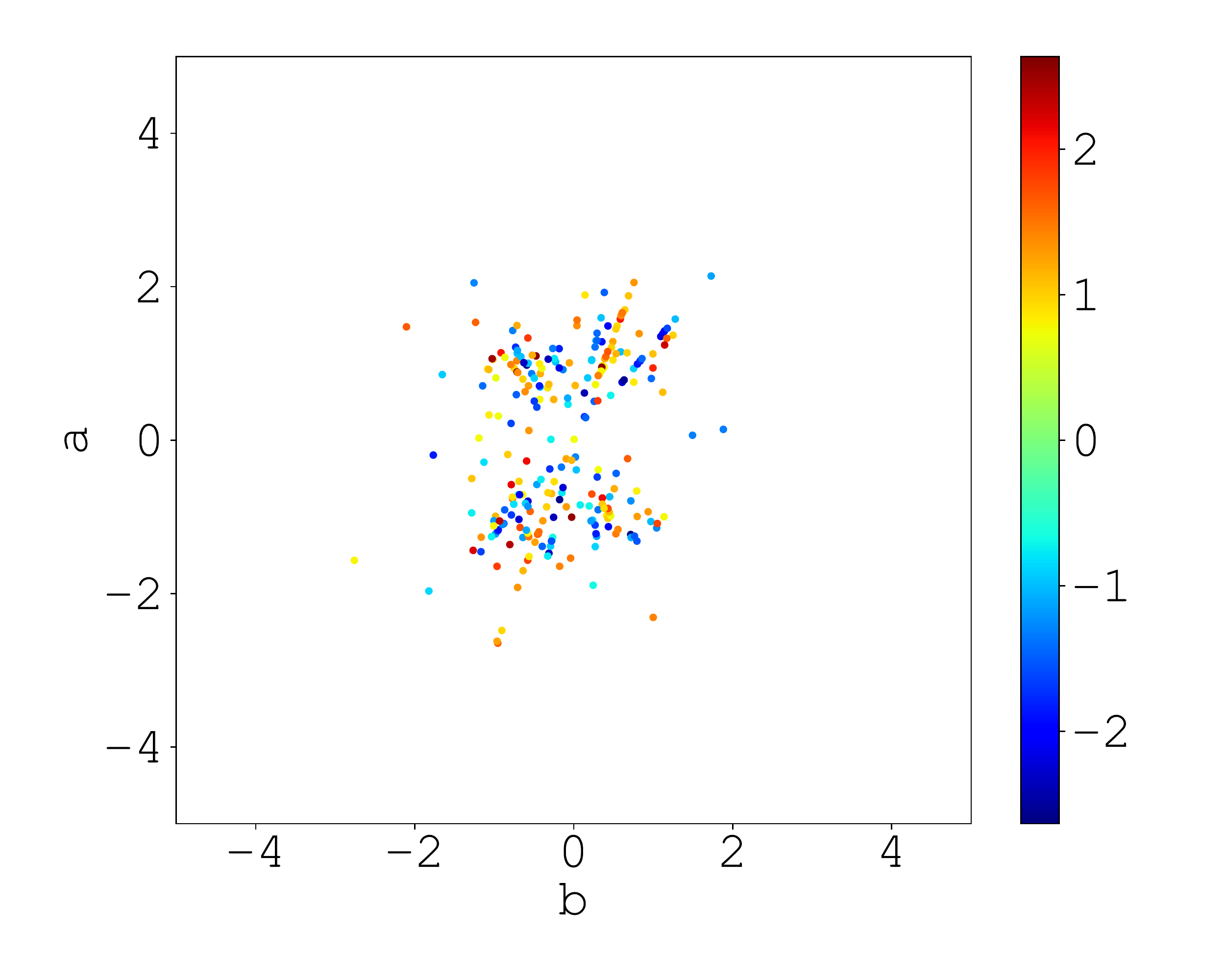}
    \caption{relu, bfgs}
    \end{subfigure}%
    \begin{subfigure}[c]{0.33\textwidth}
    \includegraphics[width=\linewidth, trim=1cm 0cm 1cm 1cm, clip]{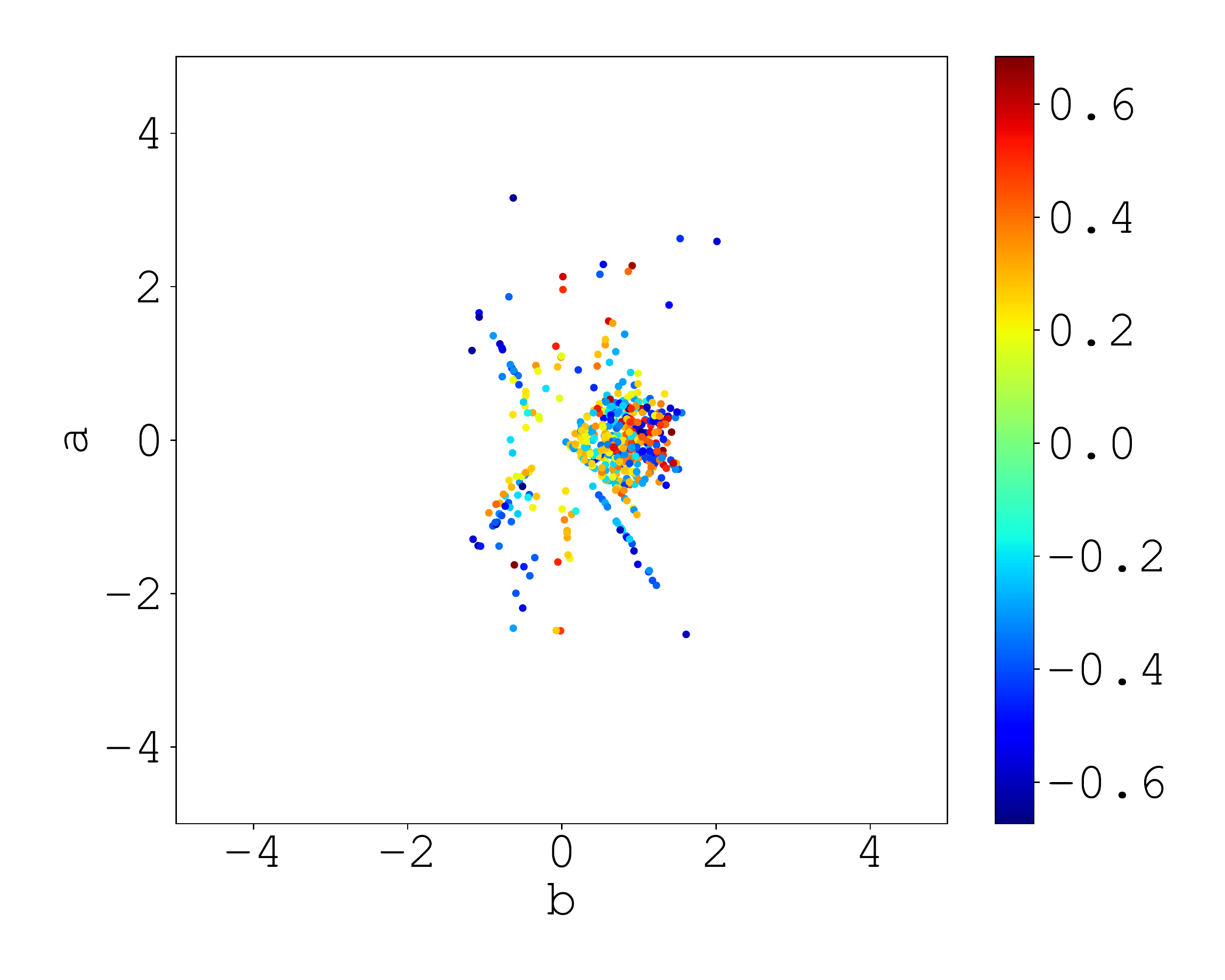}
    \caption{relu, adam}
    \end{subfigure}%
    \begin{subfigure}[c]{0.33\textwidth}
    \includegraphics[width=\linewidth, trim=1cm 0cm 1cm 1cm, clip]{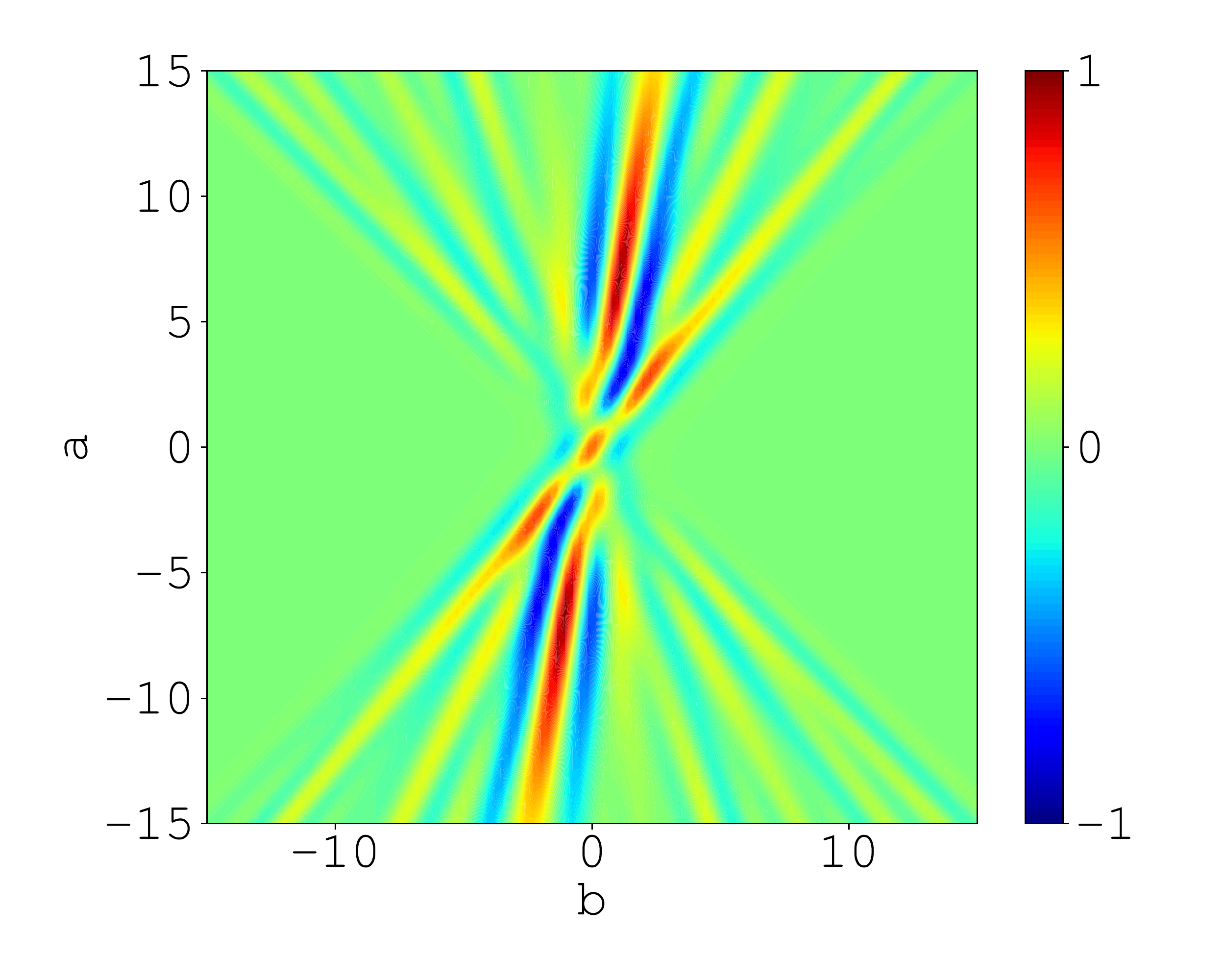}
    \caption{relu}
    \end{subfigure}\\
\caption{Gaussian Noise $N(0, 1.0^2)$}
\label{fig:gauss.m00.sd01.n0000}
\end{figure}

\begin{figure}[h]
    \begin{center}
    \begin{subfigure}[c]{0.66\textwidth}
    \includegraphics[width=\linewidth, trim=0cm 0cm 0cm 0cm, clip]{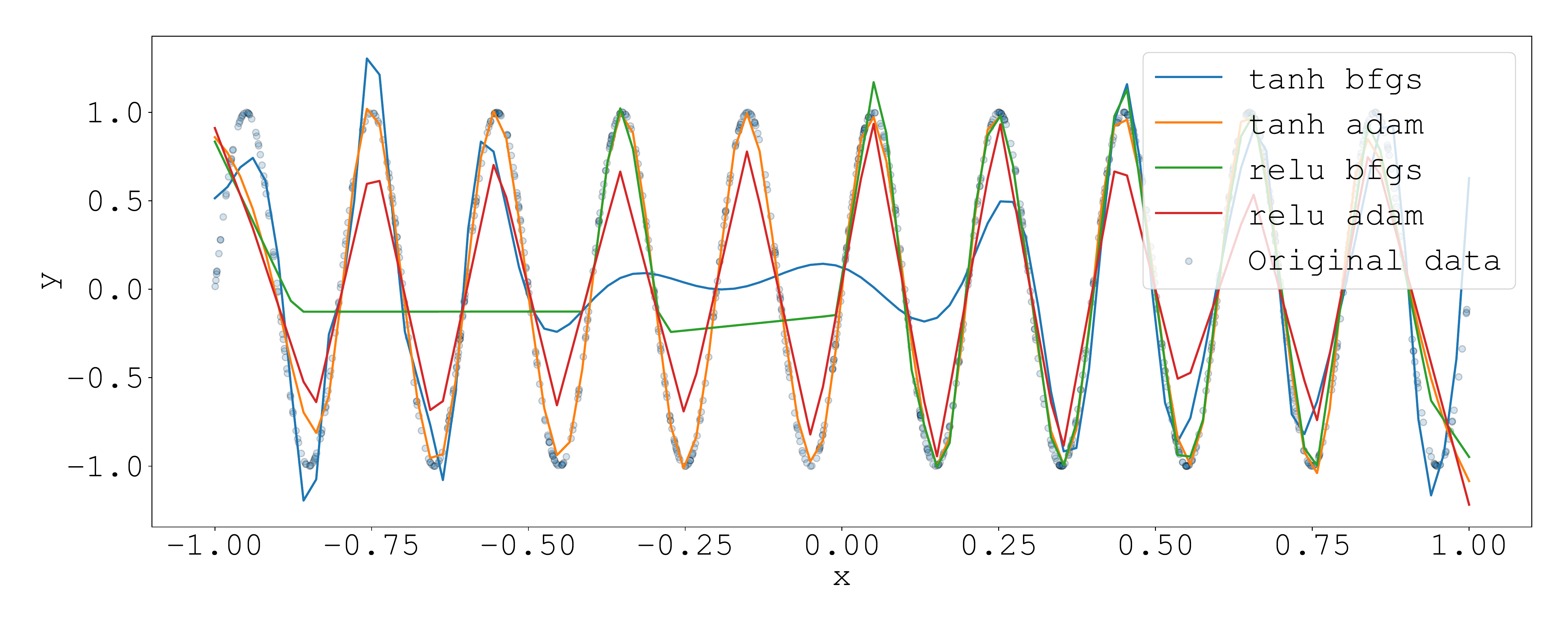}
    \caption{dataset}
    \end{subfigure}
    \end{center}
    \begin{subfigure}[c]{0.33\textwidth}
    \includegraphics[width=\linewidth, trim=1cm 0cm 1cm 1cm, clip]{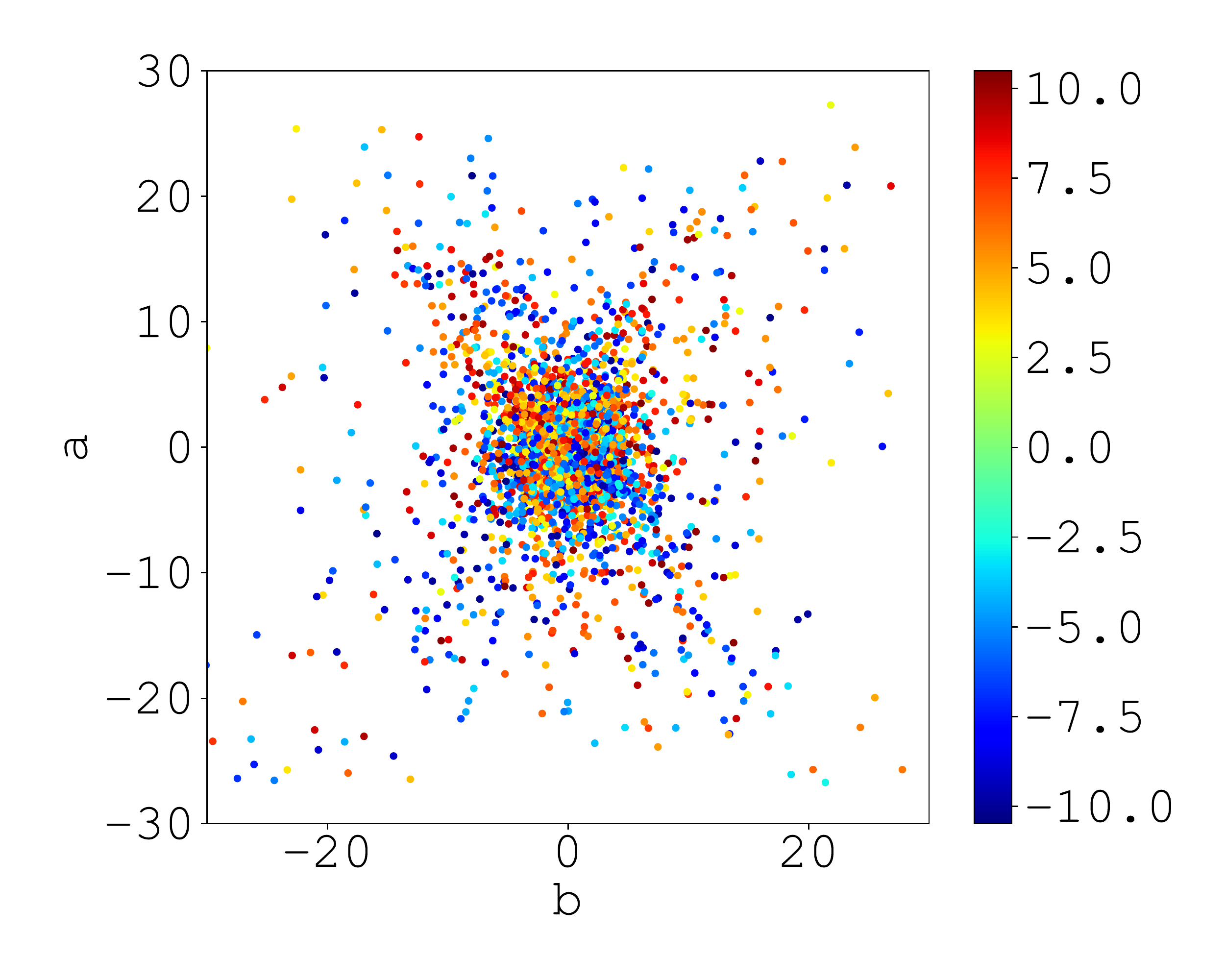}
    \caption{tanh, bfgs}
    \end{subfigure}%
    \begin{subfigure}[c]{0.33\textwidth}
    \includegraphics[width=\linewidth, trim=1cm 0cm 1cm 1cm, clip]{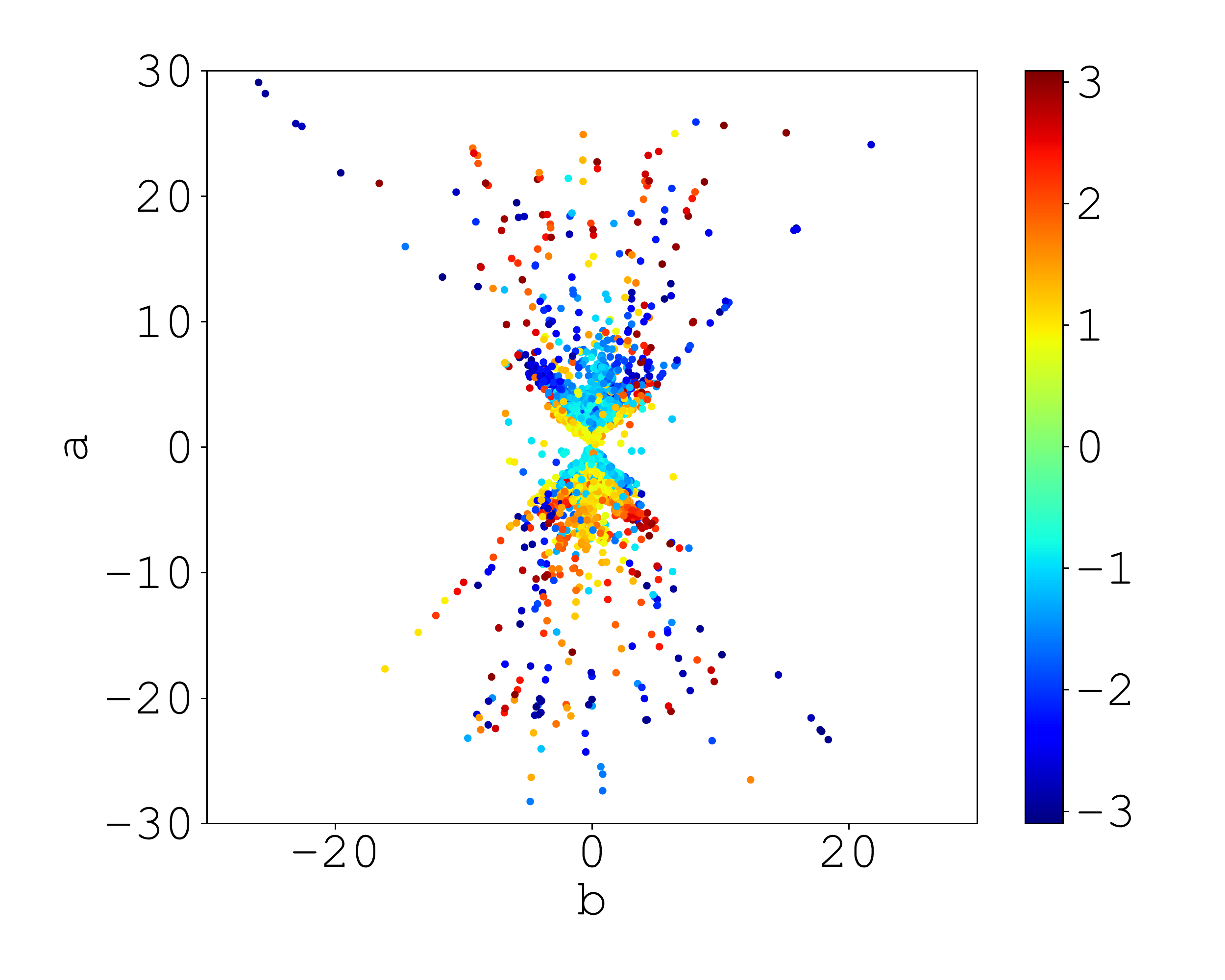}
    \caption{tanh, adam}
    \end{subfigure}%
    \begin{subfigure}[c]{0.33\textwidth}
    \includegraphics[width=\linewidth, trim=1cm 0cm 1cm 1cm, clip]{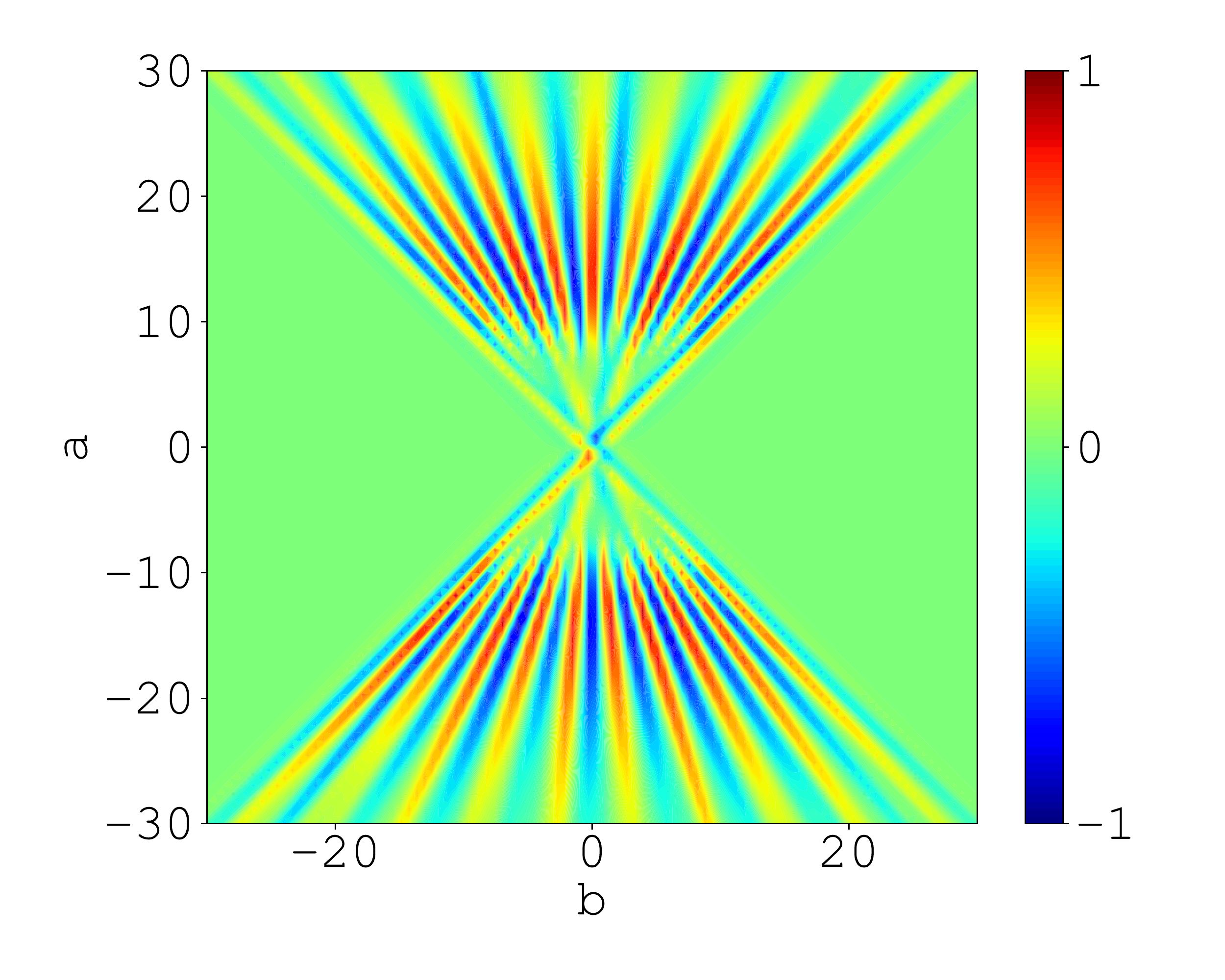}
    \caption{tanh}
    \end{subfigure}\\
    \begin{subfigure}[c]{0.33\textwidth}
    \includegraphics[width=\linewidth, trim=1cm 0cm 1cm 1cm, clip]{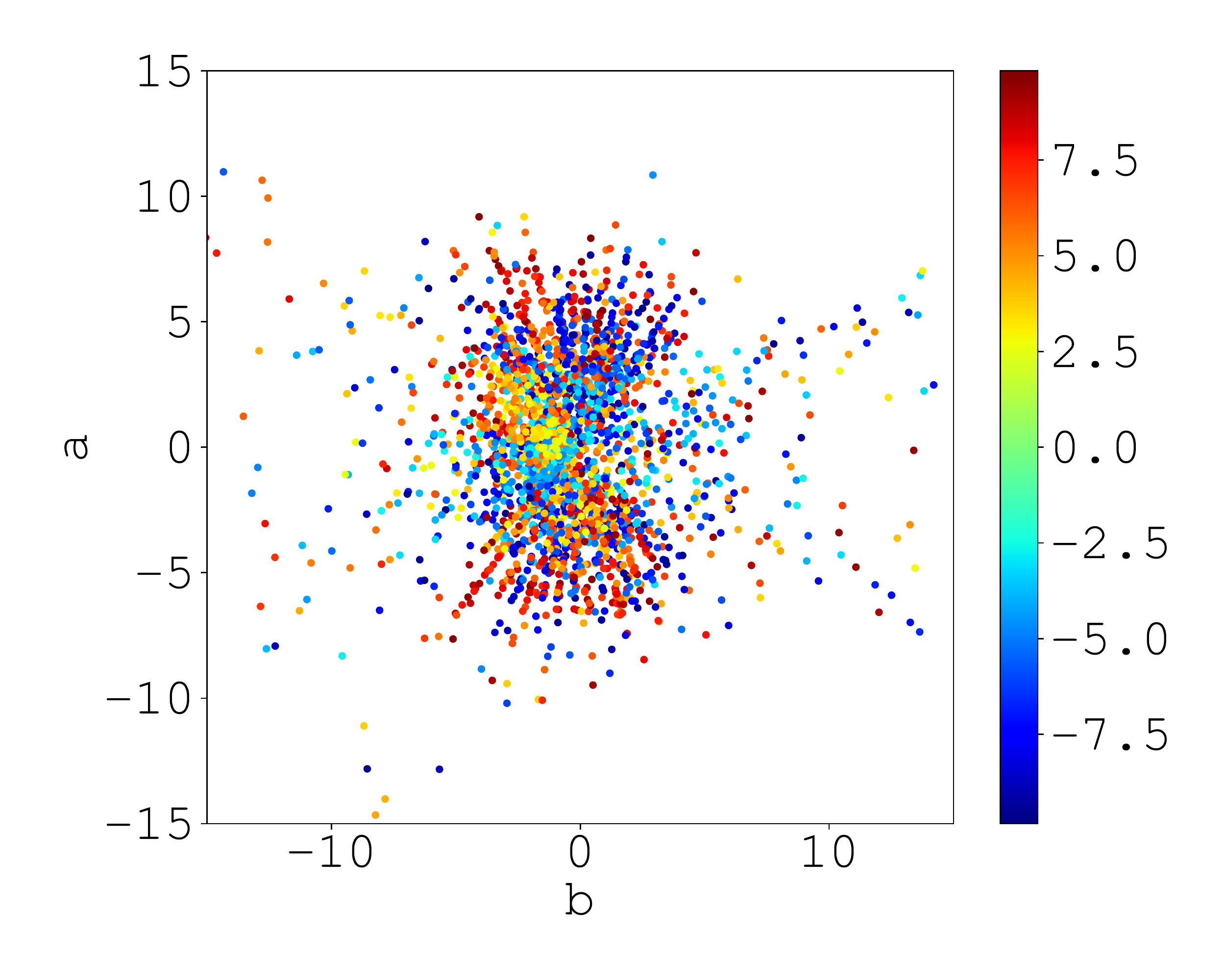}
    \caption{relu, bfgs}
    \end{subfigure}%
    \begin{subfigure}[c]{0.33\textwidth}
    \includegraphics[width=\linewidth, trim=1cm 0cm 1cm 1cm, clip]{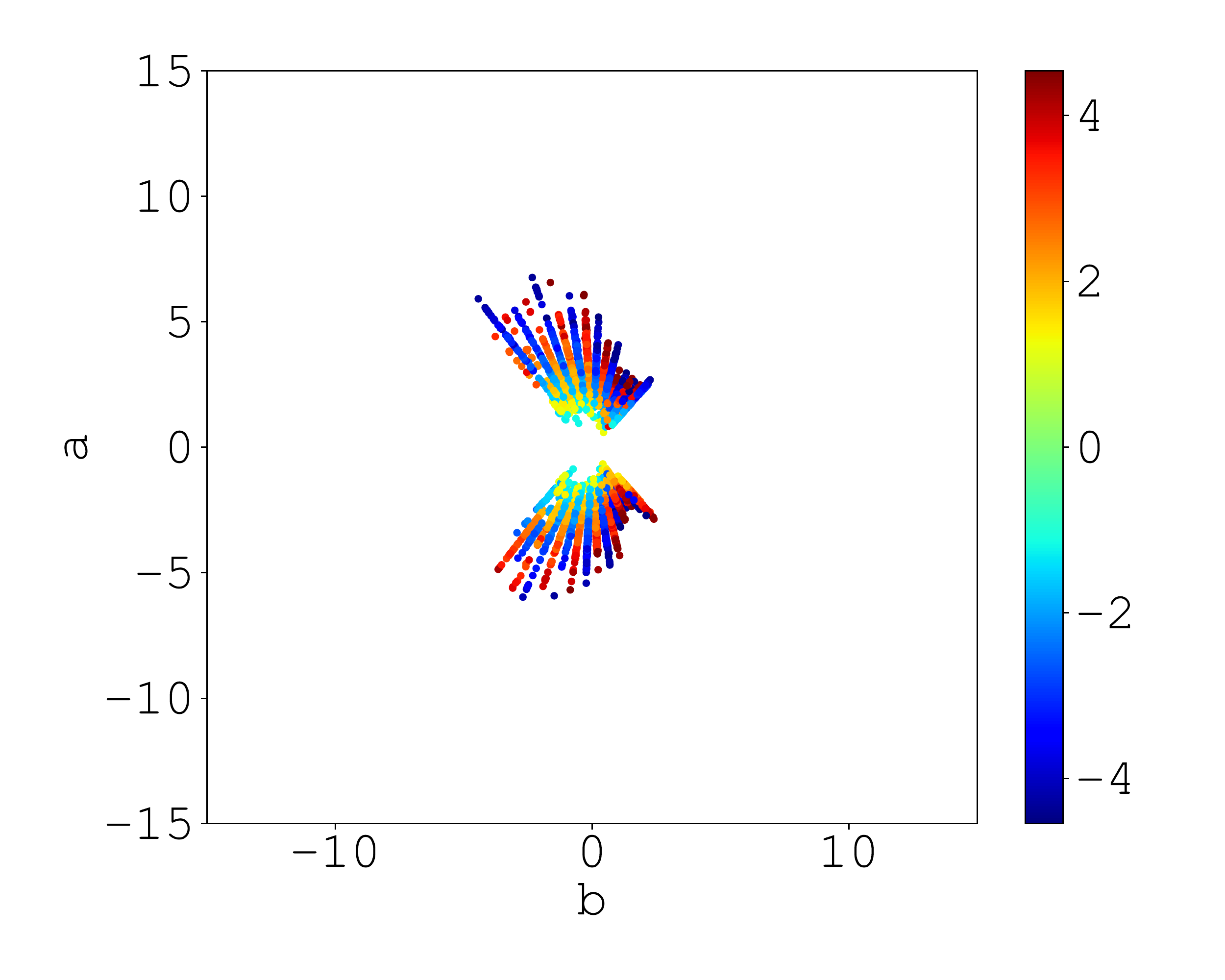}
    \caption{relu, adam}
    \end{subfigure}%
    \begin{subfigure}[c]{0.33\textwidth}
    \includegraphics[width=\linewidth, trim=1cm 0cm 1cm 1cm, clip]{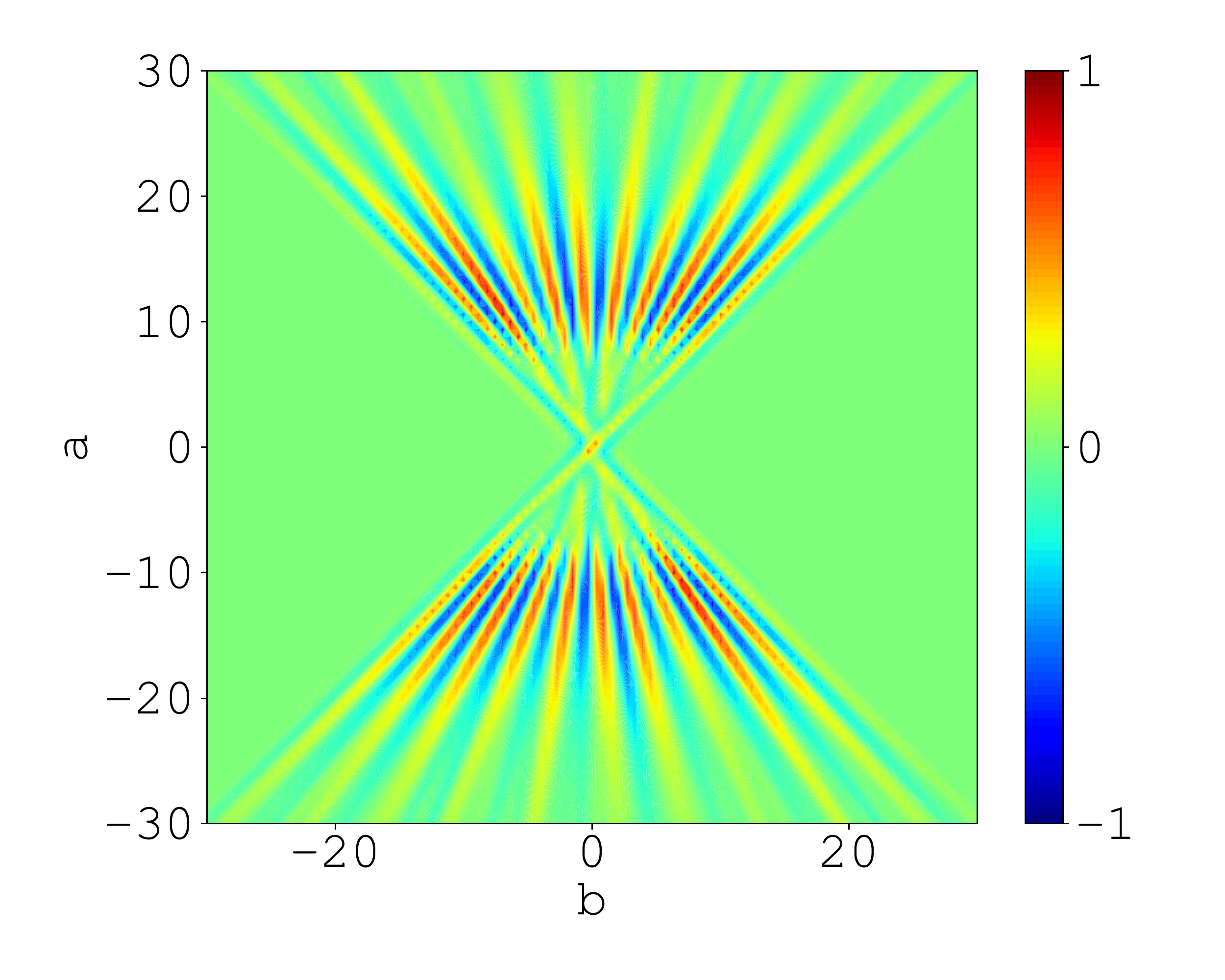}
    \caption{relu}
    \end{subfigure}\\
\caption{High Frequency Sinusoidal Curve}
\label{fig:sin10pt.n0000}
\end{figure}

\begin{figure}[h]
    \begin{center}
    \begin{subfigure}[c]{0.66\textwidth}
    \includegraphics[width=\linewidth, trim=0cm 0cm 0cm 0cm, clip]{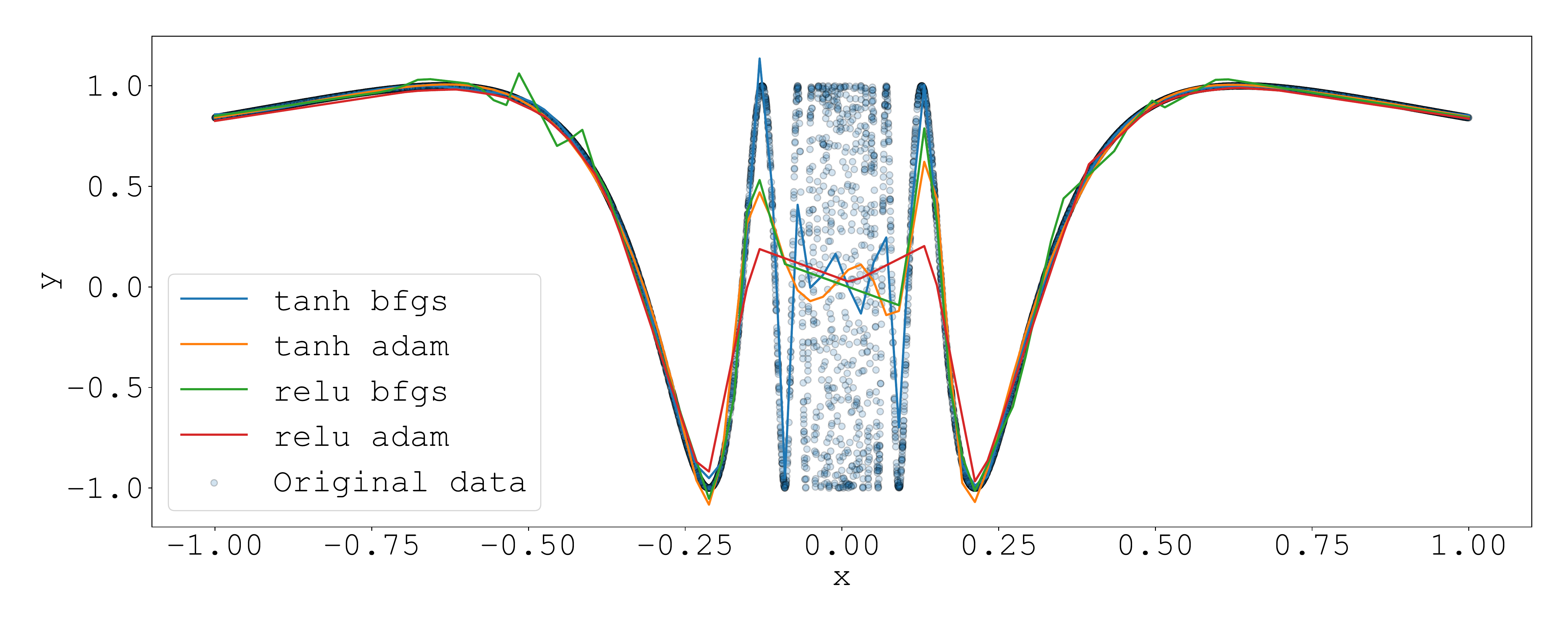}
    \caption{dataset}
    \end{subfigure}
    \end{center}
    \begin{subfigure}[c]{0.33\textwidth}
    \includegraphics[width=\linewidth, trim=1cm 0cm 1cm 1cm, clip]{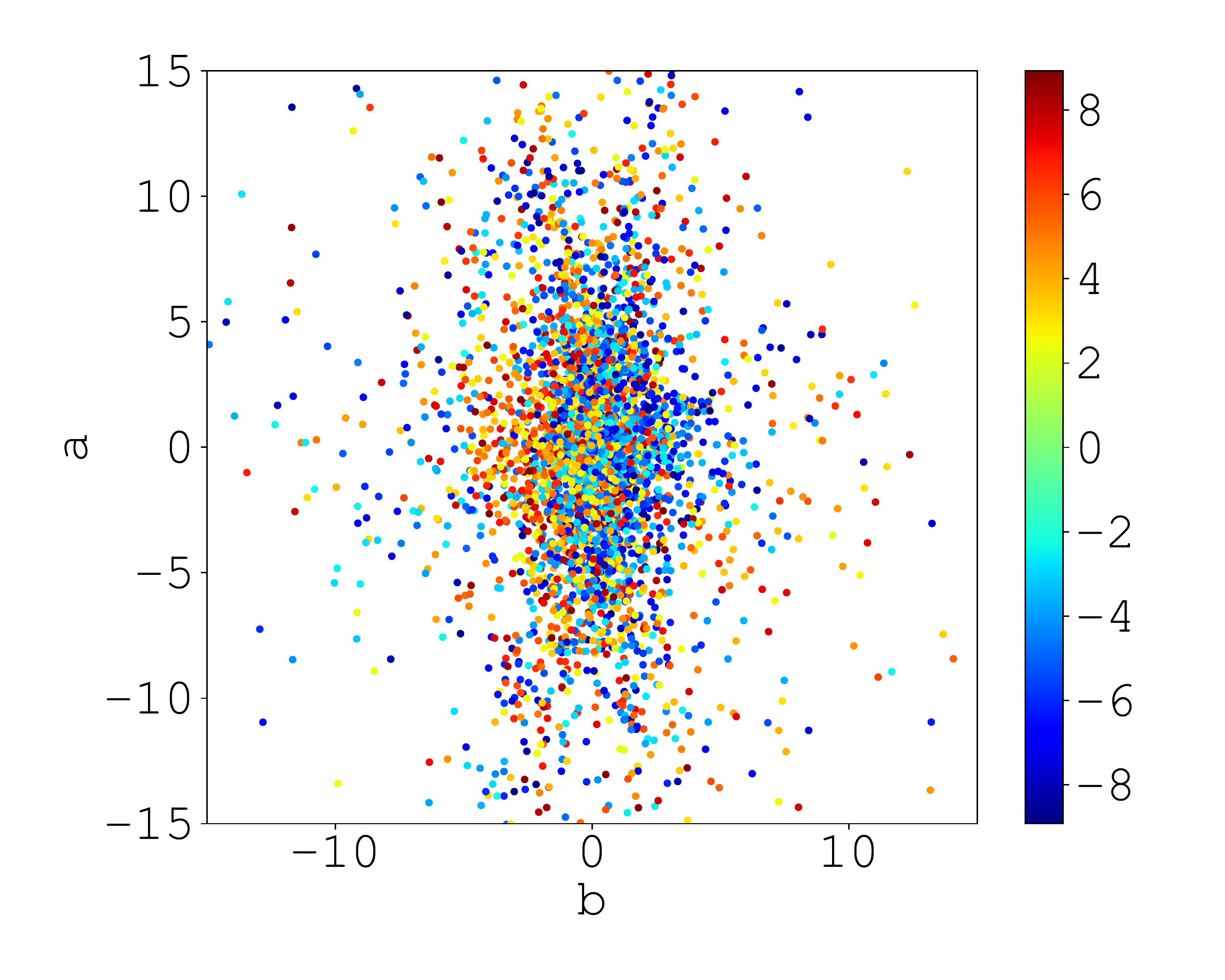}
    \caption{tanh, bfgs}
    \end{subfigure}%
    \begin{subfigure}[c]{0.33\textwidth}
    \includegraphics[width=\linewidth, trim=1cm 0cm 1cm 1cm, clip]{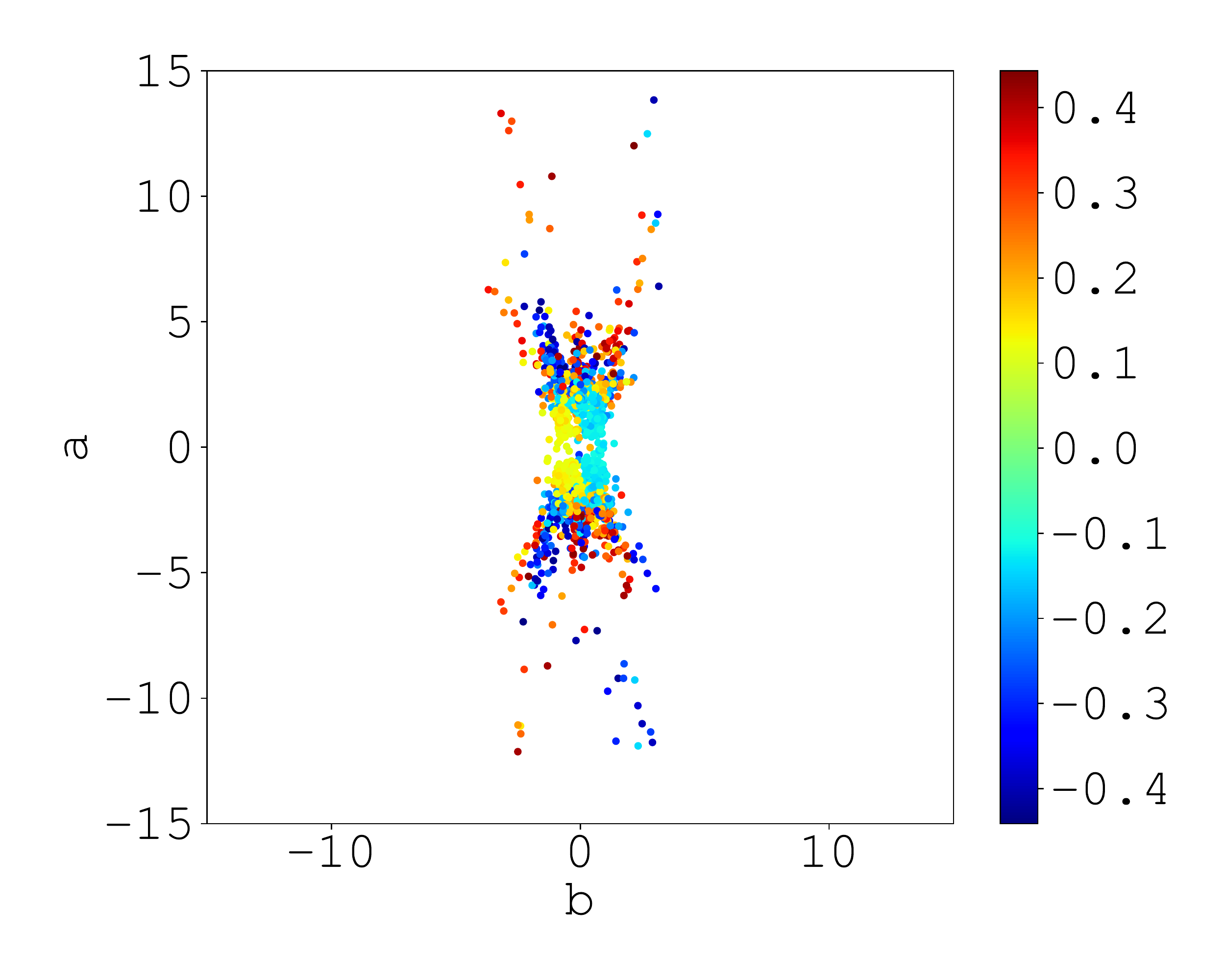}
    \caption{tanh, adam}
    \end{subfigure}%
    \begin{subfigure}[c]{0.33\textwidth}
    \includegraphics[width=\linewidth, trim=1cm 0cm 1cm 1cm, clip]{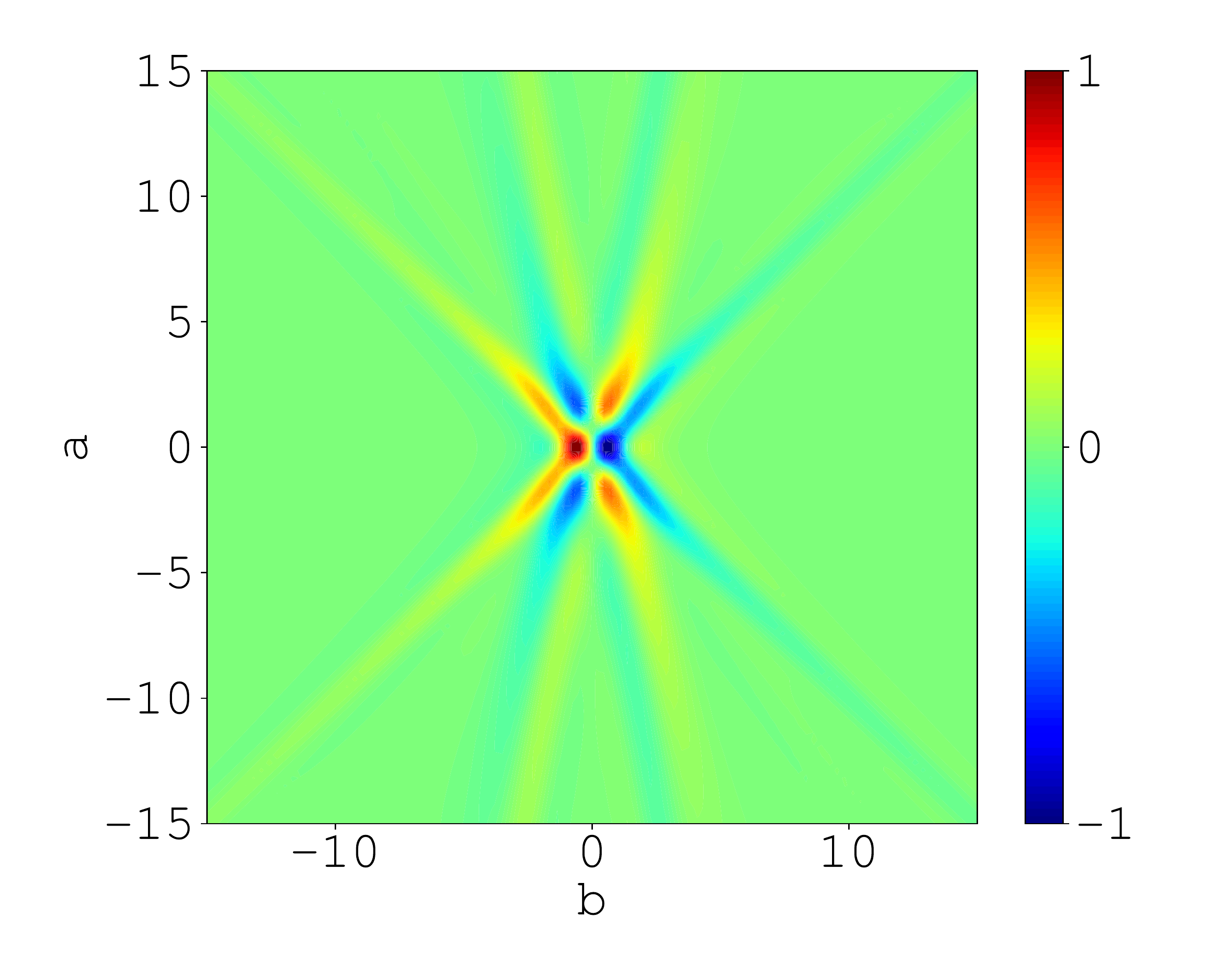}
    \caption{tanh}
    \end{subfigure}\\
    \begin{subfigure}[c]{0.33\textwidth}
    \includegraphics[width=\linewidth, trim=1cm 0cm 1cm 1cm, clip]{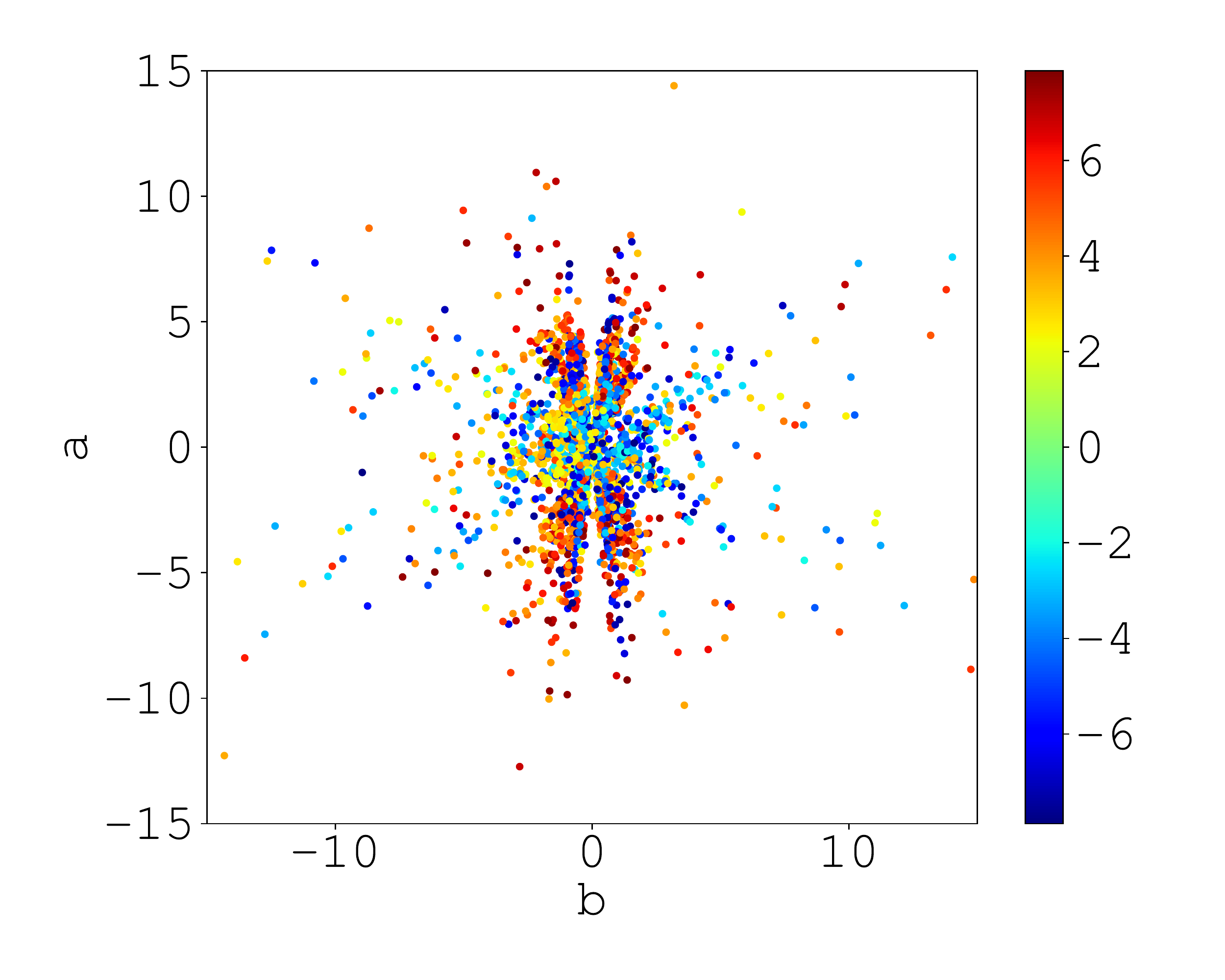}
    \caption{relu, bfgs}
    \end{subfigure}%
    \begin{subfigure}[c]{0.33\textwidth}
    \includegraphics[width=\linewidth, trim=1cm 0cm 1cm 1cm, clip]{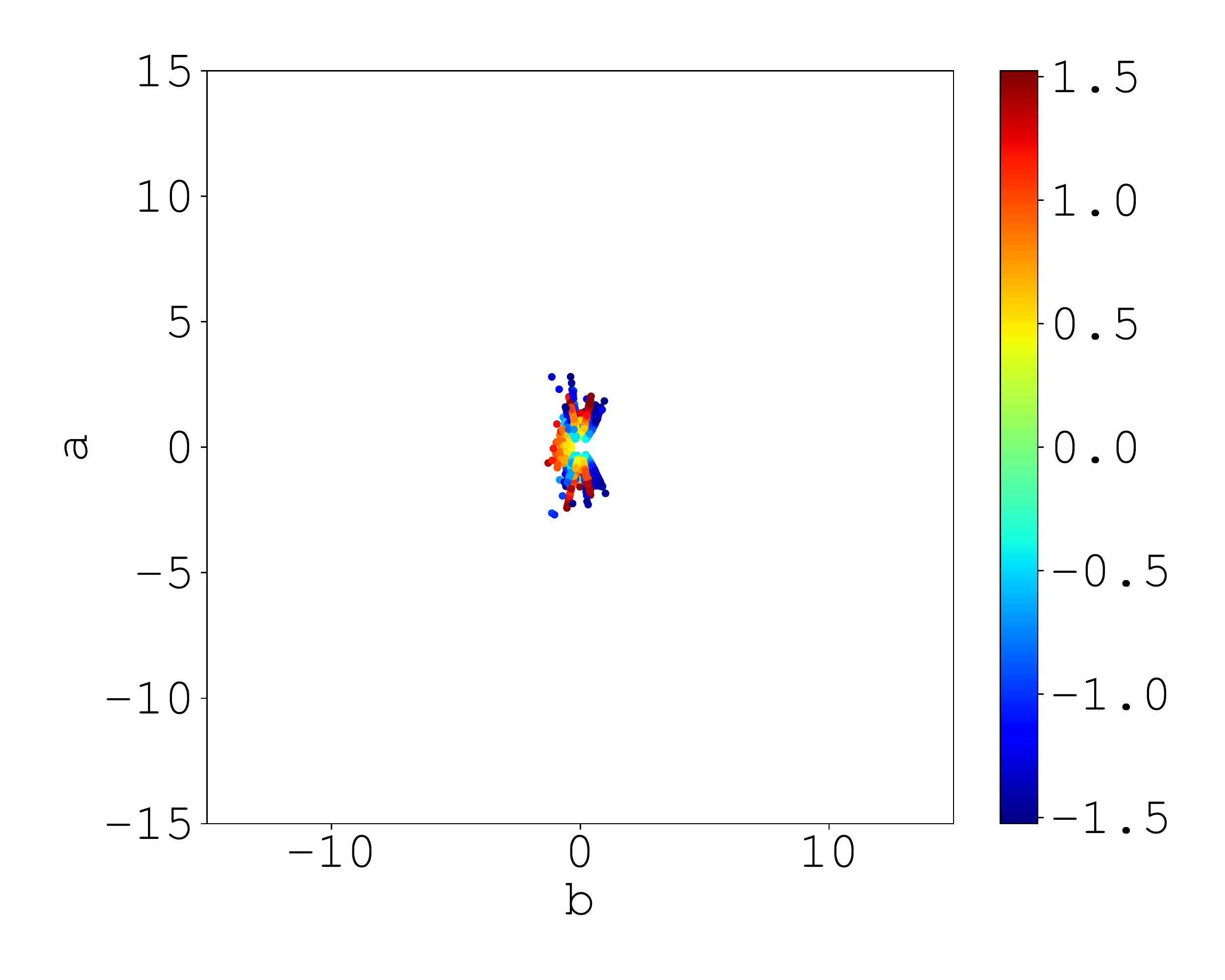}
    \caption{relu, adam}
    \end{subfigure}%
    \begin{subfigure}[c]{0.33\textwidth}
    \includegraphics[width=\linewidth, trim=1cm 0cm 1cm 1cm, clip]{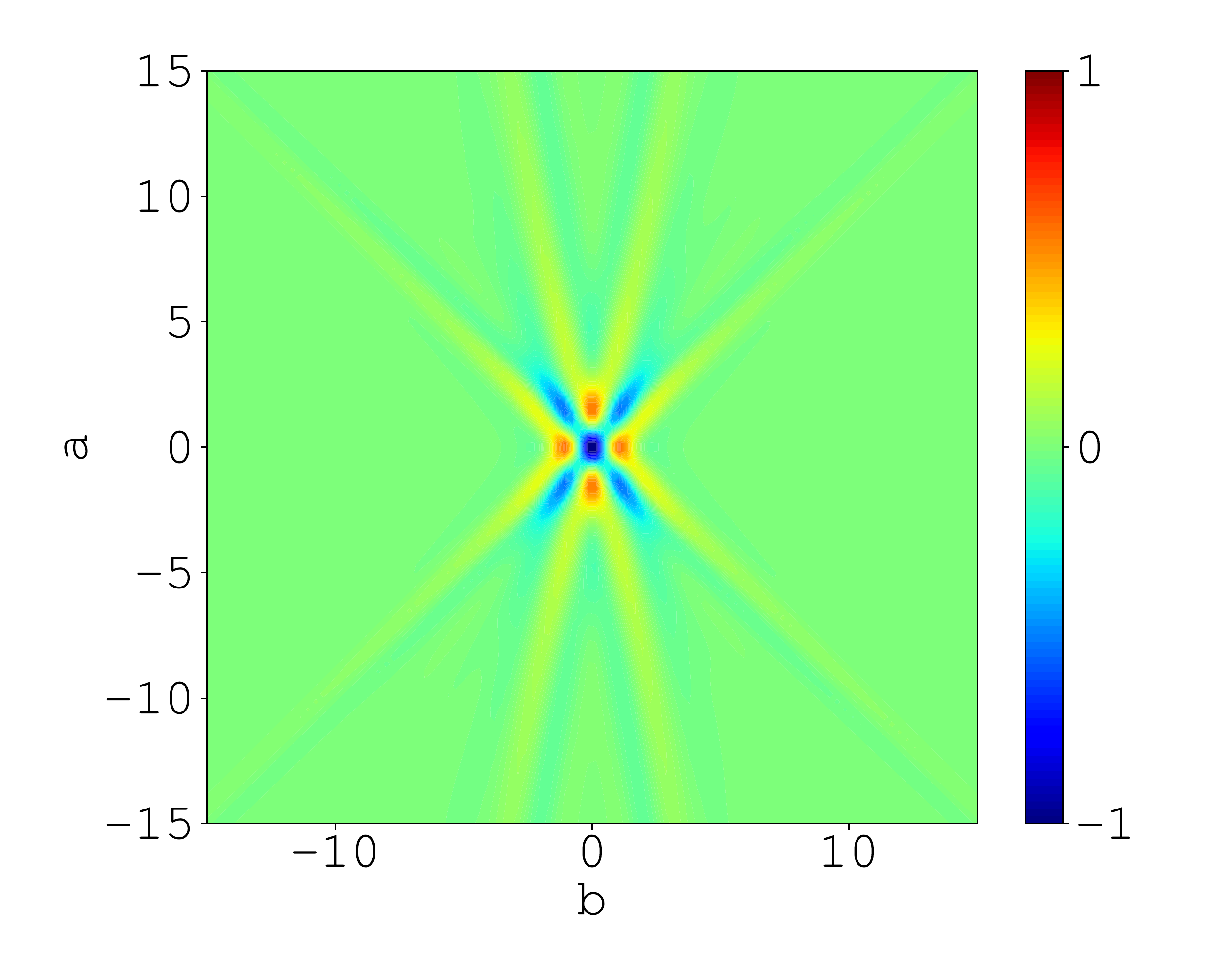}
    \caption{relu}
    \end{subfigure}\\
\caption{Topologist's Sinusoidal Curve}
\label{fig:topsin.n0000}
\end{figure}

\begin{figure}[h]
    \begin{center}
    \begin{subfigure}[c]{0.66\textwidth}
    \includegraphics[width=\linewidth, trim=0cm 0cm 0cm 0cm, clip]{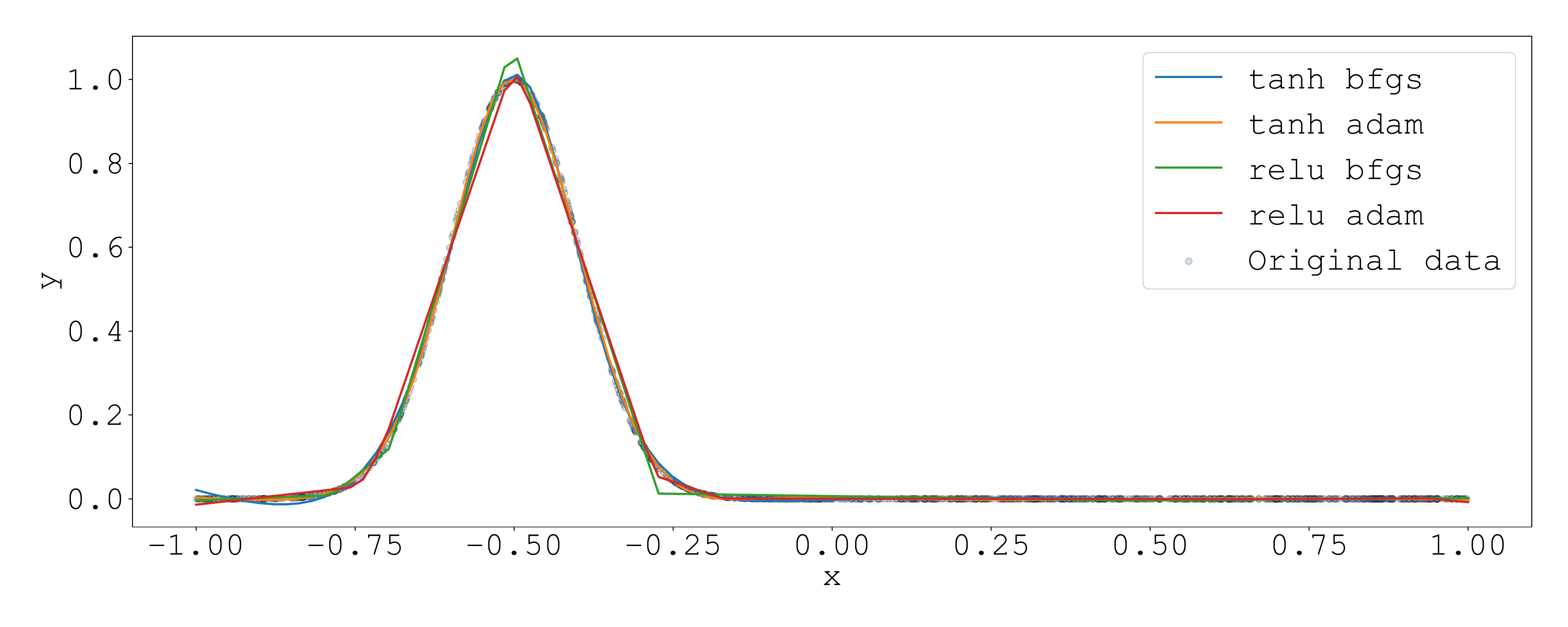}
    \caption{dataset}
    \end{subfigure}
    \end{center}
    \begin{subfigure}[c]{0.33\textwidth}
    \includegraphics[width=\linewidth, trim=1cm 0cm 1cm 1cm, clip]{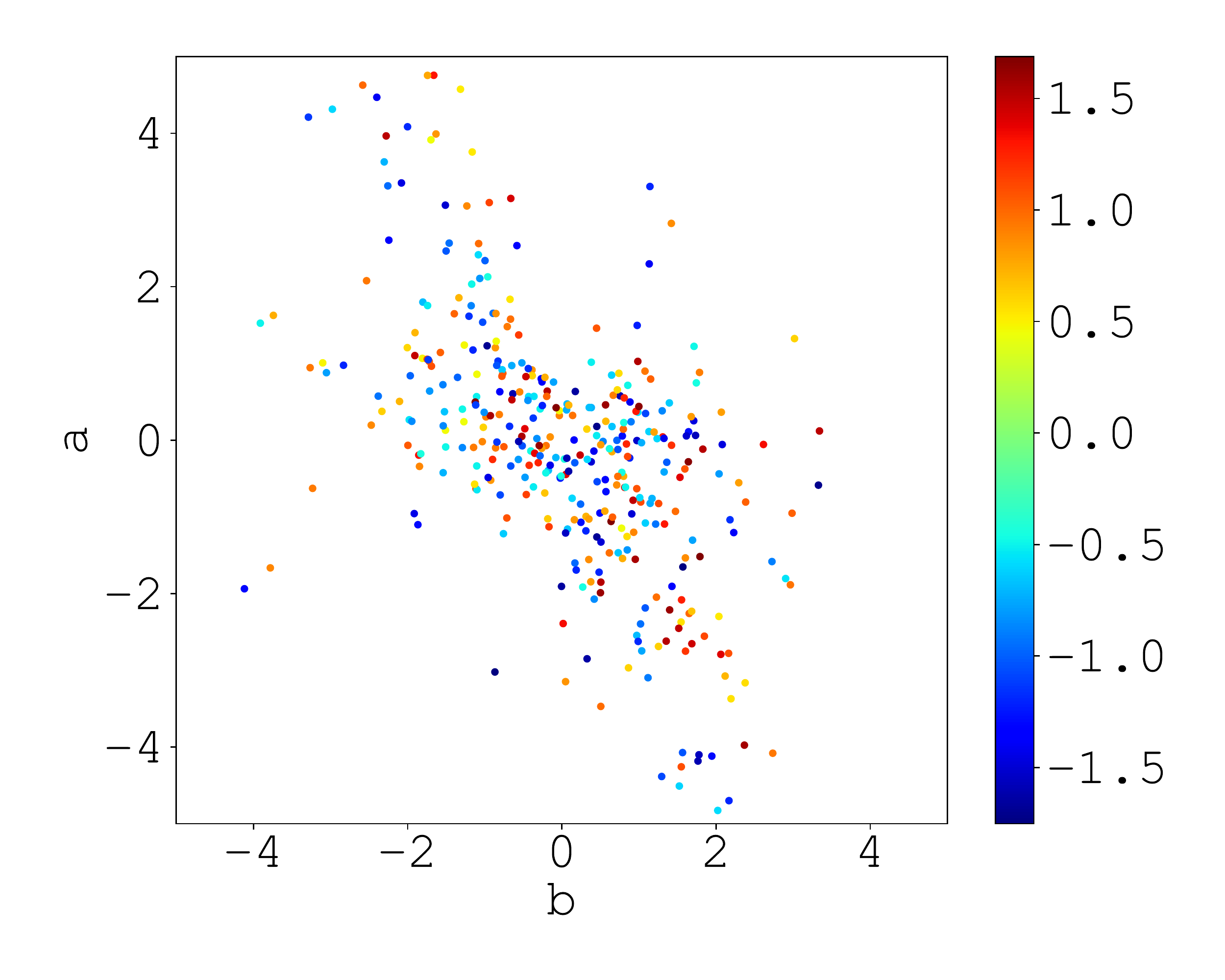}
    \caption{tanh, bfgs}
    \end{subfigure}%
    \begin{subfigure}[c]{0.33\textwidth}
    \includegraphics[width=\linewidth, trim=1cm 0cm 1cm 1cm, clip]{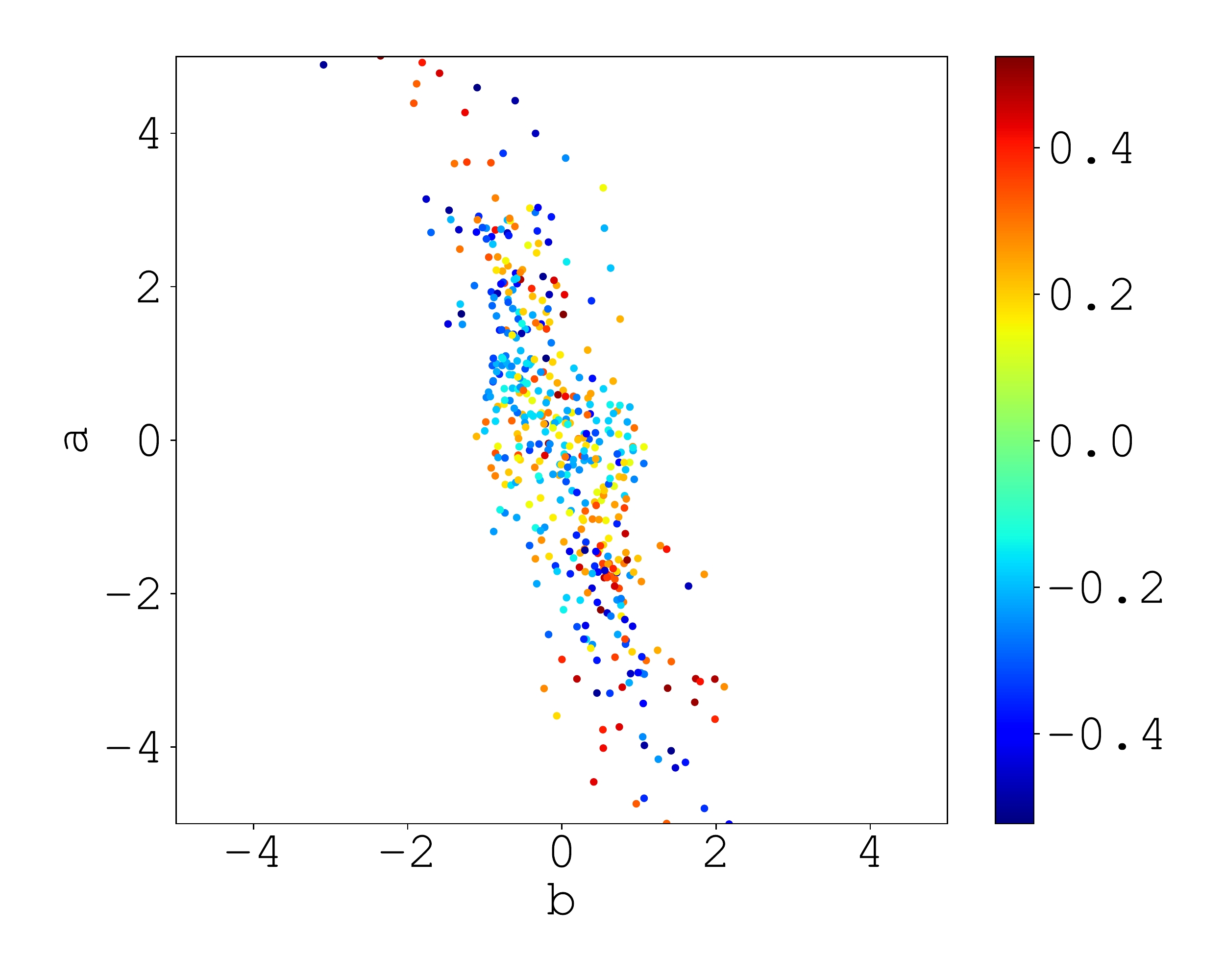}
    \caption{tanh, adam}
    \end{subfigure}%
    \begin{subfigure}[c]{0.33\textwidth}
    \includegraphics[width=\linewidth, trim=1cm 0cm 1cm 1cm, clip]{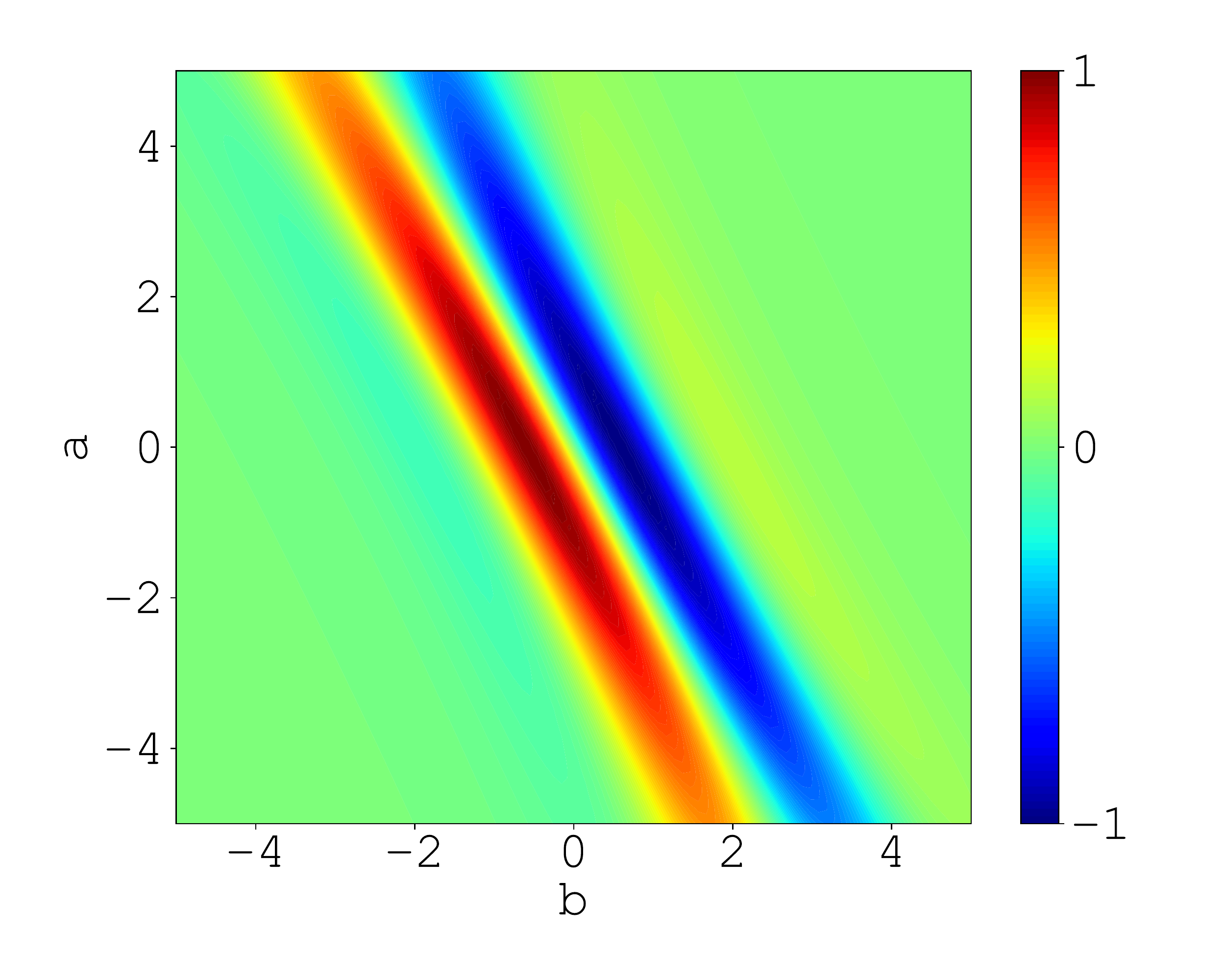}
    \caption{tanh}
    \end{subfigure}\\
    \begin{subfigure}[c]{0.33\textwidth}
    \includegraphics[width=\linewidth, trim=1cm 0cm 1cm 1cm, clip]{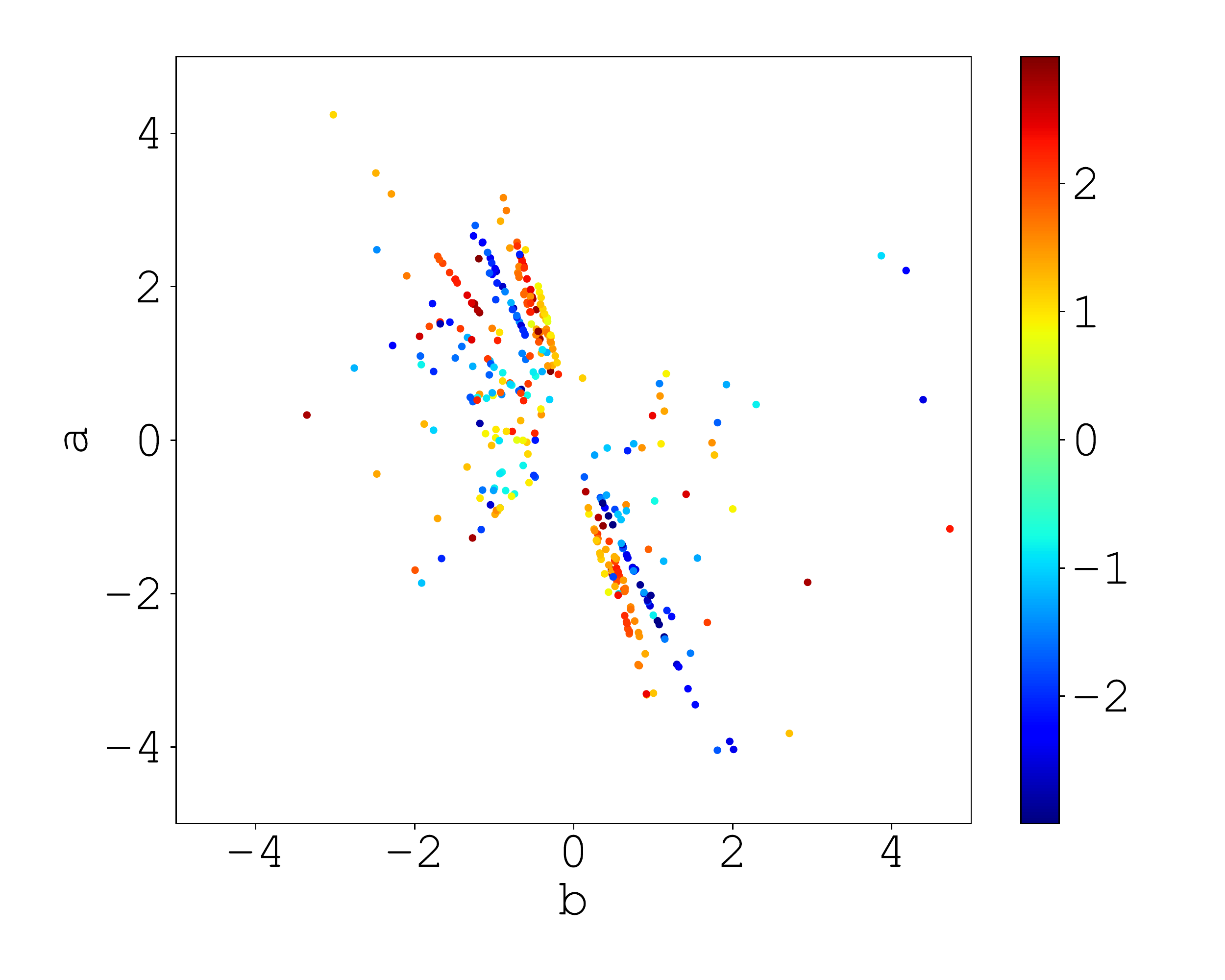}
    \caption{relu, bfgs}
    \end{subfigure}%
    \begin{subfigure}[c]{0.33\textwidth}
    \includegraphics[width=\linewidth, trim=1cm 0cm 1cm 1cm, clip]{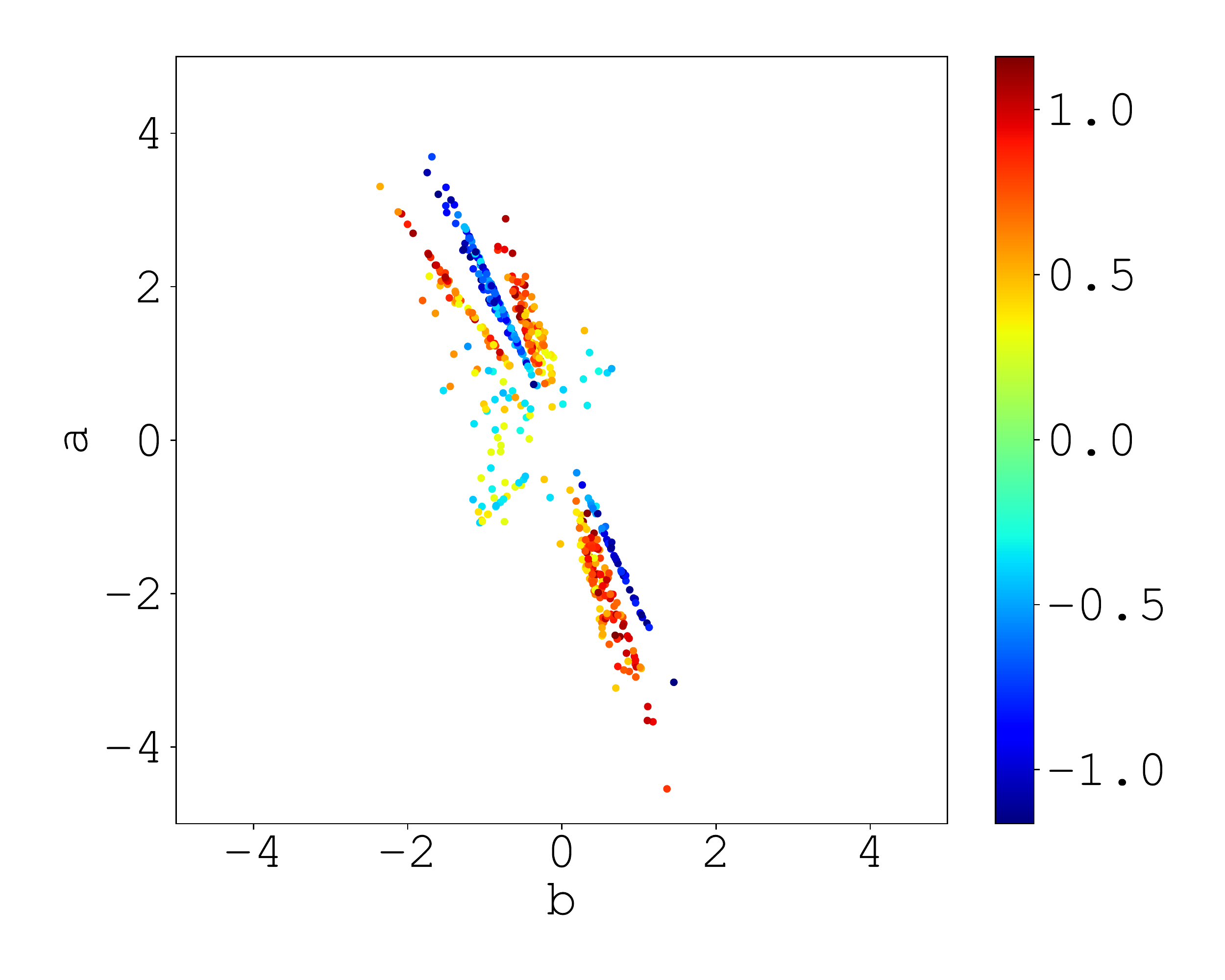}
    \caption{relu, adam}
    \end{subfigure}%
    \begin{subfigure}[c]{0.33\textwidth}
    \includegraphics[width=\linewidth, trim=1cm 0cm 1cm 1cm, clip]{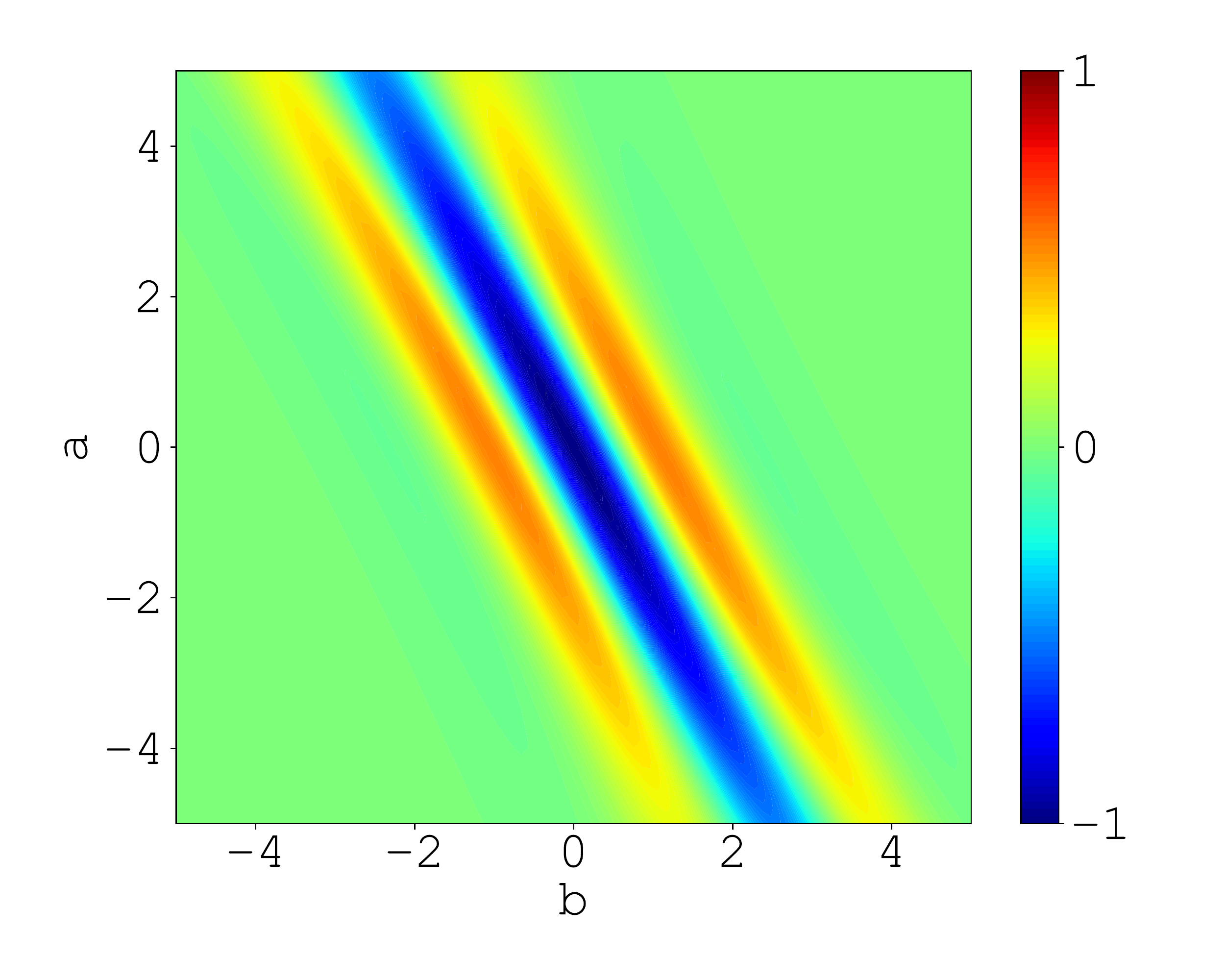}
    \caption{relu}
    \end{subfigure}\\
\caption{Gaussian Kernel $\mu = -0.5$}
\label{fig:rbf.pos-050.n0000}
\end{figure}

\begin{figure}[h]
    \begin{center}
    \begin{subfigure}[c]{0.66\textwidth}
    \includegraphics[width=\linewidth, trim=0cm 0cm 0cm 0cm, clip]{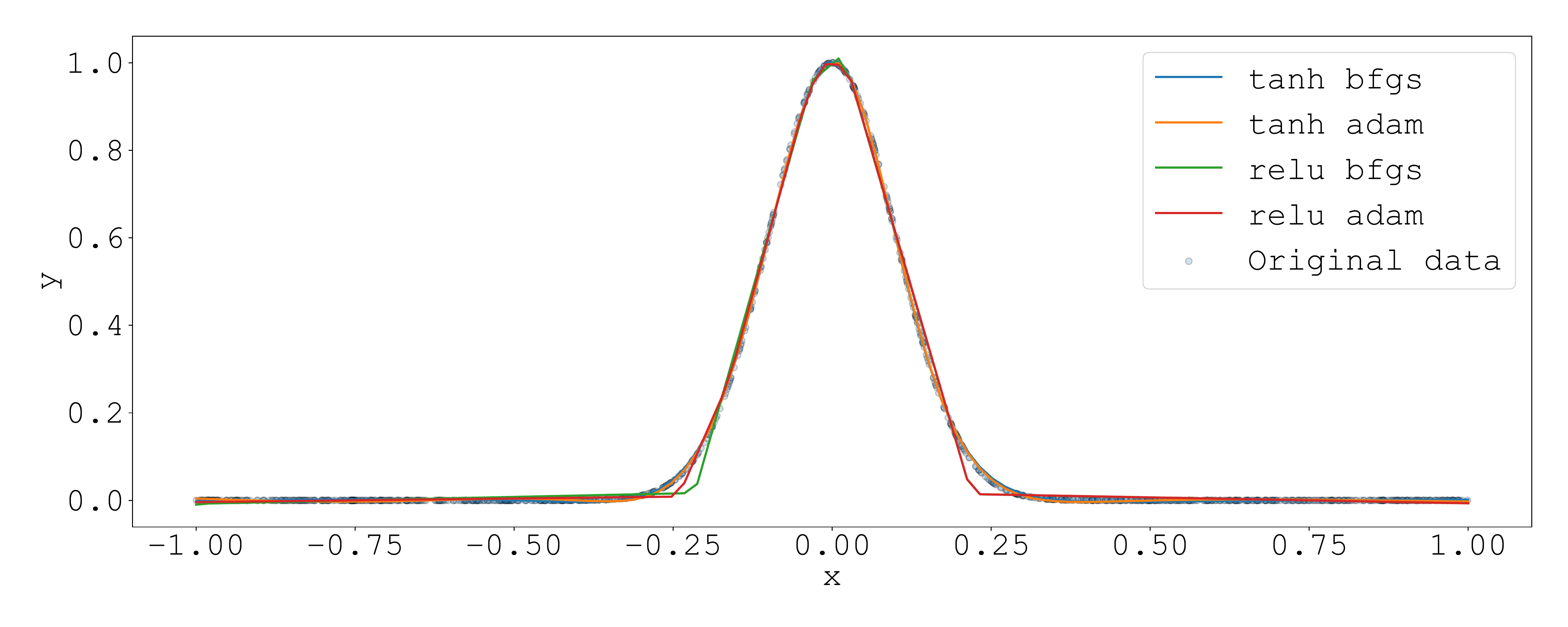}
    \caption{dataset}
    \end{subfigure}
    \end{center}
    \begin{subfigure}[c]{0.33\textwidth}
    \includegraphics[width=\linewidth, trim=1cm 0cm 1cm 1cm, clip]{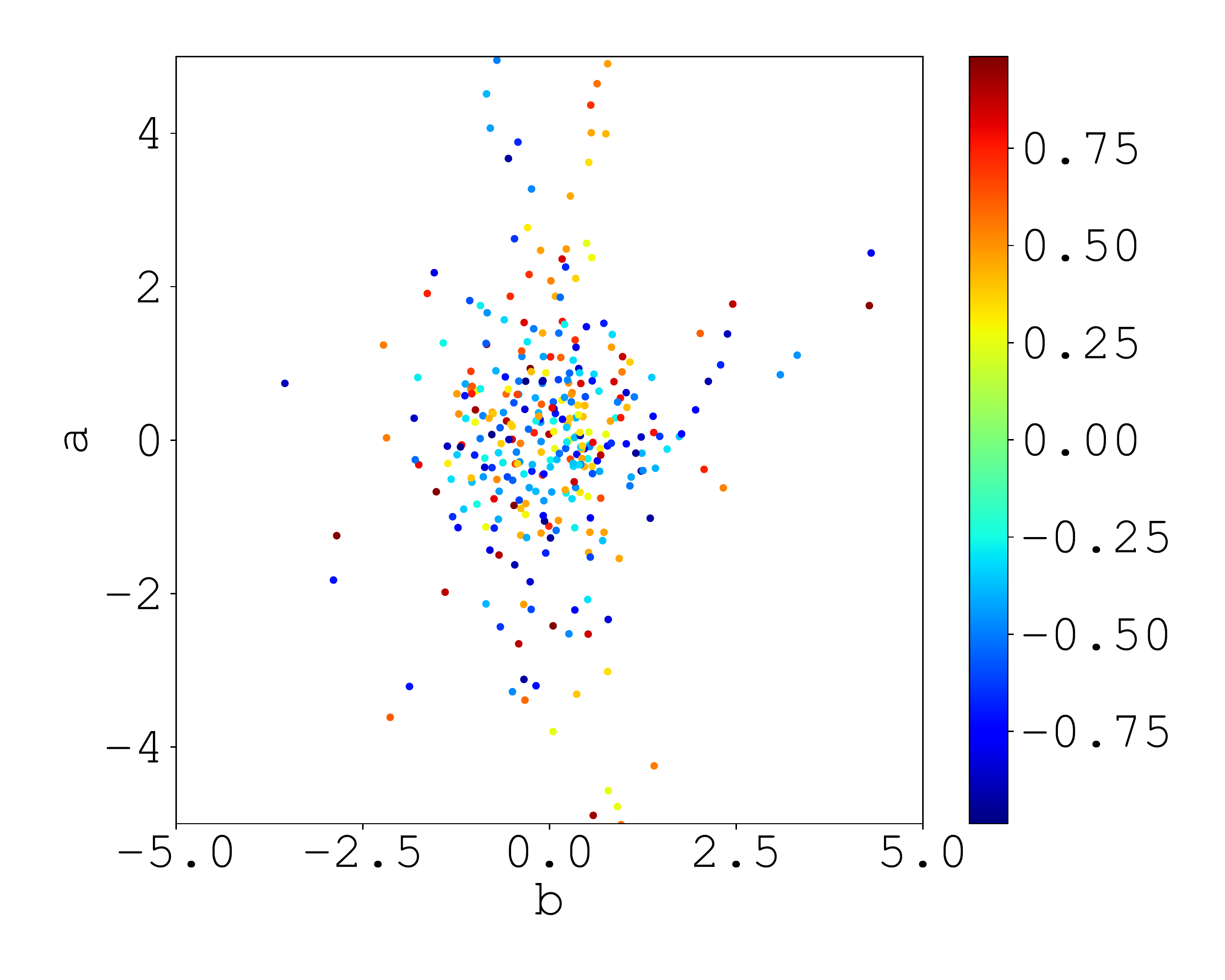}
    \caption{tanh, bfgs}
    \end{subfigure}%
    \begin{subfigure}[c]{0.33\textwidth}
    \includegraphics[width=\linewidth, trim=1cm 0cm 1cm 1cm, clip]{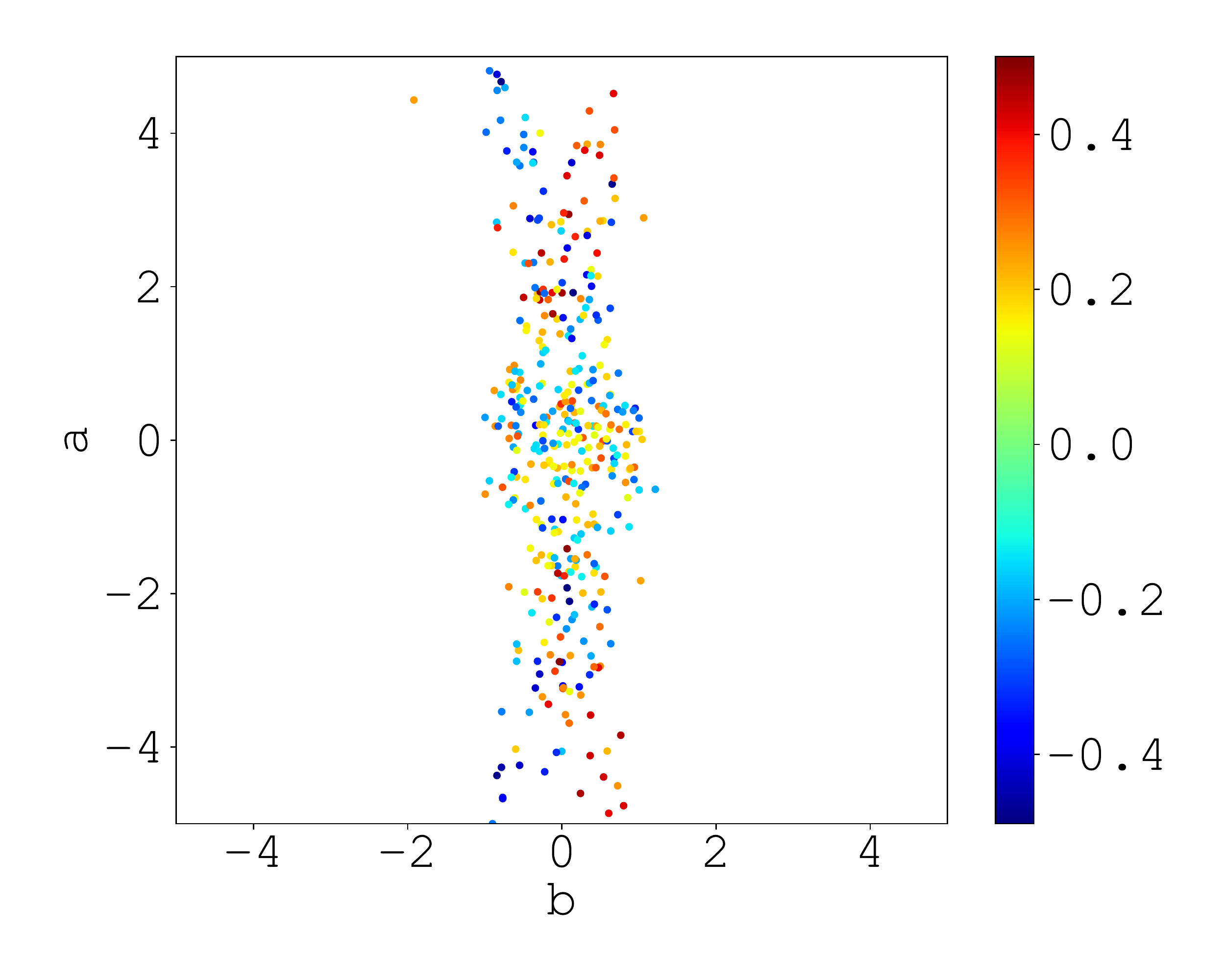}
    \caption{tanh, adam}
    \end{subfigure}%
    \begin{subfigure}[c]{0.33\textwidth}
    \includegraphics[width=\linewidth, trim=1cm 0cm 1cm 1cm, clip]{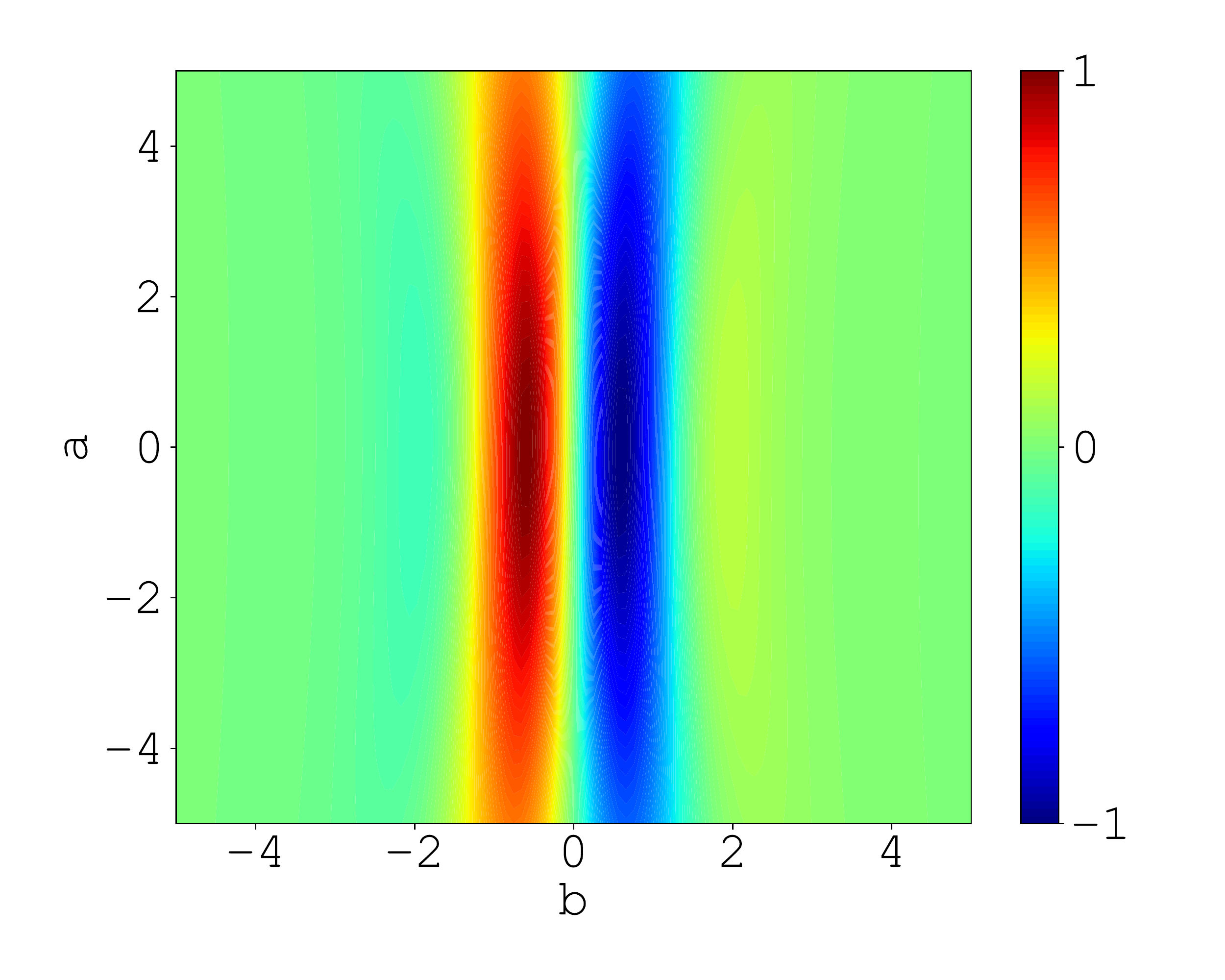}
    \caption{tanh}
    \end{subfigure}\\
    \begin{subfigure}[c]{0.33\textwidth}
    \includegraphics[width=\linewidth, trim=1cm 0cm 1cm 1cm, clip]{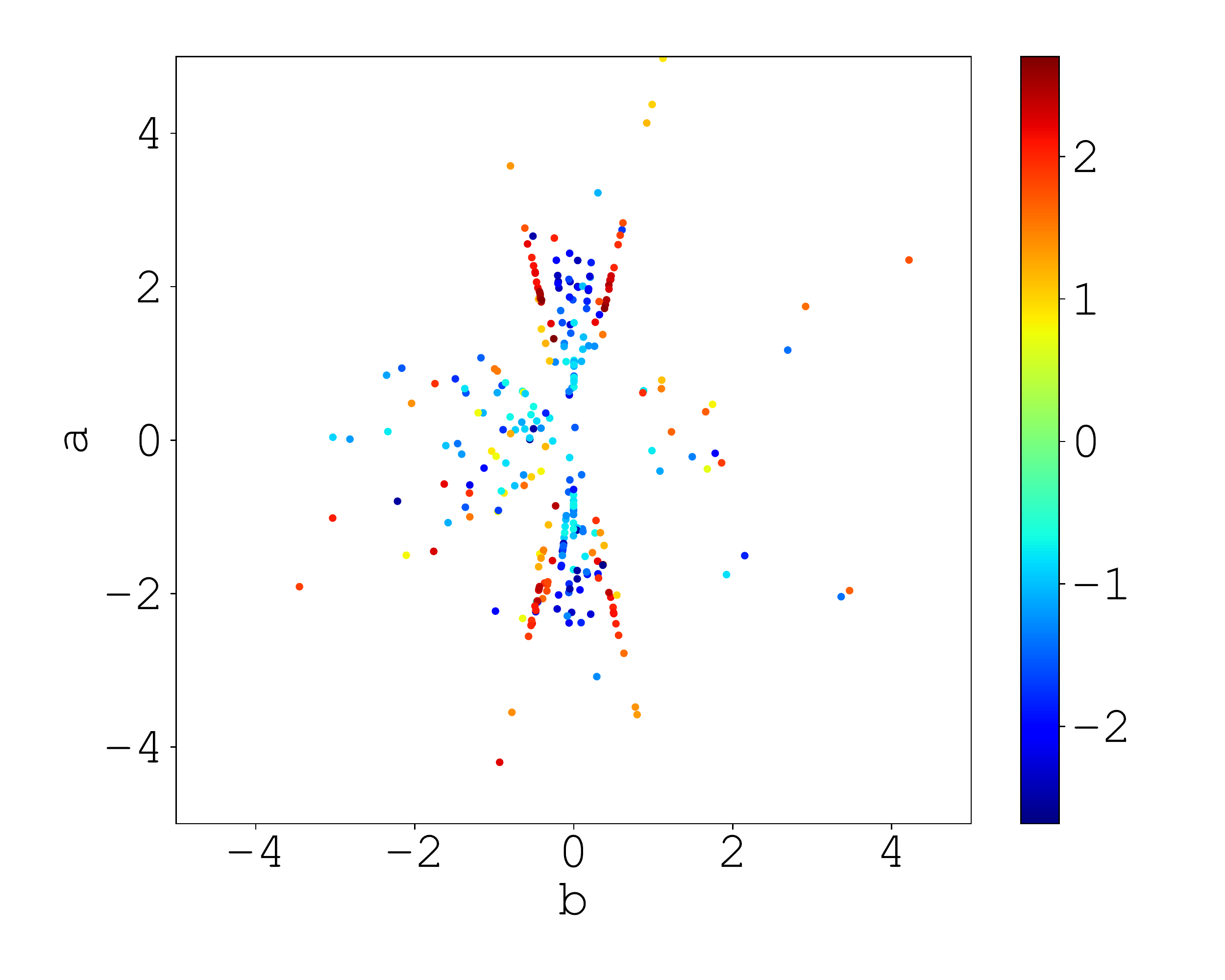}
    \caption{relu, bfgs}
    \end{subfigure}%
    \begin{subfigure}[c]{0.33\textwidth}
    \includegraphics[width=\linewidth, trim=1cm 0cm 1cm 1cm, clip]{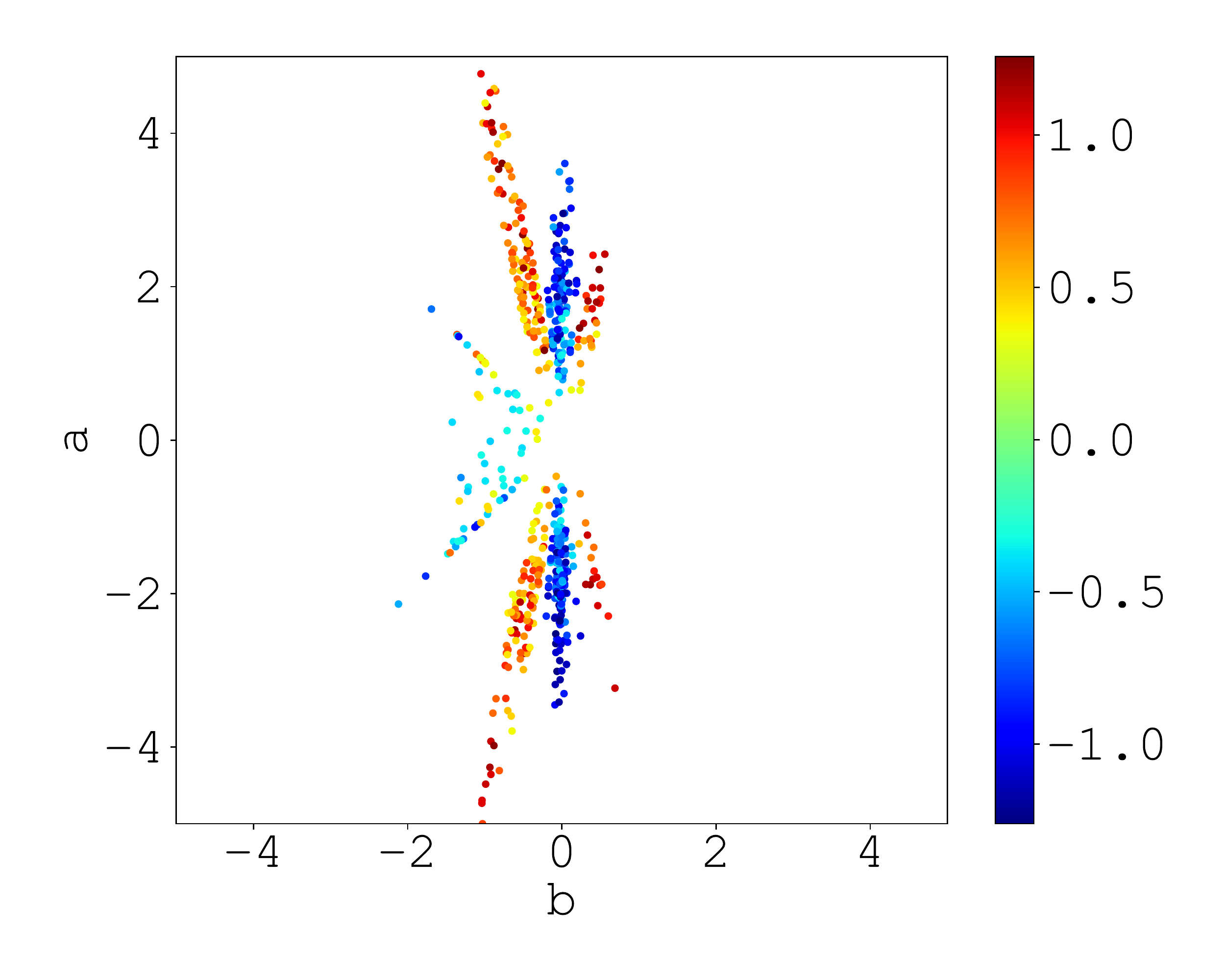}
    \caption{relu, adam}
    \end{subfigure}%
    \begin{subfigure}[c]{0.33\textwidth}
    \includegraphics[width=\linewidth, trim=1cm 0cm 1cm 1cm, clip]{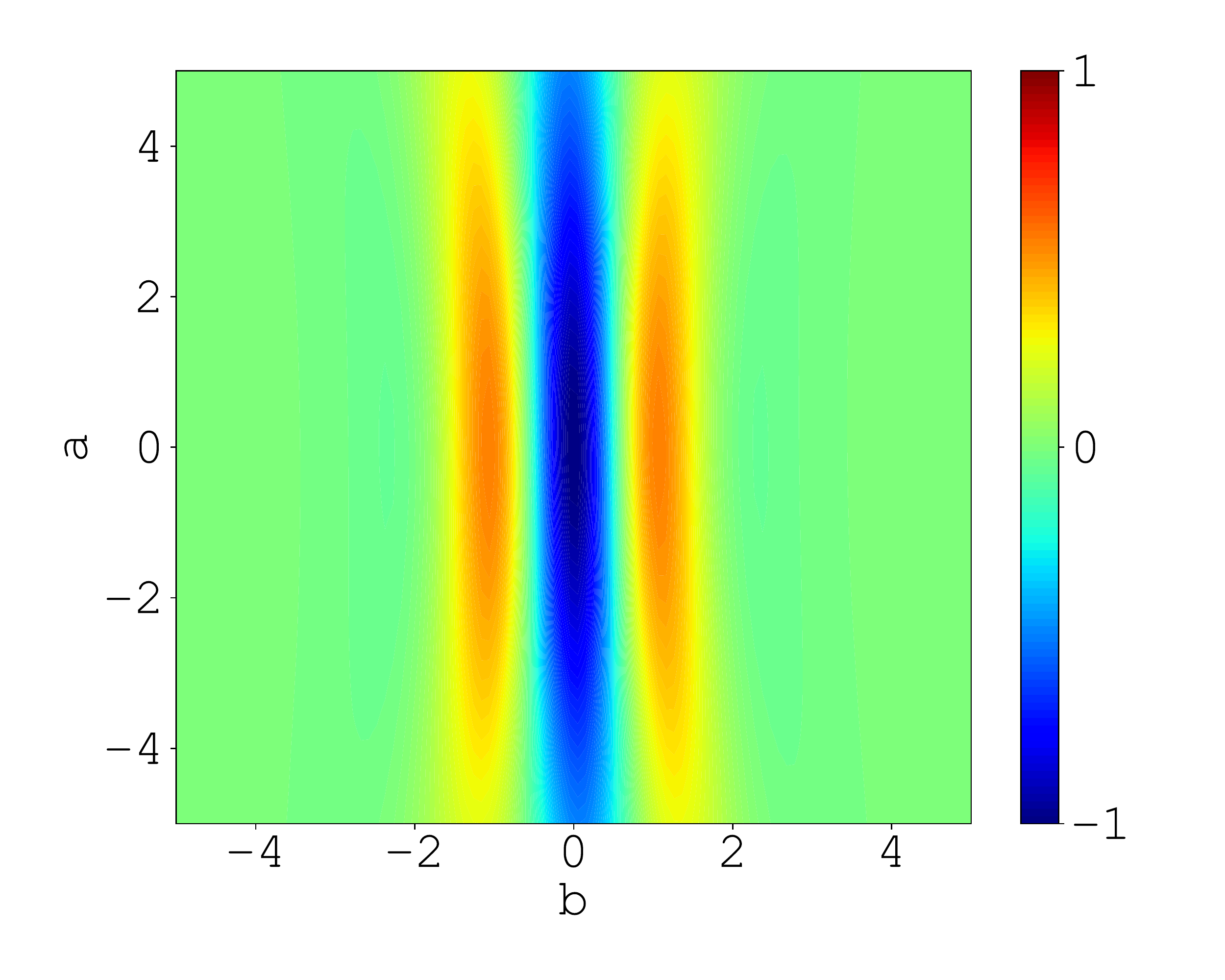}
    \caption{relu}
    \end{subfigure}\\
\caption{Gaussian Kernel  $\mu = 0.0$}
\label{fig:rbf.pos0000.n0000}
\end{figure}

\begin{figure}[h]
    \begin{center}
    \begin{subfigure}[c]{0.66\textwidth}
    \includegraphics[width=\linewidth, trim=0cm 0cm 0cm 0cm, clip]{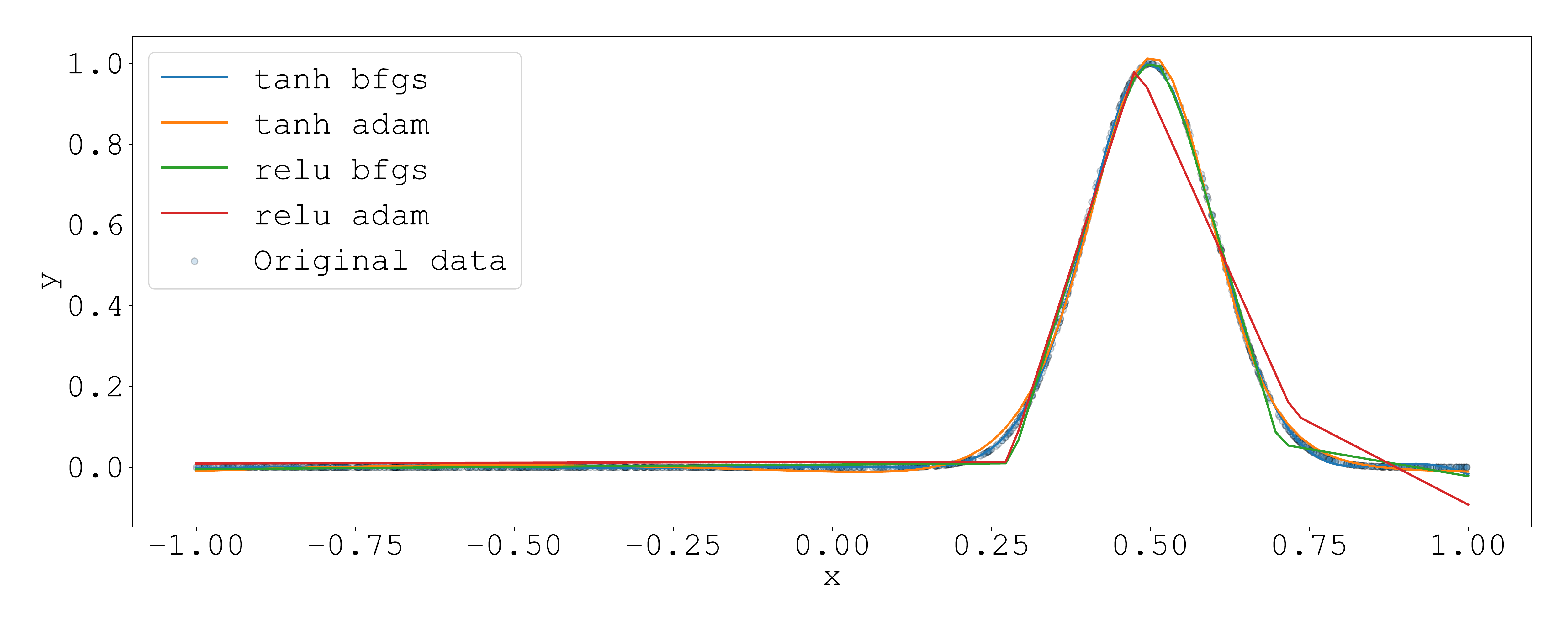}
    \caption{dataset}
    \end{subfigure}
    \end{center}
    \begin{subfigure}[c]{0.33\textwidth}
    \includegraphics[width=\linewidth, trim=1cm 0cm 1cm 1cm, clip]{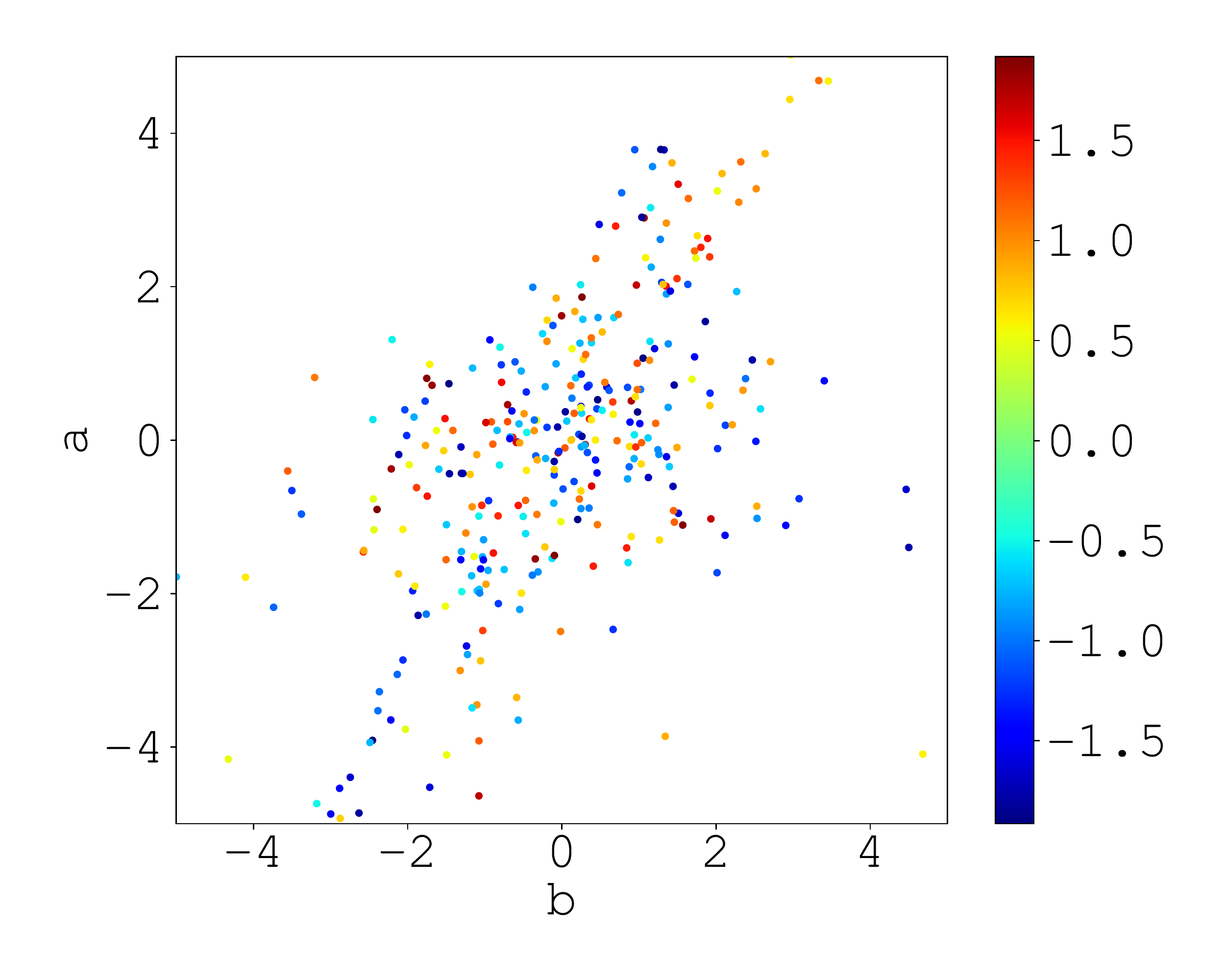}
    \caption{tanh, bfgs}
    \end{subfigure}%
    \begin{subfigure}[c]{0.33\textwidth}
    \includegraphics[width=\linewidth, trim=1cm 0cm 1cm 1cm, clip]{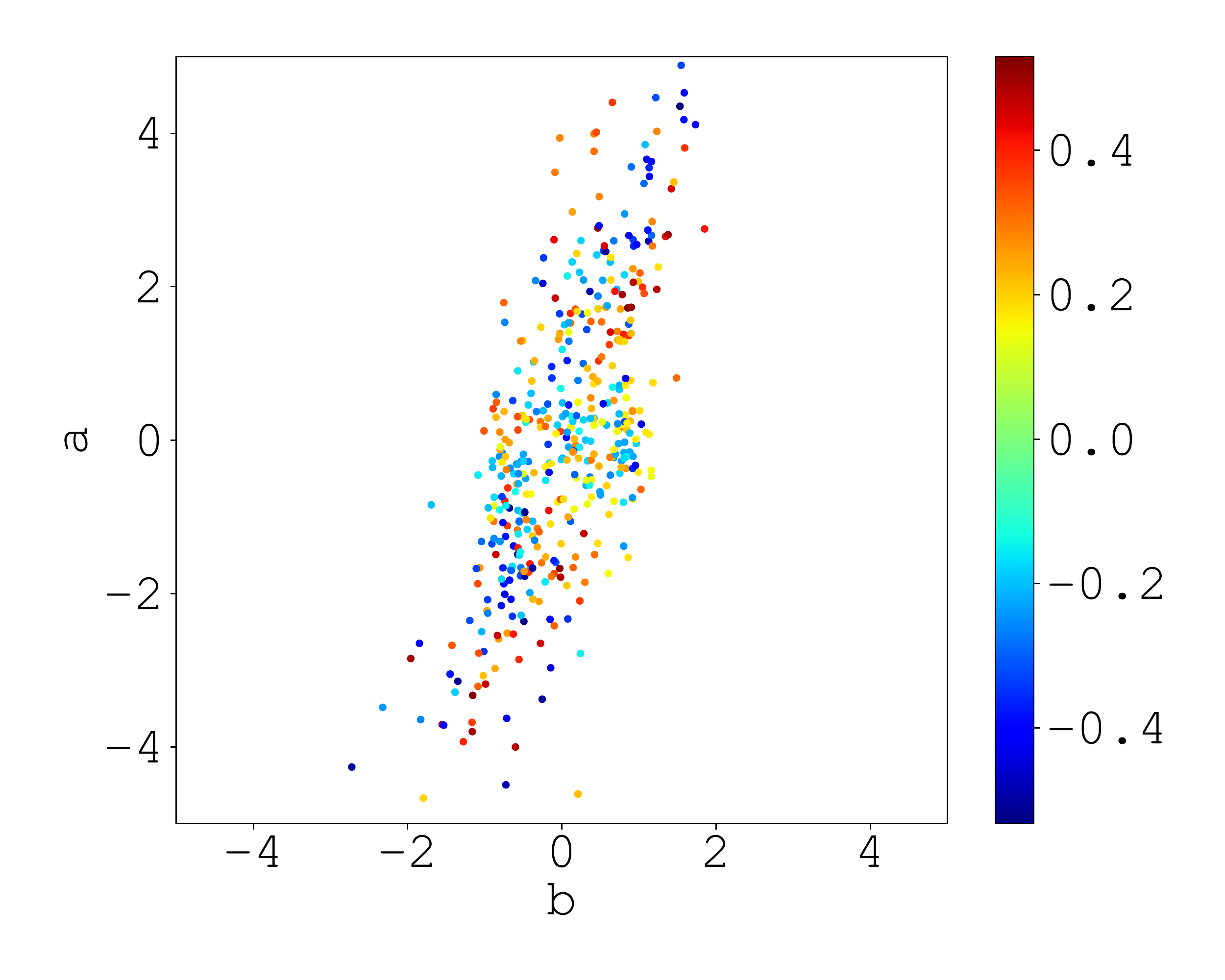}
    \caption{tanh, adam}
    \end{subfigure}%
    \begin{subfigure}[c]{0.33\textwidth}
    \includegraphics[width=\linewidth, trim=1cm 0cm 1cm 1cm, clip]{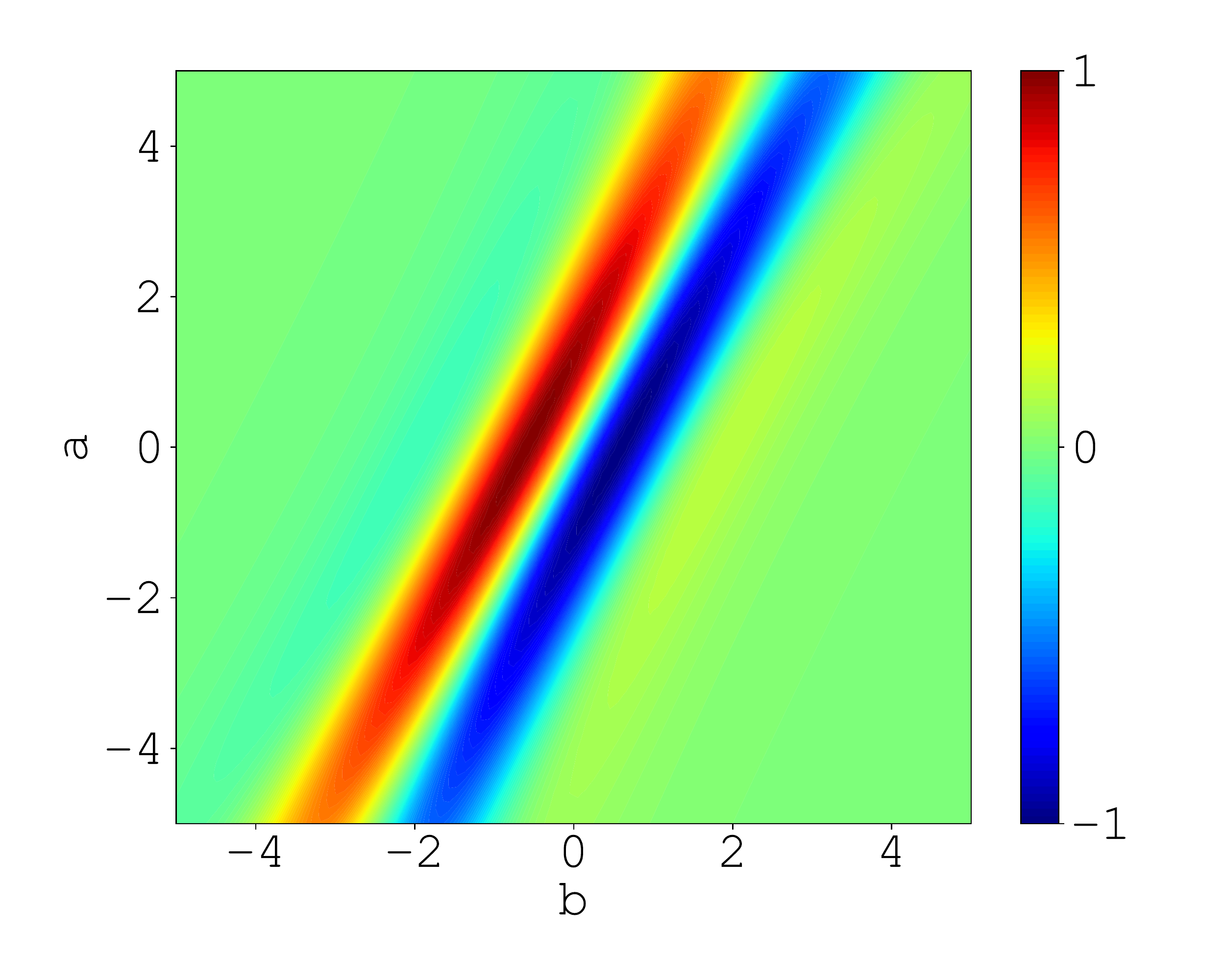}
    \caption{tanh}
    \end{subfigure}\\
    \begin{subfigure}[c]{0.33\textwidth}
    \includegraphics[width=\linewidth, trim=1cm 0cm 1cm 1cm, clip]{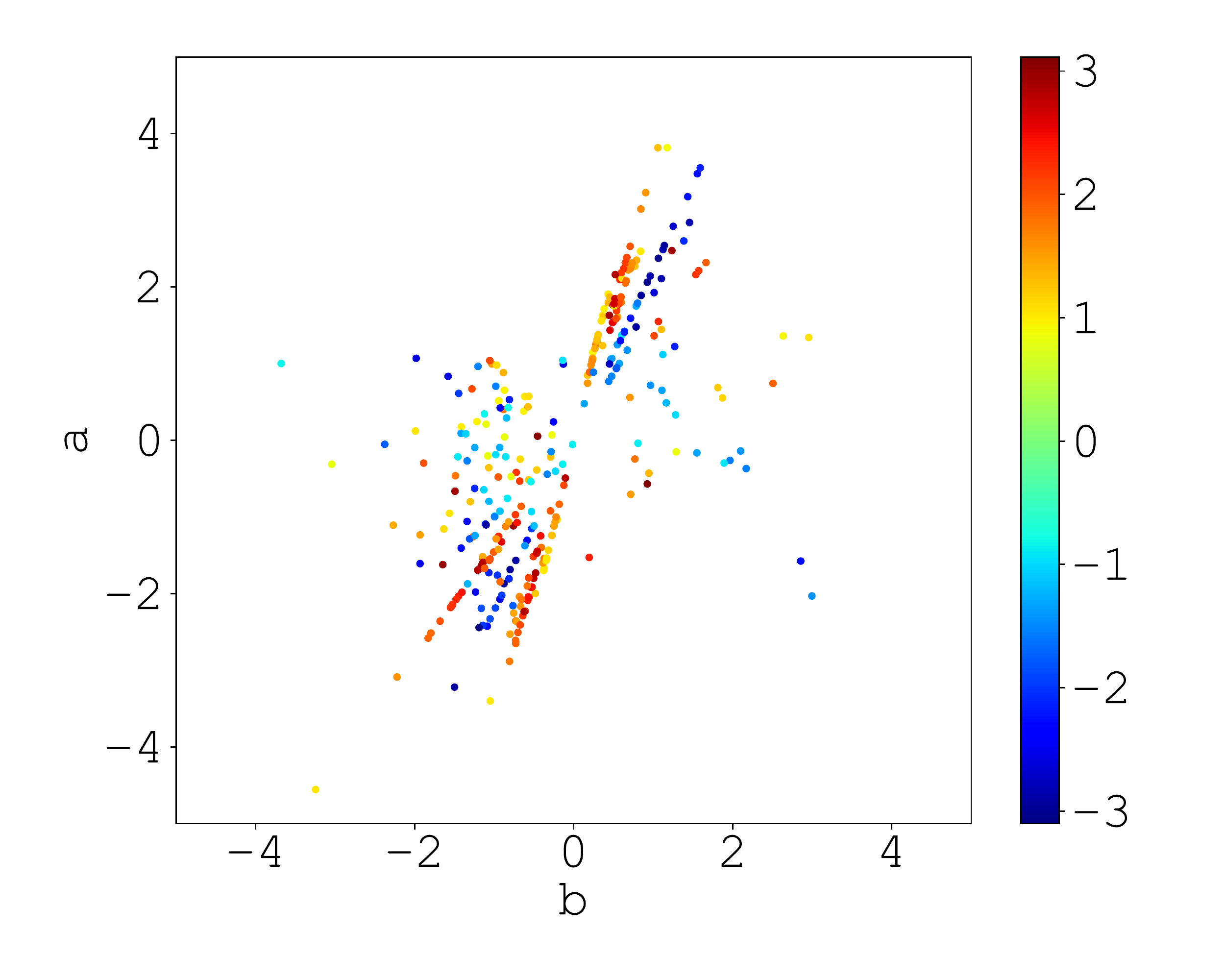}
    \caption{relu, bfgs}
    \end{subfigure}%
    \begin{subfigure}[c]{0.33\textwidth}
    \includegraphics[width=\linewidth, trim=1cm 0cm 1cm 1cm, clip]{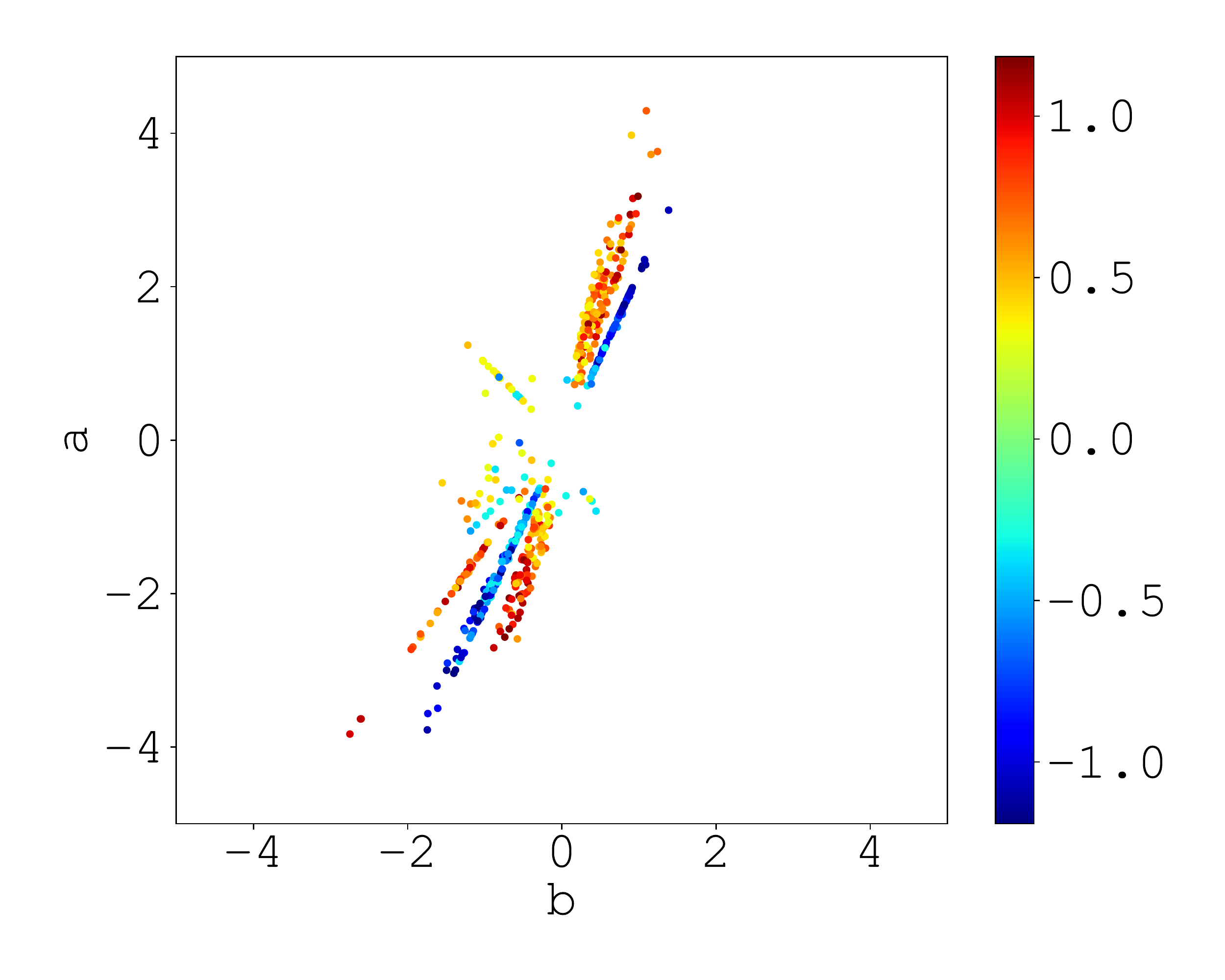}
    \caption{relu, adam}
    \end{subfigure}%
    \begin{subfigure}[c]{0.33\textwidth}
    \includegraphics[width=\linewidth, trim=1cm 0cm 1cm 1cm, clip]{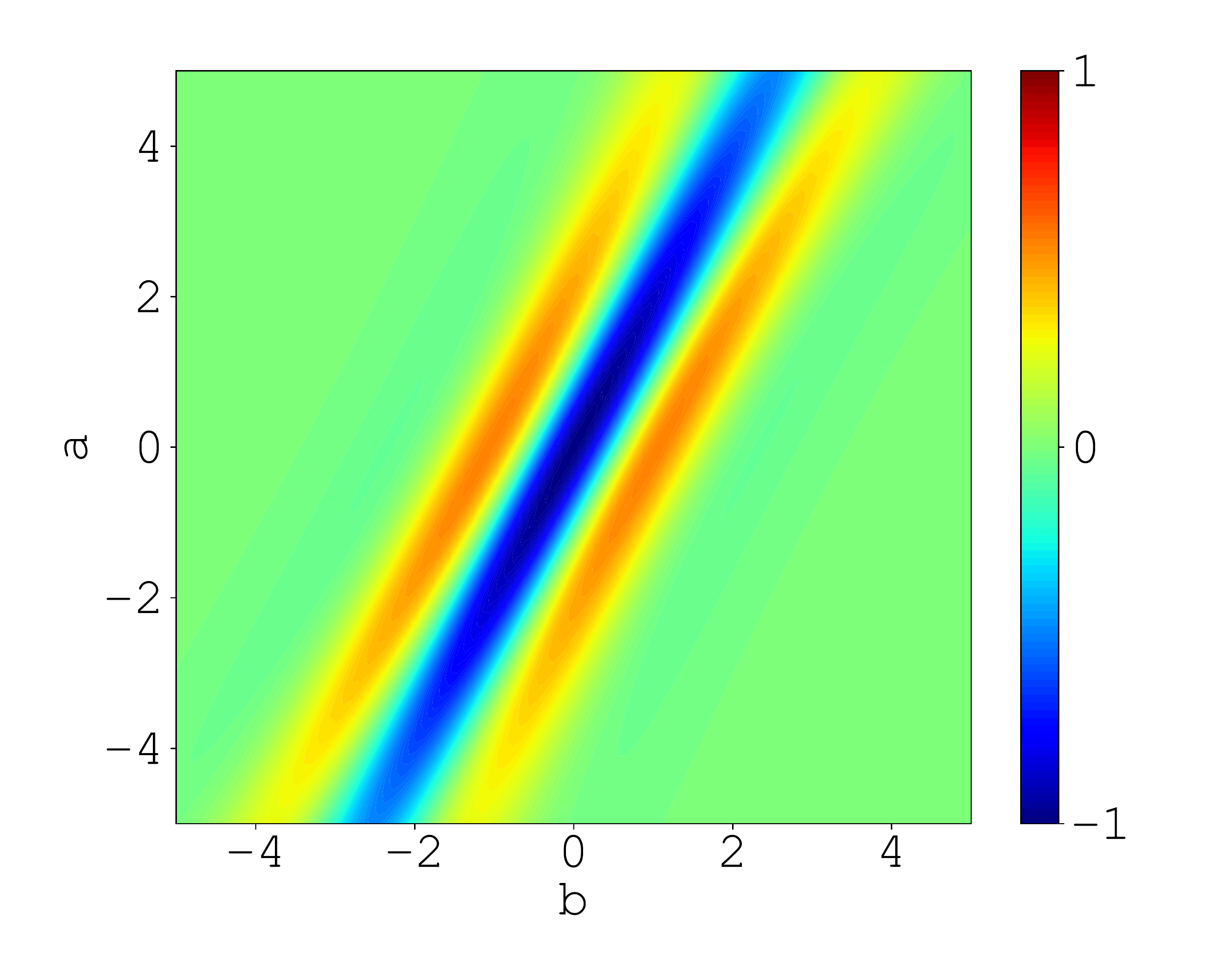}
    \caption{relu}
    \end{subfigure}\\
\caption{Gaussian Kernel  $\mu = 0.5$}
\label{fig:rbf.pos0050.n0000}
\end{figure}

\begin{figure}[h]
    \begin{center}
    \begin{subfigure}[c]{0.66\textwidth}
    \includegraphics[width=\linewidth, trim=0cm 0cm 0cm 0cm, clip]{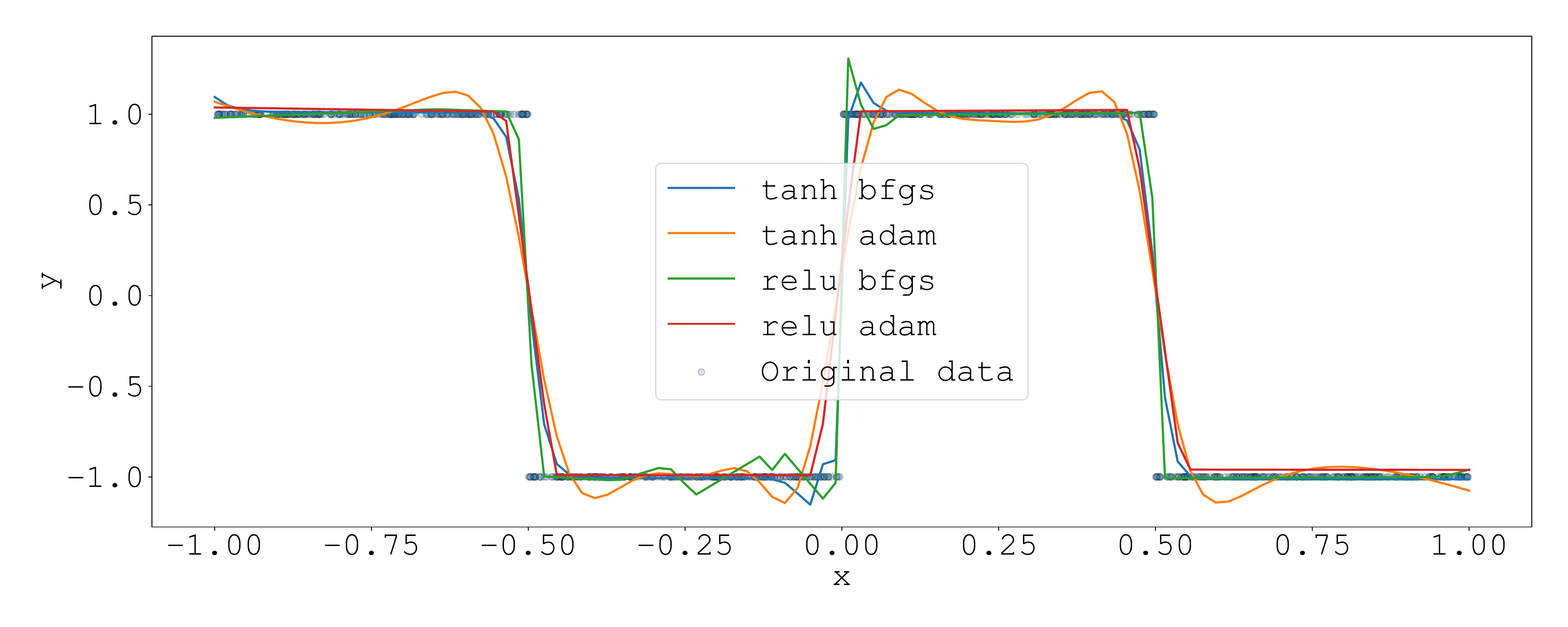}
    \caption{dataset}
    \end{subfigure}
    \end{center}
    \begin{subfigure}[c]{0.33\textwidth}
    \includegraphics[width=\linewidth, trim=1cm 0cm 1cm 1cm, clip]{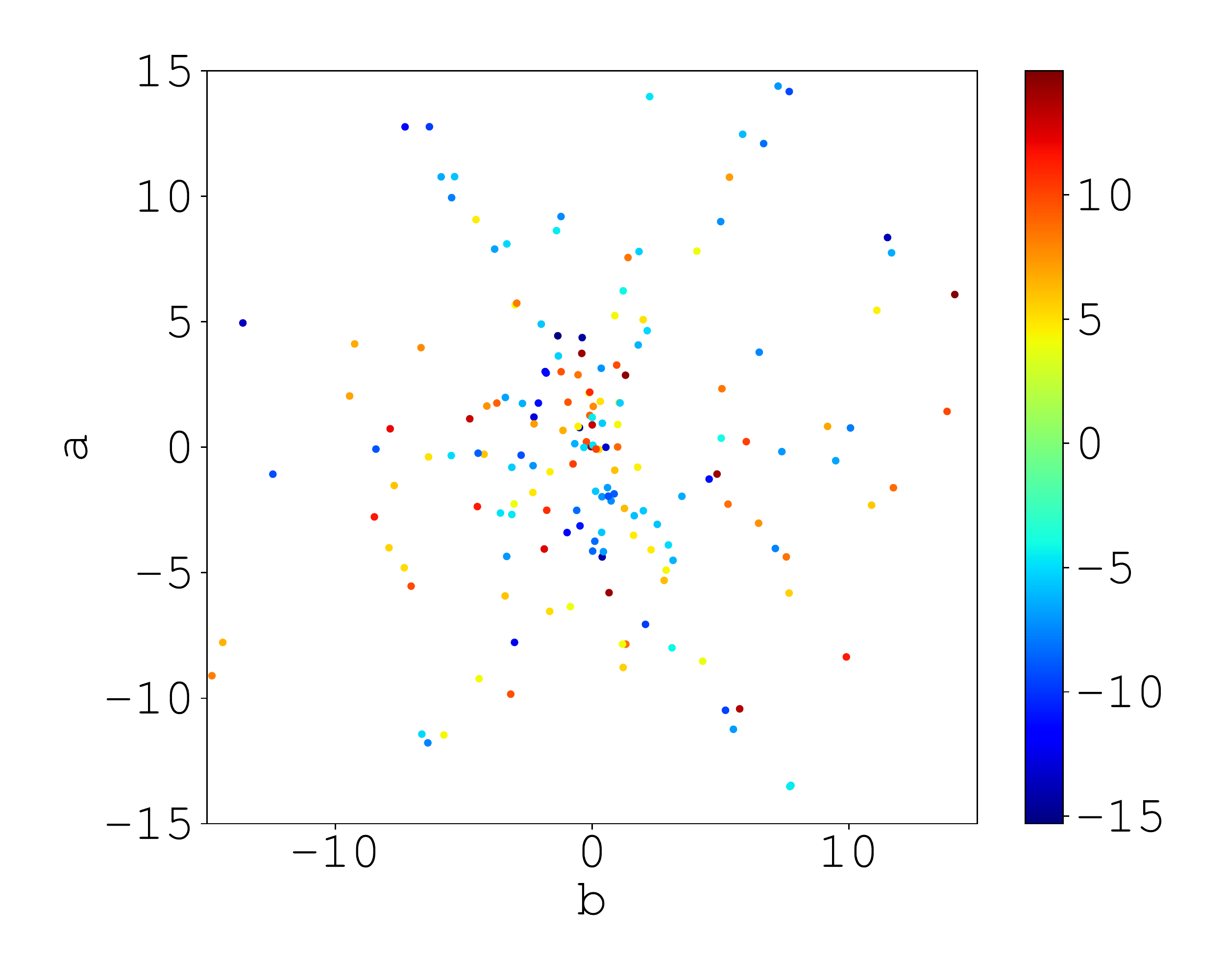}
    \caption{tanh, bfgs}
    \end{subfigure}%
    \begin{subfigure}[c]{0.33\textwidth}
    \includegraphics[width=\linewidth, trim=1cm 0cm 1cm 1cm, clip]{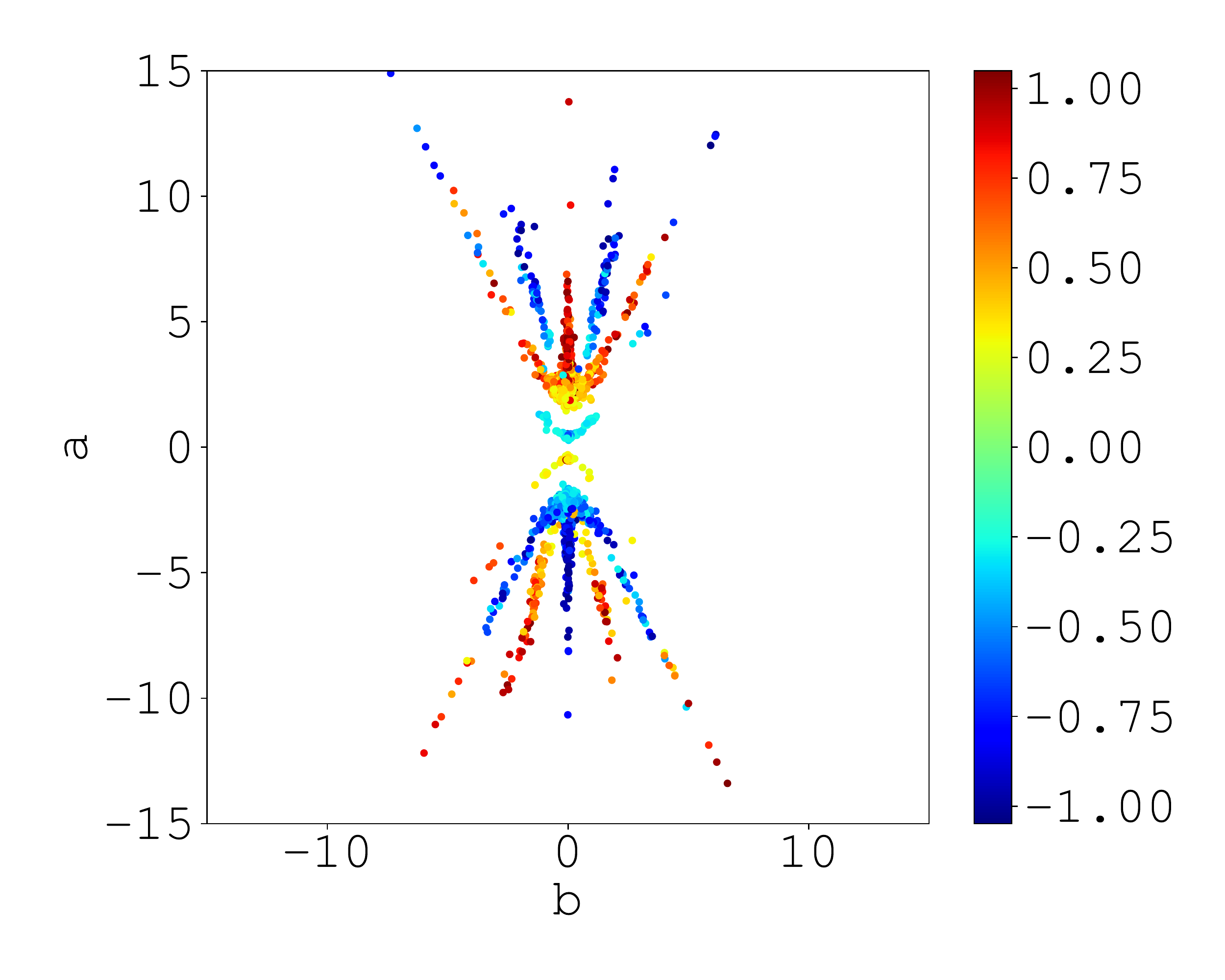}
    \caption{tanh, adam}
    \end{subfigure}%
    \begin{subfigure}[c]{0.33\textwidth}
    \includegraphics[width=\linewidth, trim=1cm 0cm 1cm 1cm, clip]{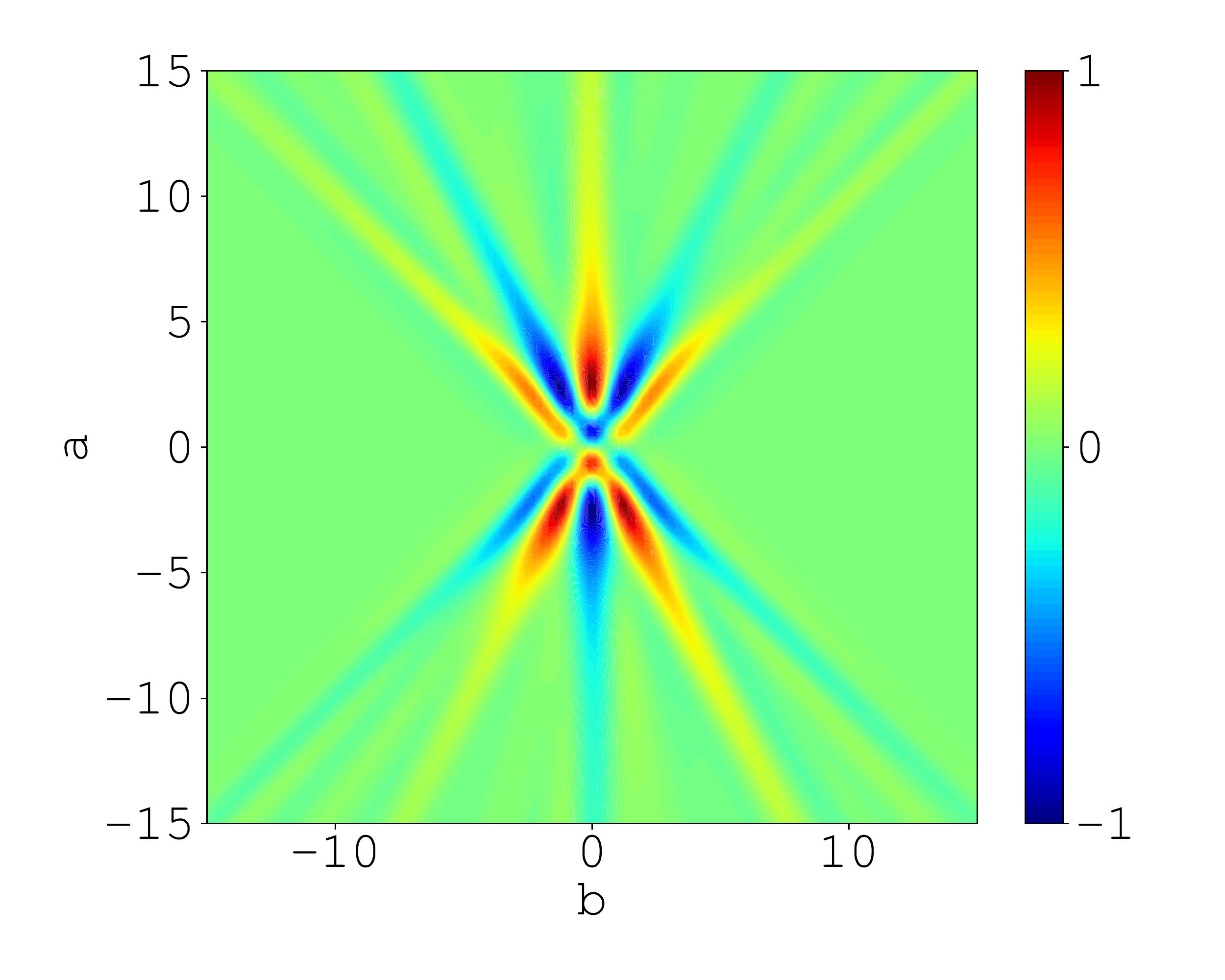}
    \caption{tanh}
    \end{subfigure}\\
    \begin{subfigure}[c]{0.33\textwidth}
    \includegraphics[width=\linewidth, trim=1cm 0cm 1cm 1cm, clip]{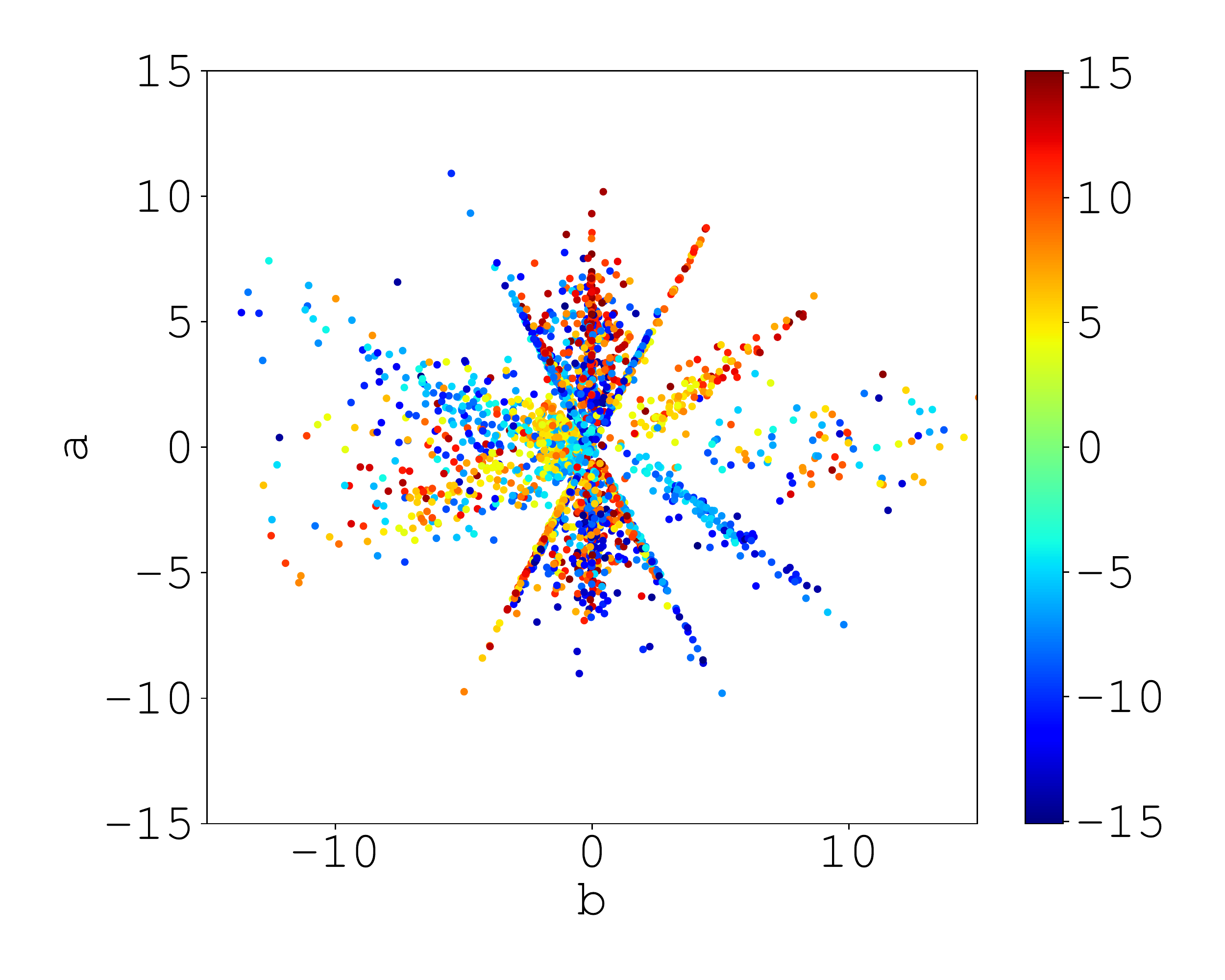}
    \caption{relu, bfgs}
    \end{subfigure}%
    \begin{subfigure}[c]{0.33\textwidth}
    \includegraphics[width=\linewidth, trim=1cm 0cm 1cm 1cm, clip]{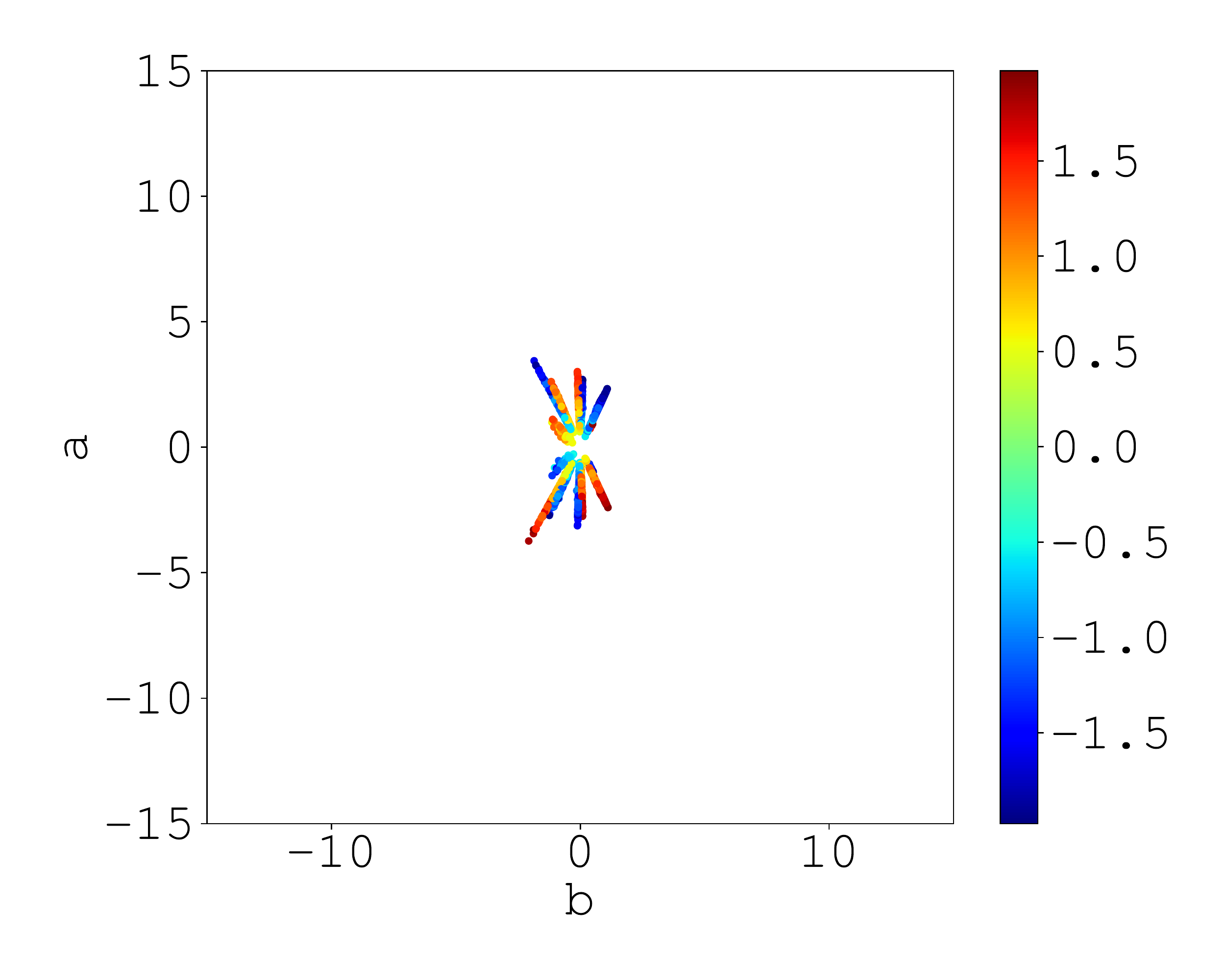}
    \caption{relu, adam}
    \end{subfigure}%
    \begin{subfigure}[c]{0.33\textwidth}
    \includegraphics[width=\linewidth, trim=1cm 0cm 1cm 1cm, clip]{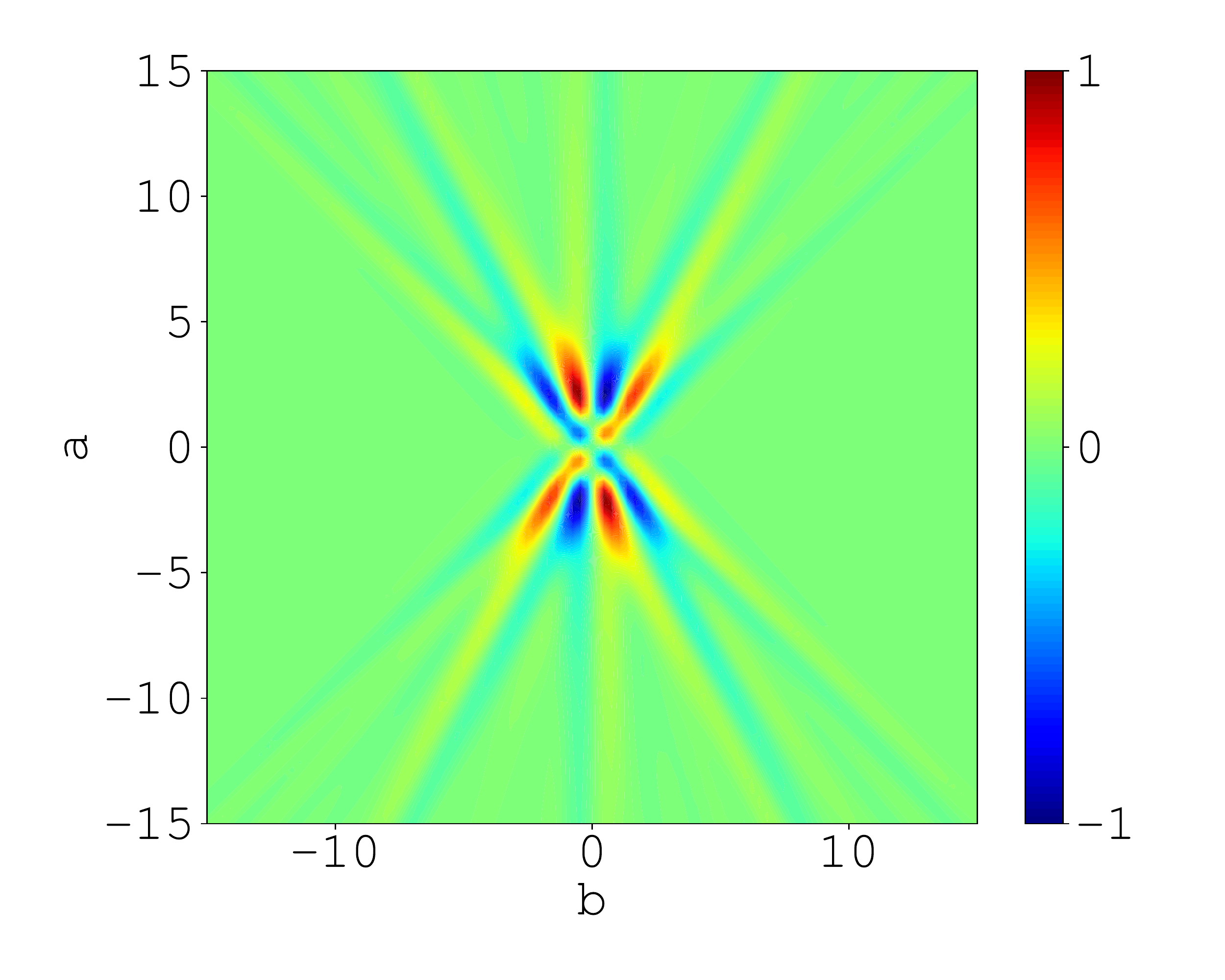}
    \caption{relu}
    \end{subfigure}\\
\caption{Square Wave}
\label{fig:sq02pt.n0000}
\end{figure}